\documentclass[twoside,11pt]{article}

\usepackage[preprint]{jmlr}
\usepackage[utf8]{inputenc} 
\usepackage[T1]{fontenc}    
\usepackage{url}            
\usepackage{booktabs}       
\usepackage{amsfonts}       
\usepackage{nicefrac}       
\usepackage{microtype}      
\usepackage{xcolor}         
\usepackage{wrapfig}
\usepackage{graphicx}
\usepackage[algo2e]{algorithm2e}
\usepackage{comment}
\usepackage{subfig}
\usepackage{enumitem}
\setlist{leftmargin=3mm}
\usepackage{multirow}
\usepackage{tcolorbox}
\usepackage{tikz}
\usepackage{listings}

\usepackage{amsmath,amssymb,amsfonts}
\usepackage{algorithm}
\usepackage{algorithmic}

\usepackage{comment}
\usepackage{bm}
\usepackage{pifont}
\newcommand{\ra}[1]{\renewcommand{\arraystretch}{#1}}

\def\1{\mathbf{1}}

\def\ra{{\textnormal{a}}}

\def\rs{{\textnormal{s}}}

\def\rx{{\textnormal{x}}}
\def\ry{{\textnormal{y}}}

\def\rva{{\mathbf{a}}}

\def\erva{{\textnormal{a}}}

\def\vtheta{{\mathbf{\theta}}}
\def\va{{\mathbf{a}}}

\def\vv{{\mathbf{v}}}
\def\vw{{\mathbf{w}}}
\def\vx{{\mathbf{x}}}
\def\vy{{\mathbf{y}}}
\def\vz{{\mathbf{z}}}

\def\mA{{\mathbf{A}}}
\def\mB{{\mathbf{B}}}

\def\mGamma{{\bm{\Gamma}}}
\def\mtheta{{\bm{\theta}}}
\DeclareMathAlphabet{\mathsfit}{\encodingdefault}{\sfdefault}{m}{sl}
\SetMathAlphabet{\mathsfit}{bold}{\encodingdefault}{\sfdefault}{bx}{n}

\def\gB{{\mathcal{B}}}

\def\gD{{\mathcal{D}}}

\def\gG{{\mathcal{G}}}
\def\gH{{\mathcal{H}}}

\def\gL{{\mathcal{L}}}

\def\gO{{\mathcal{O}}}

\def\gR{{\mathcal{R}}}
\def\gS{{\mathcal{S}}}
\def\gT{{\mathcal{T}}}

\def\gX{{\mathcal{X}}}
\def\gY{{\mathcal{Y}}}
\def\gZ{{\mathcal{Z}}}

\def\sA{{\mathbb{A}}}
\def\sB{{\mathbb{B}}}

\def\sI{{\mathbb{I}}}

\def\sN{{\mathbb{N}}}

\def\sP{{\mathbb{P}}}

\def\sR{{\mathbb{R}}}

\def\sU{{\mathbb{U}}}

\newcommand{\E}{\mathbb{E}}

\newcommand{\normlp}{L^p}

\DeclareMathOperator*{\argmax}{arg\,max}
\DeclareMathOperator*{\argmin}{arg\,min}

\usepackage{color, colortbl}
\usepackage{bbm}
\usepackage{bm}
\usepackage{hyperref}
\usepackage[capitalize,noabbrev]{cleveref}
\lstdefinestyle{mystyle}{
    backgroundcolor=\color{backcolour},   
    commentstyle=\color{codegreen},
    keywordstyle=\color{magenta},
    numberstyle=\tiny\color{codegray},
    stringstyle=\color{codepurple},
    basicstyle=\ttfamily\footnotesize,
    breakatwhitespace=false,         
    breaklines=true,                 
    captionpos=b,                    
    keepspaces=true,                 
    numbers=left,                    
    numbersep=5pt,                  
    showspaces=false,                
    showstringspaces=false,
    showtabs=false,                  
    tabsize=2
}
\usepackage{soul}
\definecolor{LightCyan}{rgb}{0.9,1,1}
\definecolor{light-gray}{gray}{0.97}
\definecolor{codegreen}{rgb}{0,0.6,0}
\definecolor{codegray}{rgb}{0.5,0.5,0.5}
\definecolor{codepurple}{rgb}{0.58,0,0.82}
\definecolor{backcolour}{rgb}{0.95,0.95,0.92}

\crefname{assumption}{Assumption}{Assumptions}
\crefname{remark}{Remark}{Remarks}
\newcommand{\Mu}{\bm{\mu}}
\usepackage{lastpage}
\jmlrheading{23}{2025}{1-\pageref{LastPage}}{1/21; Revised 5/22}{9/22}{21-0000}{Acharya, Jing, Bhushanam, Choudhary, Rabbat, Sanghavi, Dhillon}

\ShortHeadings{Contrastive PU Learning}{Acharya, Jing, Bhushanam, Choudhary, Rabbat, Sanghavi, Dhillon}

\begin{document}

\title{Understanding Contrastive Representation Learning \\
from Positive Unlabeled (PU) Data}

\author{\name Anish Acharya~\thanks{Part of the work was done while the author was at Meta.} \email anishacharya@utexas.edu \\
       \addr The University of Texas at Austin \\
       2515 Speedway, Austin, TX 78712
       \AND
       \name Li Jing~$^\ast$ \email lijing@openai.com \\
       \addr OpenAI \\
       1455 3rd Street, San Francisco, California 94158
       \AND
       \name Bhargav Bhushanam \email bbhushanam@fb.com \\
       \addr Meta \\
       1 Meta Way, Menlo Park, CA 94025
       \AND
       \name Dhruv Choudhary \email choudharydhruv@fb.com \\
       \addr Meta \\
       1 Meta Way, Menlo Park, CA 94025
       \AND
       \name Michael Rabbat \email mikerabbat@fb.com \\
       \addr Meta \\
       1 Meta Way, Menlo Park, CA 94025
       \AND 
       \name Sujay Sanghavi \email sanghavi@mail.utexas.edu \\
       \addr The University of Texas at Austin \\
       2515 Speedway, Austin, TX 78712
       \AND
       \name Inderjit S Dhillon \email inderjit@cs.utexas.edu \\
       \addr The University of Texas at Austin \\
       2515 Speedway, Austin, TX 78712
       }

\editor{}

\maketitle

\begin{abstract}%
Pretext Invariant Representation Learning (PIRL) followed by Supervised Fine-Tuning (SFT) has become a standard paradigm for learning with limited labels. We extend this approach to the Positive-Unlabeled (PU) setting, where only a small set of labeled positives and a large unlabeled pool—containing both positives and negatives—are available. We study this problem under two regimes: (i) without access to the class prior, and (ii) when the prior is known or can be estimated. We introduce Positive Unlabeled Contrastive Learning (\textsc{puCL}), an unbiased and variance-reducing contrastive objective that integrates weak supervision from labeled positives judiciously into the contrastive loss. When the class prior is known, we propose Positive Unlabeled \textsc{InfoNCE} (\textsc{puNCE}), a prior-aware extension that re-weights unlabeled samples as soft positive-negative mixtures. For downstream classification, we develop \textsc{puPL}, a pseudo-labeling algorithm that leverages the structure of the learned embedding space via PU aware clustering. Our framework is supported by theory; offering bias-variance analysis, convergence insights, and generalization guarantees via augmentation concentration; and validated empirically across standard PU benchmarks, where it consistently outperforms existing methods, particularly in low-supervision regimes.
\end{abstract}

\begin{keywords}
  keyword one, keyword two, keyword three
\end{keywords}

\tableofcontents

\clearpage
\section{Introduction}
\label{sec:intro}
Learning from limited amount of labeled data is a longstanding challenge in modern machine learning.
Owing to its recent widespread success in both computer vision and natural language processing tasks~\citep{chen2020big, Grill2020BootstrapYO, Radford2021LearningTV, Gao2021SimCSESC, dai2015semi, radford2018improving} 
{\bf Pretext Invariant Representation Learning (PIRL)} followed by {\bf Supervised Fine-Tuning (SFT)} has become the de-facto approach to learn from limited supervision. 
This two-stage approach, which first leverages unlabeled data in a task-agnostic manner and subsequently adapts to the target task using labeled data, has driven state-of-the-art performance across a wide range of computer vision and natural language processing tasks~\citep{chen2020big, Grill2020BootstrapYO, Radford2021LearningTV, Gao2021SimCSESC, dai2015semi, radford2018improving, hinton2006fast, bengio2006greedy, mikolov2013distributed, kiros2015skip, devlin2018bert, Zbontar2021BarlowTS}.
In this context, {\bf Contrastive Learning (CL)}\citep{gutmann2010noise, sohn2016improved, tian2020contrastive, chen2020simple} has emerged as a particularly powerful approach for learning such pretext-invariant representations. By encouraging embeddings of semantically similar inputs to be mapped closer together, while pushing apart dissimilar ones, CL effectively induces invariance to a class of label-preserving transformations. This inductive bias has been shown to yield representations that transfer well across tasks, especially in settings where labeled data is scarce. 
\\
\\
Despite the widespread empirical success of CL, its theoretical and algorithmic underpinnings in {\bf weakly supervised} settings remain comparatively underexplored~\citep{assran2020supervision, cui2023rethinking, zheng2021weakly, xue2022investigating}. 
Classical formulations of CL objectives typically fall into one of two extremes: either fully self-supervised, where semantic similarity is heuristically induced through data augmentations~\citep{chen2020simple, oord2018representation}, 
or fully supervised, where labels provide explicit guidance for defining positive and negative pairs~\citep{khosla2020supervised}.
However, many real-world settings fall into a spectrum of weak supervision, where the training signal is indirect, imprecise, or only partially aligned with the true task objective.  
CL relies on the ability to form reliable similar and dissimilar pairs. In weakly supervised regimes, this assumption becomes increasingly fragile. Naively applying supervised contrastive losses can introduce statistical bias when label noise or label sparsity leads to incorrect pair assignments. Conversely, purely self-supervised methods—though unbiased—often suffer from high variance in the learned representations due to the lack of semantic guidance. This creates a fundamental tension between bias and variance in weakly supervised contrastive learning that is remain unexplored in existing literature.
\\
\\
To this end, this paper investigates and extends contrastive representation learning to the {\bf Positive Unlabeled (PU) Learning}~\citep{denis1998pac} setting - 
\begin{quote}
    The weakly supervised task of learning a binary (positive vs negative) classifier {\bf in absence of any explicitly labeled negative examples}, i.e., using an incomplete set of positives and a set of unlabeled samples.
\end{quote}
This setting is frequently encountered in several real-world applications, especially where obtaining negative samples is either expensive or infeasible. For example, consider {\em personalized recommendation systems}~\citep{naumov2019deep} where the training data is typically extracted from user feedback. Since explicit user feedback (e.g. user ratings) is hard to obtain, most practical recommendation systems rely on implicit user feedback (e.g. browsing history)~\citep{kelly2003implicit} which usually indicates user's positive preference (e.g. if a user browses a product frequently or watched a movie then the user-item pair is labeled positive)~\citep{chen2021pu}. The study of PU Learning has also been motivated by diverse domains such as -- drug, gene, and protein identification~\citep{yang2012positive}, anomaly detection~\citep{blanchard2010semi}, fake news detection~\citep{ren2014positive}, matrix completion~\citep{hsieh2015pu}, data imputation~\citep{denis1998pac}, named entity recognition (NER)~\citep{peng2019distantly} and face recognition~\citep{kato2018learning} among others.

\subsection{Overview: Contrastive Approach to PU Learning}
In this work, we present a principled investigation into the design and analysis of the popular \textsc{infoNCE}~\citep{gutmann2010noise} family of contrastive objectives, under the Positive-Unlabeled (PU) setting, where the {\bf available supervision is both partial (only positives are labeled) and asymmetric (no labeled negatives)}. Our goal is to understand how to integrate this weak supervision signal into the InfoNCE family of contrastive objectives in a way that preserves statistical consistency, minimizes estimator variance, and improves representation quality. 
\\
\\
To this end, in{\bf ~\cref{sec:puCL}}, we study several adaptations of \textsc{infoNCE} that reflect different assumptions about the unlabeled data and different uses of the available supervision. We begin by analyzing two classical baselines: {\bf Self Supervised Contrastive Learning (\textsc{ssCL})}\citep{chen2020simple}, which is unbiased but high variance, and the naive adaptation of the supervised contrastive loss~\citep{khosla2020supervised} {\bf \textsc{sCL-PU}}~\eqref{eq:scl_pu}, which suffers from bias due to incorrect treatment of unlabeled samples as negatives. We quantify the sources of estimation error in each case and demonstrate empirically and theoretically how these manifest in degraded representation quality, particularly in the low-supervision regime.
\begin{algorithm*}[t]
        \SetAlgoLined
        \vspace{1em}
        {\bf initialize:}
        PU training data $\gX_\textsc{PU}$~\eqref{eq:pu_dataset};  
        batch size b; temperature parameter $\tau > 0$; randomly initialized encoder $g_{\mB}(\cdot): \sR^d \to \sR^k$, projection network: $h_{\bm{\Gamma}}(\cdot): \sR^k \to \sR^p$, family of stochastic augmentations $\gT$, (optionally) class prior estimate $\pi := \Pr(y = 1)$.

        \For{epochs  e = 1, 2, \dots , until convergence}
        {   
            {\em select mini-batch}:
            \\
            $\gD=\{\vx_i\}_{i=1}^b \sim \gX_\textsc{PU}$
            \\
            \\
            {\em create multi-viewed batch}: 
            \\
            $t(\cdot)\sim \gT, t'(\cdot) \sim \gT$
            \\
            $
            \tilde{\gD} 
            = \{\tilde{\vx}_i=t(\vx_i), \tilde{\vx}_{a(i)}=t'(\vx_i) \}_{i=1}^b 
            $
            \\
            $\sI = \{1, 2, \dots, 2b\}$ is the index set of $\tilde{\gD}$ and,
            \\
            $\sP = \{i \in \sI : \vx_i \in \gX_{\textsc{P}}\}, \sU = \{j \in \sI : \vx_j \in \gX_\textsc{U}\}$
            \\
            \\
            {\em obtain representations}: 
            \\
            $\{\vz_j\}_{j\in \sI} = \{ \vz_i = h_{\mGamma}\circ g_{\mB}(\tilde{\vx}_i), \;
            \vz_{a(i)} = h_{\mGamma}\circ g_{\mB}(\tilde{\vx}_{a(i)})\}_{i=1}^b$ 
            \\
            \\
            {\em compute pairwise similarity}:
            \\
            $
            \vz_i \boldsymbol \cdot \vz_j = \frac{1}{\tau}\frac{\vz_i^T\vz_j}{\|\vz_i\|\|\vz_j\|},\;
            P(i,j) = \frac{\exp{(\vz_i} \boldsymbol \cdot \vz_j)}{\sum\limits_{k \in \sI}\1(k \neq i)\exp(\vz_i\boldsymbol\cdot \vz_k)},  \forall i,j\in \sI
            $
            \\
            \\
            {\em compute loss} :
            \\
            \colorbox{light-gray}{
            $\gL_{\textsc{ssCL}}
            = - 
            \frac{1}{|\sI|}\sum_{i \in \sI}\log P(i, a(i))$
            }
            \\
            \colorbox{light-gray}{
            $\gL_\textsc{sCL-PU} 
            = - 
            \frac{1}{|\sI|}\sum_{i\in \sI}
            \Bigg[\1(i\in \sP) \frac{1}{|\sP \setminus i|}\sum\limits_{j \in \sP\setminus i}\log P(i,j)
            + 
            \1(i\in \sU)\frac{1}{|\sU \setminus i|}\sum\limits_{j \in \sU\setminus i}\log P(i,j)
            \Bigg]$ }
            \\ 
            \colorbox{light-gray}{
            $\gL_{\textsc{mCL}}(\lambda) = \lambda \gL_{\textsc{sCL-PU}} 
            + (1 - \lambda)\gL_{\textsc{ssCL}} \; , \; 0 \leq \lambda \leq 1$
            }
            \\
            \colorbox{light-gray}{
            $
            \gL_\textsc{puCL} = \frac{1}{|\sI|}\sum_{i\in \sI} \bigg[\1(\vx_i \in \sP)\ell^i_\textsc{mCL}(1) +  \1(\vx_i \in \sU)\ell^i_\textsc{mCL}(0)\bigg]
            $
            }
            \\
            \colorbox{light-gray}{
            $
            \gL_\textsc{puNCE} = \frac{1}{|\sI|}\sum_{i\in \sI} \bigg[\1(\vx_i \in \sP)\ell^i_\textsc{mCL}(1) +  \1(\vx_i \in \sU)\bigg(\pi \ell^i_\textsc{mCL}(1) + (1-\pi) \ell^i_\textsc{mCL}(0)\bigg)\bigg]
            $
            }
            \\
            \\
            {\em update network parameters} $\mB, \bm{\Gamma}$ to minimize contrastive loss.
        }
        {\bf return:} encoder $g_{\mB}(\cdot)$ and throw away $h_{\bm{\Gamma}}(\cdot)$. 
        \caption{
            \bf Contrastive Representation Learning from PU Data.
        }
        \label{algo:puCL}
\end{algorithm*}
\begin{algorithm*}[t]
        \SetAlgoLined
        \vspace{1em}
        {\bf initialize:}
        PU training data $\gX_\textsc{PU}$; pretrained encoder $g_{\mB}(\cdot): \sR^d \to \sR^k$ via~\cref{algo:puCL}.
        \\
        \\
        {\em obtain representations:}
        \\
        $\gZ_\textsc{P} = \{\vz_i = g_{\mB}(\vx_i) : \forall \vx_i \in \gX_\textsc{P}\}
        \\
        \gZ_{\textsc{U}} = \{\vz_j = g_{\mB}(\vx_j) : \forall \vx_j \in \gX_{\textsc{U}}\}$
        \\
        \\
        {\em initialize pseudo labels :} 
        \\
        $\tilde{y_i} = y_i = 1 : \forall \vz_i\in \gZ_\textsc{P}$ and $\tilde{y_j} = 0 : \forall \vz_j\in \gZ_{\textsc{U}}$
        \\
        \\
        {\em initialize cluster centers: }
        \\
         \colorbox{light-gray}{
        $\mu_{\textsc{P}} = \frac{1}{|\gZ_\textsc{P}|}\sum_{\vz_i \in \gZ_\textsc{P}}\vz_i$ 
        \;,\;
        $\mu_{\textsc{N}} \overset{D(\vx')}{\sim}\gZ_{\textsc{U}}$ 
        \quad
        where, 
        $D(\vx')=\frac{\|\vx' - \mu_{\textsc{P}} \|^2}{\sum_{\vx} \|\vx - \mu_{\textsc{P}}\|^2 }$
        }
        \\
        \\
        \While{not converged}
        {
        {\em update pseudo-label: } 
        \\
        $\forall \vz_i\in \gZ_{\textsc{U}} : \tilde{y_i} = 1$ if $\mu_\textsc{P} = \argmin_{\mu \in \{\mu_\textsc{P}, \mu_\textsc{N}\}}\|\vz_i - \mu\|^2$ else $\tilde{y_i} = 0$
        \\
        $\tilde{\gZ_{\textsc{P}}} = \gZ_\textsc{P} \cup \{\vz_i \in \gZ_\textsc{U} : \tilde{y_i} = 1\} 
        \\
        \tilde{\gZ_{\textsc{N}}} = \{\vz_i \in \gZ_{\textsc{U}} : \tilde{y_i} = 0\}$ 
        \\
        \\
        {\em update cluster centers: } 
        \\
        $\Mu_\textsc{P} = \frac{1}{|\tilde{\gZ_{\textsc{P}}}|}\sum_{\vz_i \in \tilde{\gZ_{\textsc{P}}}}\vz_i$
        \\
        $\Mu_\textsc{N} = \frac{1}{|\tilde{\gZ_{\textsc{N}}}|}\sum_{\vz_i \in \tilde{\gZ_{\textsc{N}}}}\vz_i$
        }
        {\bf return:} 
        $\tilde{\gX}_{\textsc{PU}} = \{ (\vx_i, \tilde{y}_i) :\forall \vx_i \in \gX_{\textsc{PU}} \}$
        \caption{\bf \textsc{puPL}: \underline{P}ositive \underline{U}nlabeled \underline{P}seudo \underline{L}abeling}
        \label{algo:puPL}
\end{algorithm*}
\\
\\
To better navigate the inherent bias-variance trade-off in weakly supervised contrastive learning, we explore a sequence of intermediate strategies grounded in variants of the \textsc{InfoNCE} objective. One approach is {\bf Mixed Contrastive Learning (\textsc{mCL})}~\citep{cui2023rethinking}, which forms a convex combination of the self-supervised loss (\textsc{ssCL}) and the supervised PU variant (\textsc{sCL-PU}). By tuning a mixing coefficient \(\lambda \in [0,1]\), \textsc{mCL} provides an interpolation between the high-variance, unbiased regime of \textsc{ssCL} and the low-variance but potentially biased regime of \textsc{sCL-PU}~\eqref{eq:mcl}. Although flexible, \textsc{mCL} requires careful selection of \(\lambda\) and does not eliminate bias unless tuned precisely. We also consider {\bf Debiased Contrastive Learning (\textsc{dCL})}~\citep{chuang2020debiased}, which corrects for the asymmetry introduced by unlabeled data by reweighting similarity terms according to a known or estimated class prior \(\pi\)~\eqref{eq:dcl}. While \textsc{dCL} can significantly reduce bias in the presence of accurate prior estimates, its performance is sensitive to estimation error—particularly in low-label or high-variance settings where estimating \(\pi\) reliably is difficult.
\\
\\
Motivated by the limitations of these alternatives, we adopt a simple and robust modification to the \textsc{InfoNCE} loss tailored to the PU regime, which we term {\bf Positive-Unlabeled Contrastive Learning (\textsc{puCL})}. It introduces additional attraction terms between pairs of labeled positive samples, while treating unlabeled data entirely via self-supervised augmentation, without making assumptions about their labels. This selective integration of supervision ensures that \textsc{puCL} avoids the bias incurred by \textsc{sCL-PU}, while also reducing the estimator variance compared to \textsc{ssCL}, particularly as the number of labeled positives increases. The resulting objective yields an \textit{unbiased estimator of the population contrastive loss} and exhibits \textit{monotonic variance reduction} as a function of the supervision ratio \(\gamma = n_P / n_U\). Moreover, unlike \textsc{mCL} or \textsc{dCL}, \textsc{puCL} requires no tuning or external prior estimation, and is thus well-suited to PU settings with minimal supervision and limited assumptions. We further support this construction through a bias-variance decomposition, a gradient-based analysis of optimization dynamics, and empirical comparisons across standard PU benchmarks.
\\
\\
When global side information about the data distribution such as class prior \(\pi := \Pr(y = 1)\) is available or can be reliably estimated, it is natural to ask whether this knowledge can be used to further improve contrastive learning under the PU setting. While \textsc{puCL} leverages only the labeled positives and makes no assumptions about the unlabeled data, this formulation overlooks potentially informative constraints imposed by the overall class proportions. Drawing inspiration from classical techniques in importance weighting and probabilistic weak supervision~\citep{elkan2008learning, du2014analysis}, in{\bf ~\cref{sec:puNCE}}, we introduce a prior-aware extension of our objective, which we refer to as {\bf Positive-Unlabeled Noise Contrastive Estimation (\textsc{puNCE})~\eqref{eq:puNCE}}.
\\
\\
The central idea behind \textsc{puNCE} is to treat each unlabeled example as a {\bf probabilistic mixture of positive and negative instances}, with weights given by the known or estimated class prior. Specifically, for each unlabeled sample, we compute contrastive terms as if it were a positive with weight \(\pi\), and a negative with weight \(1 - \pi\). This induces soft positive and negative pairings in expectation, thereby allowing the contrastive objective to make better use of unlabeled data without making hard label assignments.
From a statistical perspective, \textsc{puNCE} can be viewed as an importance-corrected extension of \textsc{puCL} that introduces an {\bf inductive bias} through the prior. When the estimate of \(\pi\) is accurate, this enables the model to better balance attraction and repulsion terms in the loss, resulting in more semantically coherent embeddings, faster convergence, and improved generalization. While, our experiments and ablation studies show that \textsc{puNCE} is robust to moderate errors in \(\pi\) {\bf (\cref{figure:puNCE_pi})}, and consistently outperforms \textsc{puCL} -- especially in low-supervision regimes where the labeled positives alone are insufficient to reliably guide representation learning {\bf (\cref{figure:puNCE_vs_puCL})} -- \textsc{puNCE} comes with the risk of introducing bias when \(\pi\) is miss-specified.
\\
\\
Next, in {\bf \cref{sec:puPL}}, we address the challenge of converting the learned representations into a downstream classifier without access to labeled negatives. A widely used strategy in other semi/weakly-supervised learning is pseudo-labeling~\citep{wang2021pico, bovsnjak2023semppl, zhang2021flexmatch, tsai2022learning, caron2018deep, asano2019self, van2020scan, caron2020unsupervised}, where labels for unlabeled data are inferred based on similarity structure or clustering in the embedding space. However, such approaches remain relatively underexplored in the PU setting~\citep{ yuan2025weighted}, where the lack of labeled negatives introduces additional ambiguity in assigning reliable cluster memberships.
Motivated by this success of label correction in other weakly supervised settings, we propose a simple adaptation of pseudo-labeling for the PU regime, which we call {\bf Positive Unlabeled Pseudo Labeling (\textsc{puPL})}. The key idea is to exploit the geometry of the contrastive embedding space to assign cluster-based pseudo-labels. We implement this by modifying the $k$-means$++$~\citep{arthur2007k} initialization: one cluster is seeded using the centroid of the labeled positives, and the other is selected via $D^2$-weighted sampling over the unlabeled examples. This PU-aware seeding anchors the clustering process and helps disambiguate the assignment of positive and negative pseudo-labels. The resulting labels are then used to train a binary classifier with a standard supervised loss. We describe the full algorithm in~\cref{algo:puPL}.  

\subsection{Contributions} 

Overall, we make several key contributions: 
\begin{itemize}
    \item 
    We systematically examine the behavior of contrastive learning (CL) in the Positive-Unlabeled (PU) setting (\cref{sec:puCL}), beginning with established self-supervised and supervised variants. 
    We investigate several strategies for incorporating weak supervision into CL and characterize the strengths and limitations of each approach in terms of estimator bias, sensitivity to hyper-parameters, and robustness to label sparsity.
    Through this study, we uncover a fundamental bias-variance trade-off (~\cref{th:scl_pu_bias},~\cref{lemma:pucl_unbiased}) that emerges when applying CL under such partial and asymmetric weak supervision.
    
    \item 
    Building on these insights, we propose Positive-Unlabeled Contrastive Learning (\textsc{puCL}), a simple yet effective contrastive objective tailored for the PU setting. \textsc{puCL} treats unlabeled examples via self-supervised learning while using labeled positives to inject structure. We show that \textsc{puCL} is an unbiased and variance-reducing estimator of the population InfoNCE loss, with theoretical guarantees that its performance improves monotonically with the amount of available supervision (\cref{lemma:pucl_unbiased}).

    \item 
    When the class prior \(\pi := \Pr(y = 1)\) is known or can be accurately estimated, we extend \textsc{puCL} to a prior-aware formulation, which we call \textsc{puNCE} (\cref{sec:puNCE}). This loss softly re-weights unlabeled samples as probabilistic mixtures of positives and negatives. Empirically, we find that this inductive bias provided by the class prior can further enhance generalization -- particularly in low-supervision regimes, without requiring hard label assignments or overconfident decisions.

    \item 
    We investigate the role of pseudo-labeling for downstream classification in PU learning (\cref{sec:puPL}), where the absence of labeled negatives poses significant training challenges. To this end, we propose Positive-Unlabeled Pseudo-Labeling (\textsc{puPL}) -- a simple and effective strategy that leverages the structure of the learned contrastive embedding space and introduces a PU-aware modification to the k-means++ initialization. Theoretically, under mild assumptions, we show that \textsc{puPL} achieves an $\gO(1)$ multiplicative error compared to the optimal clustering, and furthermore improves upon the constant factor of standard k-means++ due to its judicious initialization (\cref{th:puPL}). Empirically, we find that \textsc{puPL} enables robust and scalable classification, particularly in low-supervision regimes.

    \item 
    We establish rigorous generalization guarantees for the overall contrastive PU learning framework by leveraging recent tools from augmentation concentration theory. Specifically, we show that the downstream classification error of a non-parametric classifier (e.g., nearest neighbor) is controlled by three key factors: the concentration of the augmentation distribution (captured by parameters $(\sigma, \delta)$) (\cref{def:aug_delta_sigma}), the alignment quality of representations within each class (\cref{lemma:bound_align_puCL}), and the accuracy of pseudo-labeled centroids obtained via \textsc{puPL}. Under mild assumptions (\cref{assumption:aug_non_overlapping}), we prove that the generalization error is bounded by $(1 - \sigma) + R_\epsilon$, where $R_\epsilon$ captures the probability of misalignment between augmented views of the same sample (\cref{th:generalization}).

    \item 
    We conduct extensive experiments across a large set of PU datasets (\cref{sec:exp}), structured around three evaluation setups: (i) a comprehensive PU benchmark comparison against state-of-the-art methods on six standard datasets; (ii) detailed ablations on contrastive learning objectives; and (iii) ablations on downstream classification strategies, evaluating the impact of pseudo-labeling and representation geometry. Our proposed framework -- comprising \textsc{puCL} (when the class prior is unknown) or \textsc{puNCE} (when the prior is available), followed by \textsc{puPL} -- consistently improves over prior art.

\end{itemize}

\section{Related Work}
\label{sec:rel_work}
\subsection*{Positive Unlabeled (PU) Learning.}
Due to the unavailability of negative examples, {\em statistically consistent unbiased risk estimation is generally infeasible}, without imposing strong structural assumptions on $p(\rx)$~\citep{blanchard2010semi, lopuhaa1991breakdown, natarajan2013learning}.
\\
\\
Existing PU learning algorithms primarily differ in the way they handle the semantic annotations of unlabeled examples: 
\\
\\
One set of approaches rely on heuristic based {\em sample selection} where the idea is to identify potential negatives, positives or both samples in the unlabeled set; followed by performing traditional supervised learning using these pseudo-labeled instances in conjunction with available labeled positive data~\citep{liu2002partially, bekker2020learning, luo2021pulns, wei2020mixpul}.
\\
\\
A second set of approaches adopt a {\em re-weighting} strategy, where the unlabeled samples are treated as down-weighted negative examples~\citep{liu2003building, lee2003learning}. 
However, both of these approaches can be difficult to scale, as identifying reliable negatives or finding appropriate weights can be challenging or expensive to tune, especially in deep learning scenarios~\citep{garg2021mixture}. The milestone is~\citep{elkan2008learning}; they additionally assume {\bf a-priori knowledge of class prior} and treat the unlabeled examples as a mixture of positives and negatives. \citep{blanchard2010semi, du2014analysis, kiryo2017positive} build on this idea, and develop {\em statistically consistent and unbiased risk estimators} to perform cost-sensitive learning which has become the backbone of modern large scale PU learning algorithms~\citep{garg2021mixture}. However in practice, $\pi_p^\ast$ is unavailable and must be estimated accurately
\footnote{since inaccurate estimate can lead to significantly poor performance. For example, consider $\pi_p \neq \hat{\pi}_p = 1$ which leads to a degenerate solution i.e. all the examples wrongly being predicted as positives~\citep{chen2020variational}.} 
via a separate Mixture Proportion Estimation (MPE)~\citep{ramaswamy2016mixture, ivanov2020dedpul}, which can add significant computational overhead.
Moreover, even when $\pi_p^\ast$ is available, when supervision is scarce, these approaches can suffer from significant drop in performance or even complete collapse~\citep{chen2020variational} due to the increased variance in risk estimation, which scales as $\sim \gO(1 / n_\textsc{P})$. 
Recent works, alleviate this by combining these estimators with additional techniques. For instance, \citep{chen2021pu} performs self training;~\citep{wei2020mixpul, li2022your} use MixUp to create augmentations with soft labels but can still suffer from similar issues to train the initial teacher model.
\\
Moreover, PU learning is also closely related to other robustness and weakly supervised settings, including learning under distribution shift~\citep{garg2021mixture}, asymmetric label noise~\citep{tanaka2021novel, du2015modelling} and semi-supervised learning~\citep{chen2020big, assran2020supervision, zhou2018brief}.

\subsection*{Contrastive Representation Learning.}
Self-supervised learning has demonstrated superior performances over supervised methods on various benchmarks. Joint-embedding methods~\citep{chen2020simple, Grill2020BootstrapYO, Zbontar2021BarlowTS, Caron2021EmergingPI} are one the most promising approach for self-supervised representation learning where the embeddings are trained to be invariant to distortions. To prevent trivial solutions, a popular method is to apply pulsive force between embeddings from different images, known as contrastive learning. Contrastive loss is shown to be useful in various domains, including natural language processing~\citep{Gao2021SimCSESC}, multimodal learning~\citep{Radford2021LearningTV}. Contrastive loss can also benefit supervised learning~\citep{khosla2020supervised}.

\subsection*{Clustering based Pseudo Labeling.}
Simultaneous clustering and representation learning has gained popularity recently. {\bf DeepCluster}~\citep{caron2018deep} uses off-the-shelf clustering method e.g. kMeans to assign pseudo labels based on cluster membership and subsequently learns the representation using standard CE loss over the pseudo-labels.  However, this standard simultaneous clustering and representation learning framework is often susceptible to degenerate solutions (e.g. trivially assigning all the samples to a single label) even for linear models~\citep{xu2004maximum, joulin2012convex, bach2007diffrac}. {\bf SeLA}~\citep{asano2019self} alleviate this by adding the constraint that the label assignments must partition the data in equally-sized subsets. {\bf Twin contrastive clustering (TCC)} ~\citep{shen2021you}, {\bf SCAN}~\citep{van2020scan}, ~\citep{qian2023stable}, {\bf SwAV}~\citep{caron2020unsupervised, bovsnjak2023semppl} combines ideas from contrastive learning and clustering based representation learning methods to perform simultaneous clusters the data while enforcing consistency between cluster assignments produced for different augmentations of the same image in an online fashion.

\section{Problem Setup}
\label{sec:setup}
\label{sec:problem_setup}
\begin{figure}[t]
\centering
\subfloat
{\includegraphics[width=0.55\textwidth]{plots/problem_setup/pu_problem_visual.pdf}}
\subfloat
{\includegraphics[width=0.45\textwidth]{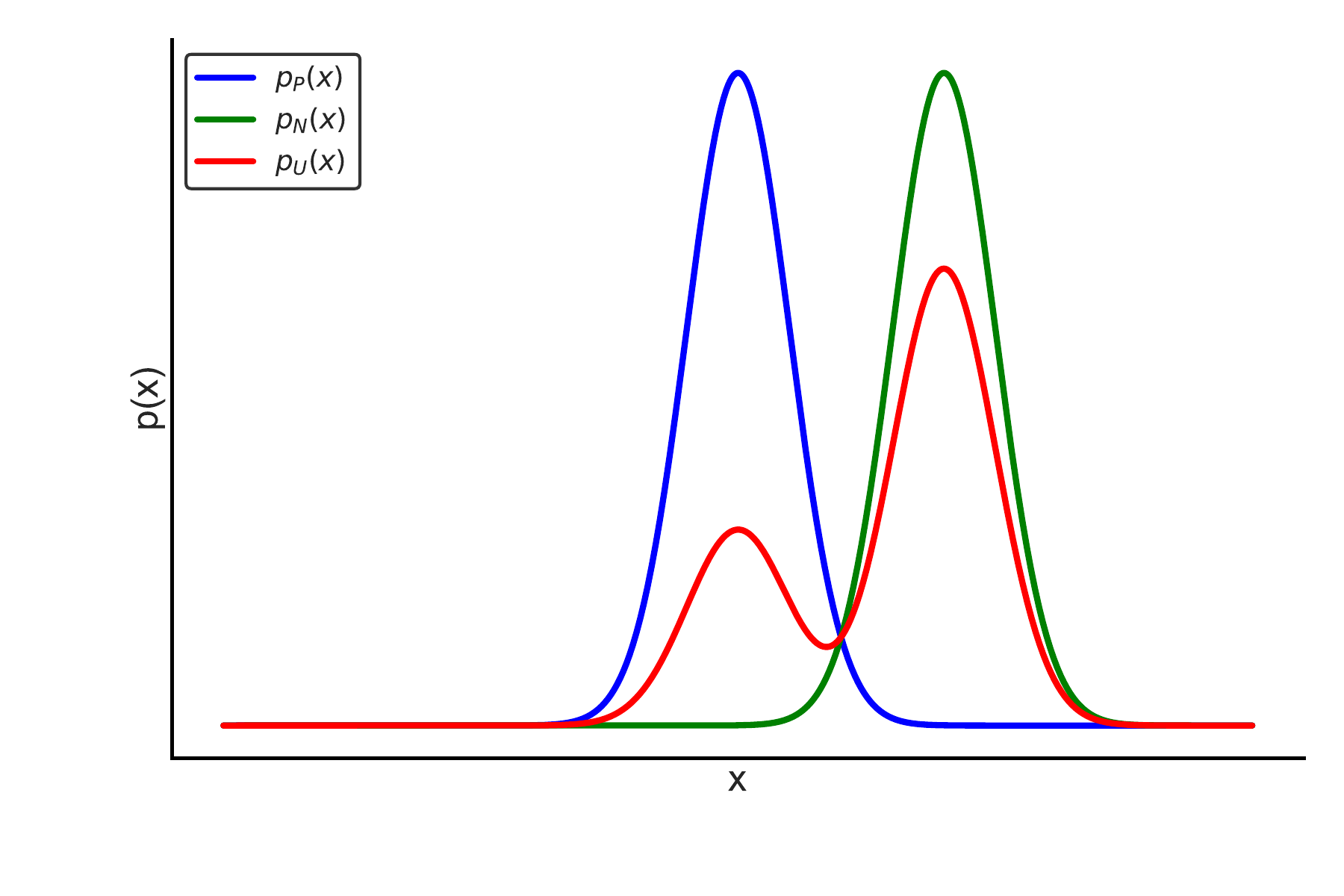}}

\caption{\small{\bf Positive Unlabeled Learning .} 
No negative examples are labeled, a binary classifier needs to be trained using a set of labeled positives $\sim p_\textsc{P}(\rx)$ and a set of unlabeled samples drawn from $\sim p_\textsc{U}(\rx) = \pi_\textsc{P}p_\textsc{P}(\rx)+ (1 - \pi_\textsc{P})p_\textsc{N}(\rx)$ -- the mixture distribution of the positive and negative (unobserved) class marginals .}
\label{fig:problem_setup}
\end{figure}
Formally, let, $\rx \in \sR^d$, $d \in \sN$ and $y \in Y = \{0, 1\}$ denote the underlying input (i.e., feature) and output (label) random variables respectively, and $p(\rx, y)$ denotes the true underlying joint density of $(\rx, y)$. Then, a PU training dataset is composed of a set $\gX_\textsc{P}$ of $n_\textsc{P}$ positively labeled samples and a set $\gX_\textsc{U}$ of $n_\textsc{U}$ unlabeled samples: 
\begin{equation} 
    \gX_\textsc{PU} = \gX_\textsc{P} \cup \gX_\textsc{U}, \;
    \gX_\textsc{P} = \{\vx_i^\textsc{P}\}_{i=1}^{n_\textsc{P}} \sim p(\rx | y=1),\;
    \gX_\textsc{U} = \{\vx_i^\textsc{U} \sim p(\rx)\}_{i=1}^{n_\textsc{U}} 
    \label{eq:pu_dataset}
\end{equation}
In other words, $p(\rx)$ is the mixture distribution of positives and negatives:
\begin{equation}
    p(\rx) = \pi p(\rx|\ry=1) + (1 - \pi) p(\rx|y=0)
    \label{eq:gmm}
\end{equation}
where, $p(\rx|\ry=1) = p_P(\rx)$ and $p(\rx|\ry=0) = p_\textsc{N}(\rx)$ denote the true positive (observed) and negative (unobserved) class marginals and $\pi = p(\ry=1|\rx)$ denotes the {\bf class prior}. 
\\
\\
Since, information about $y$ is unavailable for all samples, it is convenient to define an indicator random variable $s$ such that:
\begin{flalign}
    s = 
    \begin{cases}
        1 \; &\text{if } \vx \text{ is labeled}\\
        0 \; &\text{o/w}
    \end{cases}
\end{flalign}
Now, viewing $\vx, y, s$ as random variables allows us to assume that there is some true underlying distribution $p(\vx, y, s)$ over triplets $(\vx, y, s)$. However, for each sample only $(\vx, s)$ is recorded. The definition of PU dataset~\eqref{eq:pu_dataset} immediately implies the following two results: 
\begin{flalign}
p(y = 1 | s = 1) & = 1
\\
p(s = 1 | y = 0) & = 0
\label{eq.pu.dataset.lemma}
\end{flalign}
This particular setup of how the PU learning dataset is generated - known as the {\bf Case Control Setting}~\citep{bekker2019beyond, blanchard2010semi}; is the most popular and widely studied in the literature. 
While, most of the (theoretical) results in this paper primarily focuses on the case-control setting; it is worth noting that there is another setup called {\bf Single Dataset Setting}, where positive samples are randomly labeled from the data set as opposed to being independent samples from the positive marginal. Thus the unlabeled set is no longer truly representative of the mixture. 
\\
\\
Without the loss of generality, throughout the paper we will assume that the overall classifier $f_\mtheta(\vx): \sR^d \to \sR^{|Y|}$ is parameterized in terms of:
\begin{itemize}
    \item {\bf Encoder}: A feature map $g_{\mB}(\cdot): \sR^d \to \sR^k$ to a lower dimensional manifold referred to as the {\em embedding space} hereafter; and
    \item {\bf Linear Layer}:  $v_{\vv}(\cdot) : \sR^k \to \sR^{|Y|}$, mapping the representations to output (label) space. 
\end{itemize}
Thus, the overall classifier is expressed as the composition:
\begin{equation}
    f_\mtheta(\vx) = v_{\vv} \circ g_{\mB}(\vx)
\end{equation}
The goal in PU learning is to learn $\mtheta = \vv^T \mB$ from $\gX_\textsc{PU} = \gX_P \cup \gX_\textsc{U}$. At a high level, the contrastive framework involves two key steps - (a) learning a mapping function $g_{\mB}(\cdot)$ to a cluster-preserving representation space via contrastive learning and (b) exploit the geometry of the feature space to train the subsequent linear layer $v_{\vv}(\cdot)$. 

\section{Background}
\label{sec:background}
PU Learning~\eqref{eq:pu_dataset} is closely related to the well-studied problem of learning under label noise, where the objective is to robustly train a classifier despite a fraction of training examples being mislabeled. This problem has been extensively studied under both generative and discriminative settings and remains an active area of research~\citep{ghosh2015making, ghosh2017robust, ghosh2021contrastive, wang2019symmetric, zhang2017mixup}.
\subsection{Reduction of PU Learning to Learning with Label Noise}
To illustrate the connection, we frame PU Learning as a {\em special case of binary classification under class-dependent label noise}. Consider, an underlying clean binary dataset $\gX_\textsc{PN}$, i.e. 
\begin{equation} 
    \gX_\textsc{PN} = \gX_{\textsc{P}}^\ast \cup \gX_\textsc{N}, \;
    \gX_\textsc{P}^\ast = \bigg\{\vx_i^\textsc{P}\sim p(\rx | y=1)\bigg\}_{i=1}^{n_\textsc{P}^\ast},\;
    \gX_\textsc{N} = \bigg\{\vx_i^\textsc{N}\sim p(\rx | y=0)\bigg\}_{i=1}^{n_\textsc{N}} 
    \label{eq:supervised_dataset}
\end{equation}
Note that, in the {\bf Label Noise} setting, instead of $\gX_\textsc{PN}$,  a binary classifier needs to be trained from a noisy dataset $\tilde{\gX}_\textsc{PN}$, where the {\bf class conditioned noise rates} i.e. the probability of being mislabeled is $\xi_\textsc{P}$ and $\xi_\textsc{N}$ respectively for the positive and negative samples i.e. 
\begin{equation}
   \tilde{\gX}_\textsc{PN} = \bigg\{(\vx_i, \tilde{y}_i)\bigg\}_{i=1}^{n_\textsc{P}^\ast+n_\textsc{N}}, \; \xi_\textsc{P} = p\bigg(\tilde{y}_i \neq y_i \big| y_i=1\bigg),\; \xi_\textsc{N} = p\bigg(\tilde{y}_i \neq y_i \big| y_i=0\bigg)
\end{equation}
Consider the naive {\bf Disambiguation Free} approach~\citep{li2022your}, where the idea is to pseudo label the PU dataset as follows:
{\em Treat the unlabeled examples as negatives and train an ordinary binary classifier over the pseudo-labeled dataset}. Clearly, since the unlabeled samples (a mixture of positives and negatives) are being pseudo labeled as negative, this is an instance of learning with class dependent label noise:  
\begin{equation} 
    \tilde{\gX}_\textsc{PN} = \gX_\textsc{P} \cup \tilde{\gX}_\textsc{N}, \;
    \gX_\textsc{P} = \bigg\{\vx_i^\textsc{P} \sim p(\rx | y=1) \bigg\}_{i=1}^{n_\textsc{P}},\;
    \tilde{\gX}_\textsc{N} = \bigg\{\vx_i^\textsc{U}\sim p(\rx)\bigg\}_{i=1}^{n_\textsc{U}} 
    \label{eq:disambiguation_free_approach}
\end{equation}
where, the noise rates are: 
\begin{equation}
    E(\xi_\textsc{P}) =  \frac{\pi}{\gamma + \pi} \; \text{and }   \xi_\textsc{N} = 0
\end{equation}
where, $\gamma = \frac{n_\textsc{P}}{n_\textsc{U}}$ and $\pi = p(y=1 | \vx)$ are training distribution dependent parameters. 
\\
\\
Under the standard Empirical Risk Minimization (ERM) framework, the goal is to robustly estimate the true risk from noisy data i.e. for some loss function $\ell(\cdot, \cdot)$ with high probability, we seek: 
\begin{equation}
    \Delta = \bigg\| \hat{\gR}(\mtheta) - \gR(\mtheta^\ast) \bigg\| = \E \bigg\|\ell\bigg(f_{\mtheta}(\vx), \tilde{y}\bigg)- \ell\bigg(f_{\mtheta^\ast}(\vx), y\bigg) \bigg\| \leq \epsilon
\end{equation}
A common way to measure the resilience of an estimator against corruption is via breakdown point analysis~\citep{donoho1983notion, lopuhaa1991breakdown, acharya2022robust}.    
\begin{definition}[\bf Breakdown point]
Breakdown point $\zeta_T$ of an estimator $T(\cdot)$ is simply defined as the smallest fraction of corruption $\psi$ that must be introduced to cause an estimator to break i.e.
\begin{flalign}
    \zeta_T = \inf \left\{ \psi \geq 0 : \sup_{\mathcal{D}_{\mathcal{B}}} \bigg\| T(\gD) - T(\gD_\gG)\bigg\| = \infty \right\}
\end{flalign}
where, $\gD_\gG$ is the uncontaminated (clean) dataset, $\gD_\gB$ represents the corrupted subset of data i.e. $\gD = \gD_\gG \bigcup \gD_\gB$, and $\psi = {|\gD_\gB|}/{|\gD|}$ denotes the fraction of corrupted samples. $T(\cdot)$ is said to achieve the {\bf optimal breakdown point $\zeta^\star_T=1/2$} if it remains bounded $\forall \; 0 \leq \psi < 1/2$. 
\end{definition}
It can be shown that:
\begin{lemma}
\label{lemma:label_noise}
    Consider learning a binary classifier (\textsc{P} vs \textsc{N}) in presence of class-dependent label noise with noise rates $E(\xi_\textsc{P}) =  \frac{\pi}{\gamma + \pi},\; \xi_\textsc{N} = 0$. Without additional distributional assumption, no robust estimator can guarantee bounded risk estimate if: 
    \[
    \gamma \leq 2 \pi - 1
    \]
    where $\gamma = \frac{n_\textsc{P}}{n_\textsc{U}}$ and $\pi = p(y=1|\vx)$ denotes the underlying class prior. 
\end{lemma} 

\begin{proof}
    This result (\cref{lemma:label_noise}) follows from using the fact that for any estimator $0 \leq \psi < \frac{1}{2}$~\citep{lopuhaa1991breakdown, minsker2015geometric, cohen2016geometric, acharya2022robust} i.e. for robust estimation to be possible, the corruption fraction  $\alpha =  \frac{\pi}{\gamma + 1} < \frac{1}{2}$. 
\end{proof}
In summary, {\em this indicates that PU Learning cannot be solved by standard off-the-shelf label noise robust algorithms and specialized algorithms need to be designed}. 

\subsection{Cost Sensitive PU Learning}
While, without additional assumptions the PU Learning problem is generally infeasible, a natural question remains:
\begin{quote}
    {\em Is it still possible to overcome the limitations imposed by such high noise rates, for instance by exploiting additional side information?}
\end{quote}
Remarkably, by assuming {\bf additional knowledge of the true class prior} $\pi = p(\ry=1 |\rx)$, state-of-the-art (SOTA) cost-sensitive PU learning algorithms tackle this by forming an unbiased estimate of the true risk from PU data~\citep{blanchard2010semi}. 
\\
\\
An unbiased estimate of the negative risk in cost‐sensitive PU learning can be derived via a straightforward application of Bayes’ Rule~\citep{elkan2008learning}. Specifically,
\begin{flalign}
    \hat{R}_\textsc{N}^-(f_\mtheta) = \frac{1}{(1 - \pi)}\Bigg[\hat{R}_\textsc{U}^-(f_\mtheta) - \pi \hat{R}_\textsc{P}^-(f_\mtheta)\Bigg]
\end{flalign}
Substituting this into the overall risk yields the well‐known \textsc{uPU} estimator~\citep{blanchard2010semi, du2014analysis} for $R_\textsc{PN}$:
\begin{flalign}
    \hat{R}_\textsc{uPU}(f_\mtheta) = \pi \hat{R}_\textsc{P}^+(f_\mtheta) + \Bigg[\hat{R}_\textsc{U}^-(f_\mtheta) - \pi \hat{R}_\textsc{P}^-(f_\mtheta)\Bigg]
    \label{eq:uPU}
\end{flalign}
where, we have denoted the empirical estimates computed over PU dataset as: 
\begin{flalign*}
    \hat{R}_\textsc{P}^+(f_\mtheta) = \frac{1}{n_\textsc{P}}\sum_{i=1}^{n_\textsc{P}} \ell\bigg(f_\mtheta(\rx_i^\textsc{P}), 1\bigg)
    ,\quad
    \hat{R}_\textsc{N}^-(f_\mtheta) =\frac{1}{n_\textsc{N}}\sum_{i=1}^{n_\textsc{N}} \ell\bigg(f_\mtheta(\rx_i^\textsc{N}), 0\bigg)
    \\
    \hat{R}_\textsc{P}^-(f_\mtheta) = \frac{1}{n_\textsc{P}}\sum_{i=1}^{n_\textsc{P}} \ell\bigg(f_\mtheta(\rx_i^\textsc{P}), 0\bigg)
    ,\quad
    \hat{R}_\textsc{u}^-(f_\mtheta) =\frac{1}{n_\textsc{U}}\sum_{i=1}^{n_\textsc{U}} \ell\bigg(f_\mtheta(\rx_i^\textsc{U}), 0\bigg)
\end{flalign*}
where, $\ell(\cdot, \cdot): Y\times Y \to \sR$ is the classification loss.
\\
\\
In practice, further improvements are achieved by clipping the estimated negative risk. This approach, known as \textsc{nnPU}~\citep{kiryo2017positive}, modifies the estimator as follows:
\begin{equation}
    \hat{R}_\textsc{nnPU}(f_\mtheta) = \pi \hat{R}_\textsc{P}^+(f_\mtheta) + \max \bigg\{ 0, \;\hat{R}_\textsc{U}^-(f_\mtheta) - \pi \hat{R}_\textsc{P}^-(f_\mtheta) \bigg\}
    \label{eq:nnPU}
\end{equation}
This clipped loss has become the the de-facto approach to solve PU problems in practical settings and forms a strong theoretical baseline for training the downstream PU classifier. 

\subsection{Limitations of Cost Sensitive Approaches}
\label{sec:nnPU_limitation}
Despite their wide adoption in industry, these estimators exhibit notable theoretical and practical limitations, as detailed below:
\paragraph{Class Prior Estimate. } The success of these estimators is based on knowledge of the oracle class prior to $\pi^\ast$ for their success. However, in practice, the true $\pi$ is often not available and must be estimated $\hat{\pi}$ from the data through a separate Mixture Proportion Estimation (MPE) subroutine~\citep{garg2021mixture, ivanov2020dedpul, ramaswamy2016mixture, yao2021rethinking, christoffel2016class, chen2020variational, niu2016theoretical}.
Moreover, an error in class prior estimate $\|\hat{\pi} - \pi^\ast\| \leq \xi$ results in an estimation bias $\sim \gO(\xi)$:
\begin{flalign}
    \bigg\|\hat{R}_\textsc{uPU}(f_\mtheta, \pi) - \hat{R}_\textsc{uPU}(f_\mtheta, \hat{\pi})\bigg\|
    \leq \xi \max_{f_\mtheta}\bigg\|\hat{R}_\textsc{P}^+(f_\mtheta) - \hat{R}_\textsc{P}^-(f_\mtheta)\bigg\|
\end{flalign}
Thus, even small approximation error in estimating the class prior can lead to notable degradation in the overall performance of the estimators resulting in poor generalization, slower convergence or both~\citep{yao2021rethinking, garg2021mixture} as also validated by our experiments in~\cref{fig:puCL_convergence}. Furthermore, obtaining highly accurate approximations with the MPE subroutine often entails considerable computational overhead. This increased computational demand can become a bottleneck, particularly in scenarios where hardware resources are limited or real-time processing is required. As a result, the practicality of using such subroutines may be compromised in resource-constrained environments, thereby necessitating the development of more efficient estimation techniques or alternative strategies to mitigate the associated computational costs.

\paragraph{Low Supervision Regime. } While these estimators are significantly more robust than the vanilla supervised approach, our experiments (~\cref{fig:puPL_gmm}) suggest that they might produce decision boundaries that are not closely aligned with the true decision boundary especially as $\gamma$ becomes smaller~\citep{kiryo2017positive, du2014analysis}. Note that, when available supervision is limited i.e. when $\gamma$ is small, the estimates $\hat{R}_p^+$ and $\hat{R}_p^-$ suffer from increased variance resulting in increase variance of the overall estimator $\sim \gO(\frac{1}{n_\textsc{P}})$. For sufficiently small $\gamma$ these estimators are likely result in poor performance due to large variance.
\\
\\
Some recent works alleviate this by combining these estimators with additional techniques. For instance, \citep{chen2021pu} performs self training;~\citep{wei2020mixpul, li2022your} use MixUp to create augmentations with soft labels but can still suffer from similar issues to train the initial teacher model. 

\section{Contrastive Representation Learning from PU Data}
\label{sec:puCL}
Central to the contrastive approach is the construction of a representation space that fosters the proximity of semantically related instances while enforcing the separation of dissimilar ones. 
One way to obtain such a representation space is via pretext-invariant representation learning where the representations $\vz_i = g_{\mB}(\vx_i) \in \sR^k$ are {\bf trained to be invariant to label-preserving distortions aka augmentations}~\citep{wu2018unsupervised, misra2020self} (\cref{def:augmentation}).
\begin{definition}[Invariance under Transformation]
\label{def:augmentation}
    Consider the data set $\vx \in \gX, y \in Y$ with the underlying ground truth labeling mechanism $y = \gY(\vx) \in Y$. The parameterized representation function $f_{\mtheta}(\cdot)$ is said to be invariant under transformation $t: \gX \to \gX$ that does not change the ground truth label, i.e. $\gY(t(\vx)) = \gY(\vx)$ if $f_{\mtheta}(t(\vx)) \approx f_{\mtheta}(\vx)$.
    \begin{flalign}
        \gY(t(\vx)) = \gY(\vx) \;{\text if}\; f_{\mtheta}(t(\vx)) \approx f_{\mtheta}(\vx).
    \end{flalign}
\end{definition}
To avoid trivial solutions~\citep{tian2021understanding}, a popular trick is to apply an additional repulsive force between the embeddings of semantically dissimilar images, known as contrastive learning~\citep{chopra2005learning, schroff2015facenet,sohn2016improved}. 
\\
\\
In particular, we study minimizing variants of InfoNCE family of losses~\citep{oord2018representation} -- a popular contrastive objective based on the idea of {\em Noise Contrastive Estimation} (NCE), a method of estimating the likelihood of a model by comparing it to a set of noise samples~\citep{gutmann2010noise}:
\begin{equation}
    \gL_\textsc{CL}^\ast =
    \mathop{-\E}_{(\vx_i, y_i) \sim p(\vx, \vy)}
    \mathop{\E}_{\substack{\vx_j \sim p(\vx | y_j = y_i) 
    \\ \{\vx_k\}_{k=1}^N \sim p(\vx | y_k \neq y_i)}}
    \Bigg[ \vz_i \boldsymbol \cdot \vz_j - \log Z(\vz_i) \Bigg]
    \label{eq:asymptotic_scl}
\end{equation}
where, the operator $\boldsymbol\cdot$ and the partition function $Z(\vz_i)$ are defined as: 
\begin{equation}
    \vz_i \boldsymbol \cdot \vz_j = \frac{1}{\tau}\frac{\vz_i^T\vz_j}{\|\vz_i\|\|\vz_j\|} \;,\; Z(\vz_i) = \exp(\vz_i\boldsymbol\cdot \vz_j) + \sum_{k=1}^N \exp(\vz_i\boldsymbol\cdot \vz_k).
\end{equation}
Intuitively, the loss projects the representation vectors onto the hypersphere $\gS^{k-1}_1 =\{\vz \in \sR^k : \|\vz\|=\frac{1}{\tau}\}$ and aims to minimize the angular distance between similar samples while maximizing the angular distance between dissimilar ones. $\tau \in \sR^+$ is a hyper-parameter that balances the spread of the representations on the hypersphere~\citep{wang2020understanding}. 

\subsection{Self Supervised Contrastive Learning (\textsc{ssCL})}
In the unsupervised setting, since {\em identifying similar and dissimilar example pairs from the appropriate class conditionals is intractable};
different augmentations of the same image are treated as similar, while the rest are considered as dissimilar pairs.  
\\
\\
While, several representation learning frameworks~\citep{caron2020unsupervised, Grill2020BootstrapYO, he2020momentum, Zbontar2021BarlowTS} have been proposed to realize the infoNCE family of losses in the finite sample setting, in this paper we focus on the SimCLR (Simple Contrastive Learning) framework~\citep{chen2020simple} for simplicity.
\\
\\
In particular, for any random batch of samples $\gD$, the corresponding {\bf multi-viewed batch} is constructed by obtaining two augmentations (correlated views) of each sample: 
\begin{flalign}
    \gD=\{\vx_i\}_{i=1}^b,\;, \tilde{\gD} = \{t(\vx_i), t'(\vx_i) \}_{i=1}^b
    \label{eq:mv_dataset}
\end{flalign}
where $t(\cdot),t^\prime(\cdot): \sR^d\to\sR^d$ are stochastic label preserving transformations (\cref{def:augmentation}), such as color distortion, cropping, flipping, etc~\footnote{While, for simplicity, in this paper we will only construct one augmentation pair per sample; it is also common to compute the expectation over multiple augmentation pairs~\citep{tian2020contrastive}}. 
\\
\\
Furthermore, following~\citep{saunshi2019theoretical, tosh2021contrastive}, for ease of exposition we make the following simplifying assumption:
\begin{assumption}[Transformation Independence]
\label{assumption:aug_independence}
Let \(\gT\) be a family of stochastic, label-preserving transformations, i.e., for every \(t\in \gT\) and \(\vx\in\gX\), \(\gY\big(t(\vx)\big)=\gY(\vx)\). 
Then, given a sample \((\vx_i, \gY(\vx_i))\), two independent augmented samples $\tilde{\vx}_i, \tilde{\vx}_i'$
are independently and identically distributed draws from the underlying class marginal. i.e.
\begin{flalign}
    \tilde{\vx}_i,\, \tilde{\vx}_i' \stackrel{\text{i.i.d.}}{\sim} p\Big(\vx \mid \gY(\vx_i)\Big).
\end{flalign}
where, 
\(
t(\cdot), t'(\cdot)\in \gT, \quad \tilde{\vx}_i = t(\vx_i), \quad \tilde{\vx}_i' = t'(\vx_i).
\)
\end{assumption}
However, note that,~\cref{assumption:aug_independence} is made solely for the clarity of exposition; in \cref{sec:generalization}, we relax this assumption and derive generalization guarantees by analyzing the concentration properties of the augmentation sets~\citep{huang2023towards}.
\\
\\
To facilitate the subsequent discussion, let us introduce the {\bf index set} $\sI \equiv \{1, \dots, 2b\}$ corresponding to the elements of the multi-viewed batch. For the augmentation indexed by $i \in \sI$, the other augmentation originating from the same source sample is indexed as $a(i)$. 
\\
\\
Then, \textsc{ssCL} minimizes the following objective~\citep{chen2020simple} :
\begin{equation}
    \gL_{\textsc{ssCL}}
    = - 
    \frac{1}{|\sI|}\sum_{i \in \sI}
    \Biggl[\vz_i \boldsymbol \cdot \vz_{a(i)} - \log Z(\vz_i) \Biggr]
\label{eq:sscl}
\end{equation}
where, $Z(\vz_i) = \sum_{j \in \sI}\1(j \neq i) \exp(\vz_i\boldsymbol\cdot \vz_j)$ is the {\em finite-sample approximation of the partition function within the batch}. 
\\
\\
Indeed, if~\cref{assumption:aug_independence} holds, $\gL_{\textsc{ssCL}}$ is an unbiased estimator of $\gL_\textsc{CL}^\ast$~\eqref{eq:asymptotic_scl}. 
Simply put, within each batch of samples, \textsc{ssCL} approximates $\gL_\textsc{CL}^\ast$~\eqref{eq:asymptotic_scl} -- driving the representations of augmented views of the same image (positive pairs) to converge to low-energy wells, while simultaneously pushing the representations of different images (negative pairs) to be separated by high-energy barriers~\citep{ranzato2007unified}. 

\paragraph{Projection Network.} 
In practice, rather than computing the loss over the encoder output, i.e. $\vz_i = g_{\mB}(\vx_i)$; it is beneficial~\citep{chen2020simple, Zbontar2021BarlowTS, Grill2020BootstrapYO, khosla2020supervised,  assran2020supervision} to feed it through a small {\em non-linear projection network} $h_{\mGamma}(\cdot): \sR^k \to \sR^p$ to obtain a lower-dimensional representation $\vz_i = h_{\mGamma} \circ g_{\mB}(\vx_i)\in \sR^p$~\citep{chen2020simple,schroff2015facenet}. Note that {\em the $h_{\mGamma}(\cdot)$ is only used during training and is discarded during inference}. 

\subsection*{Incorporating Supervision.}
Despite its ability to learn robust representations,
\textsc{ssCL} is completely agnostic to semantic annotations, hindering its ability to benefit from additional supervision, especially when such supervision is reliable. This lack of semantic guidance often leads to inferior visual representations compared to fully supervised approaches~\citep{he2020momentum, kolesnikov2019revisiting}. Motivated by these observations, we ask the question:
\begin{quote}
    {\em How to design a contrastive loss that can leverage the available weak supervision in PU learning (in the form of labeled positive examples) in an efficient manner, to learn more discriminative representations compared to} $\gL_{\textsc{ssCL}}$.
\end{quote}
In the remainder of this section, we explore several strategies to integrate additional weak supervision into the contrastive framework. We discuss different modifications to the loss function, analyze their theoretical implications, and evaluate their empirical benefits.

\subsection{Supervised Contrastive Learning (\textsc{sCL})}
In the fully supervised setting~\eqref{eq:supervised_dataset}, \textsc{sCL}~\citep{khosla2020supervised} addresses this issue by utilizing the semantic annotations to guide the choice of similar and dissimilar pairs, resulting in significantly better representations than \textsc{ssCL}. 
\begin{align}
    \gL_{\textsc{sCL}}  
    =
    - \frac{1}{|\sI|}\sum_{i\in \sI}
    \Bigg[
    \Bigg(
    \frac{\1(i \in \sP^\ast)}{|\sP^\ast \setminus i|}\sum\limits_{j \in \sP\setminus i}\vz_i \boldsymbol \cdot \vz_j
    + \frac{\1(i\in \sN) }{|\sN\setminus i|}\sum\limits_{j \in \sN\setminus i}\vz_i \boldsymbol \cdot \vz_j
    \Bigg)
    - \log Z(\vz_i)
    \Bigg] 
\label{eq:scl}
\end{align}
where, $\sP^\ast$ and $\sN$ denote the subset of indices of the samples in the {\bf multi-viewed batch} $\tilde{\gD}$ that are {\bf labeled} positives and negatives respectively. i.e. 
\begin{flalign}
    \sP^\ast= \bigg\{i \in \sI : \ry_i =1\bigg\}, \; \sN = \bigg\{i \in \sI : \ry_i = 0 \bigg\}
\end{flalign}
The indicator function $\1(\cdot)$ selects the appropriate term depending on whether the anchor is a labeled positive or labeled negative sample. 
\\
\\
Clearly, under~\cref{assumption:aug_independence}, $\gL_{\textsc{sCL}}$~\eqref{eq:scl} is a consistent and unbiased estimator of the asymptotic objective $\gL_{\textsc{CL}}^\ast$~\eqref{eq:asymptotic_scl}. Furthermore, since the expected similarity of positive pairs is computed on all available samples of the same marginal class as anchor, this loss enjoys a lower variance, compared to its self-supervised counterpart~\eqref{eq:sscl}. This variance reduction $\sim \gO(1/|\sI|^2)$ often results in significant empirical gains.
\begin{figure}[t] 
\centering
\setlength{\abovecaptionskip}{0pt} 
\subfloat[$\kappa = \frac{\pi(1 -\pi)}{1 + \gamma}$]{
\includegraphics[width=0.35\textwidth]{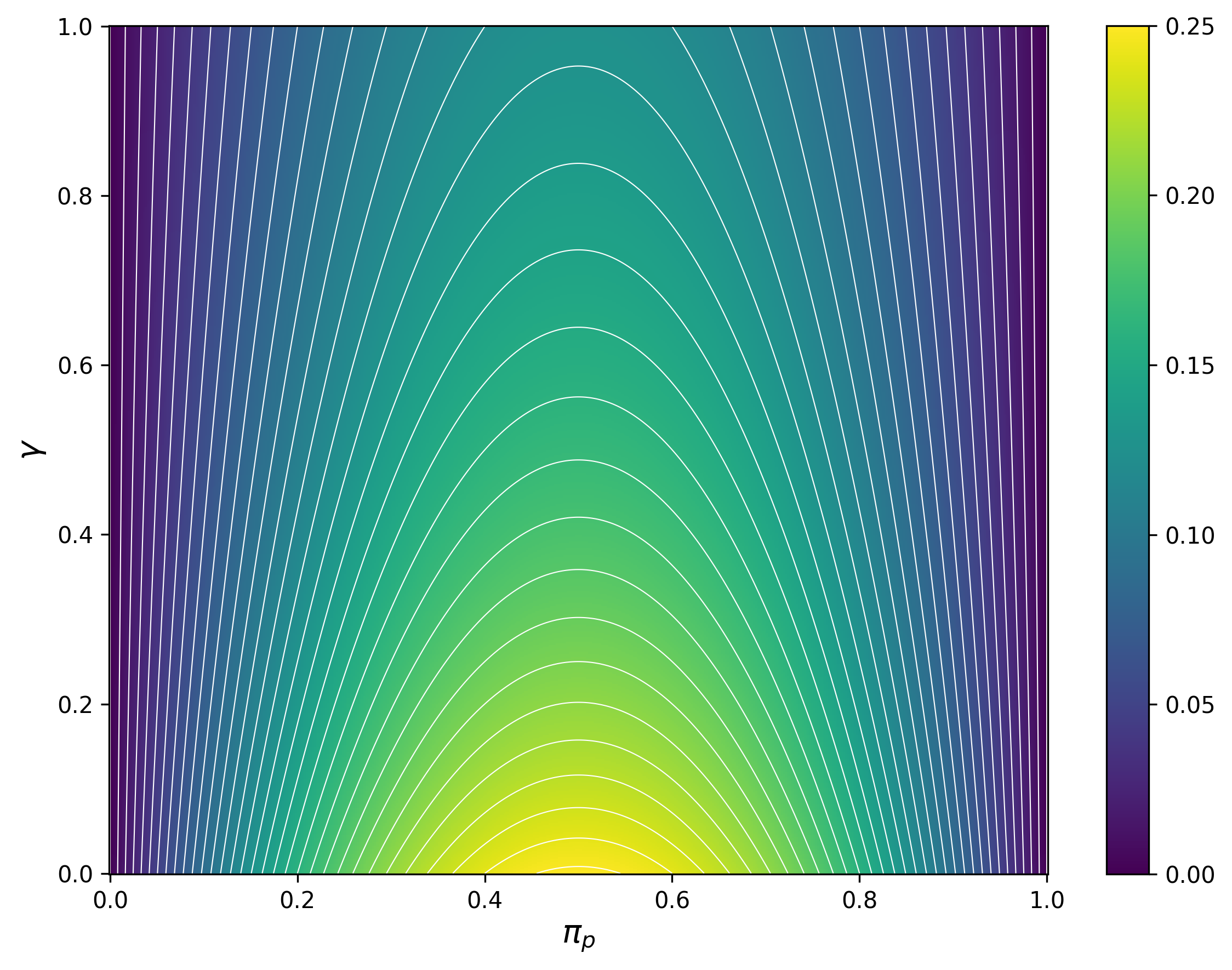}
}
\subfloat[\bf Generalization ($\kappa$)]
{\includegraphics[width=0.35\textwidth]{plots/kappa_variation/kappa_var_3D.pdf}}
\\
\subfloat[\bf Generalization ($\kappa$)]
{\includegraphics[width=0.9\textwidth]{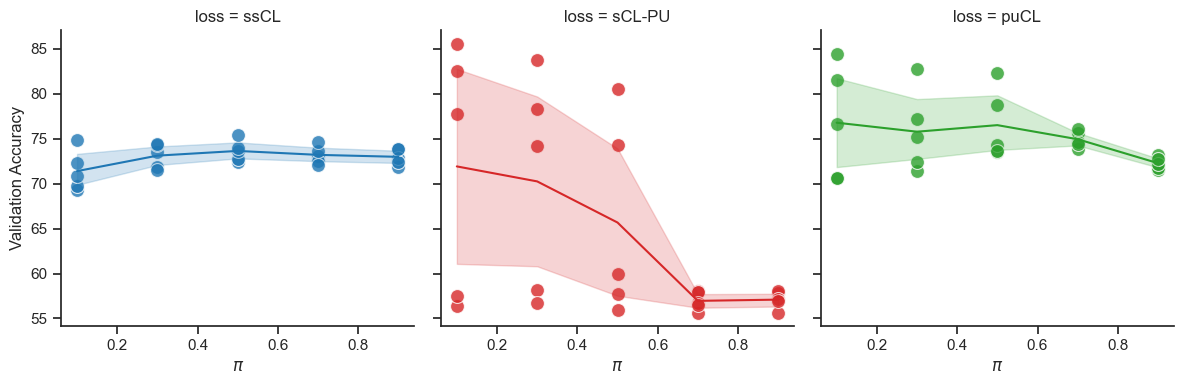}}
\caption{\footnotesize {\bf (Ablations over Varying $\kappa$)} ResNet-34 trained on ImageNet-I (a) Variation of $\kappa$ w.r.t class prior ($\pi_p$) and PU supervision ($\gamma$) (b) Generalization performance of contrastive objectives with varying $\kappa$. (c) 2D visualization of (b) across each loss.}
\label{fig:kappa_var_3D}
\end{figure}
\\
\\
However, unfortunately in PU learning~\eqref{eq:pu_dataset}, {\em since $\sN$ is intractable (as no labeled negatives are available), and only a subset of labeled positives $\sP \subseteq \sP^\ast$ is available, it is non-trivial to extend \textsc{sCL} in this setting}. 
Naive disambiguation-free adaptation~\eqref{eq:disambiguation_free_approach} of \textsc{sCL}~\eqref{eq:scl} gives: 
\begin{equation}
\begin{aligned}
    \gL_\textsc{sCL-PU} 
    = - 
    \frac{1}{|\sI|}\sum_{i\in \sI}
    \Bigg[
    \Bigg(
    \frac{\1(i\in \sP)}{|\sP \setminus i|}\sum\limits_{j \in \sP\setminus i}\vz_i \boldsymbol \cdot \vz_j
    + \frac{\1(i \in \sU) }{|\sU\setminus i|}\sum\limits_{j \in \sU\setminus i}\vz_i \boldsymbol \cdot \vz_j
    \Bigg)
    - \log Z(\vz_i)
    \Bigg] 
\end{aligned}
\label{eq:scl_pu}
\end{equation}
$\sP$ and $\sU$ denote the subset of indices in $\tilde{\gD}$ that are labeled positive and unlabeled respectively:
\begin{flalign}
    \sP= \bigg\{i \in \sI \;\big|\; \vx_i \in \gX_\textsc{P}, \rs_i = \ry_i =1\bigg\}, \; \sU = \bigg\{i \in \sI \;\big|\; \vx_i \in \gX_\textsc{U}, \rs_i = 0 \bigg\}
    \label{eq:pu_dataset_batch}
\end{flalign}
Using linearity of expectation, we can show that $\gL_\textsc{sCL-PU}$~\eqref{eq:scl_pu} suffers from statistical bias in estimating $\gL_{\textsc{CL}}^\ast$. This bias becomes increasingly pronounced as the level of available supervision decreases as characterized in~\cref{th:scl_pu_bias}.
\begin{theorem}
\label{th:scl_pu_bias}
Under~\cref{assumption:aug_independence}, $\gL_{\textsc{sCL-pu}}$~\eqref{eq:scl_pu} is a biased estimator of the population risk $\gL_{\textsc{CL}}^\ast$~\eqref{eq:asymptotic_scl} characterized as follows:
\begin{flalign}
    \mathop{\E}_{\gX_{\textsc{PU}}} 
    \bigg[ 
    \gL_{\textsc{sCL-PU}}
    \bigg]
    - \gL_{\textsc{CL}}^\ast
     &= 2 \kappa
    \bigg(\rho_{intra} - \rho_{inter} \bigg)
\end{flalign}
where, $\rho_{intra}=\mathop{\E}_{\substack{\vx_i,\vx_j \sim p(\vx | y_i=y_j)}}\big( \vz_i \boldsymbol \cdot \vz_j \big)$ 
captures the concentration of embeddings of samples from same latent class marginals and 
$\rho_{inter} = \mathop{\E}_{\substack{\vx_i,\vx_j \sim p(\vx | \ry_i \neq \ry_j)}}\big( \vz_i \boldsymbol \cdot \vz_j \big)$ 
captures the expected proximity between embeddings of dissimilar samples. $\gamma = n_{\textsc{P}}/n_{\textsc{U}}$ and $\kappa = \pi (1 - \pi)/(1 + \gamma)$ are dataset dependent constants.
\end{theorem}
\cref{th:scl_pu_bias} reveals that the bias -- that stems from the implicit use of the unlabeled set as a surrogate negative class -- scales with $\Delta_\rho = (\rho_{intra} - \rho_{inter})$, quantifying the separability of the representation manifold and  $\kappa_\textsc{PU}$, a dataset specific parameter. 
\\
\\
$\Delta_\rho$ quantifies the separability of the representation space: low values indicate poor clustering, while high values reflect strong intra-class cohesion and inter-class separation. The bias penalizes settings where negative pairs (across classes) are not well-separated, or where same-class pairs are not tightly clustered -- both of which are common in low-data regimes or early in training.
Furthermore, even when the representation space is well clustered i.e. $\Delta_\rho$ is small, the overall bias can still be significant if $\kappa$ is sufficiently large -- when $\gamma \ll 1$, i.e., when labeled positives are scarce. In the extreme case where $\gamma \to \infty$, all examples are labeled and the bias vanishes. These theoretical observation are also supported via ablations across different dataset conditions as described in~\cref{sec:exp}. 
\\
\\
In summary,
{\em  While $\gL_{\textsc{sCL-PU}}$ suffers from significant drop in generalization performance in the low-supervision regime; it still results in significant improvements over the unsupervised $\gL_{\textsc{ssCL}}$ when sufficient labeled positives are available. This indicates a {\bf bias-variance trade-off} that can be further exploited to arrive at an improved loss.}

\subsection{Mixed Contrastive Learning (\textsc{mCL})}
\begin{figure} 
\small 
\centering
\setlength{\abovecaptionskip}{0pt} 
\subfloat[\bf \textsc{mCL}($\lambda$)]
{\includegraphics[width=0.46\textwidth]{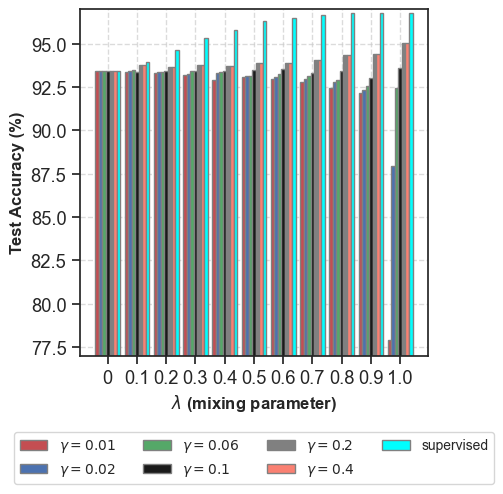}}
\subfloat[\bf Generalization ($\gamma$)]{
\includegraphics[width=0.4\textwidth]{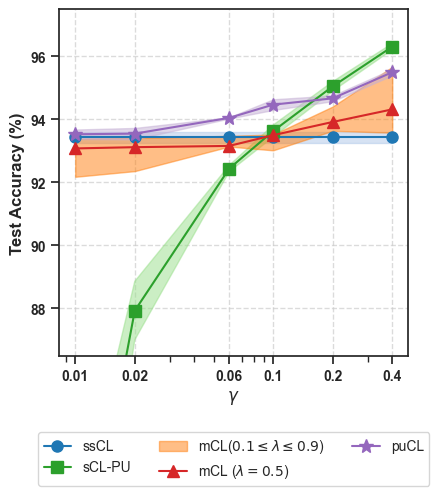}
}
\caption{\footnotesize {\bf Mixed Contrastive Learning} ResNet-18 trained on CIFAR-III (vehicle vs animal). (a) Variation of $\kappa$ w.r.t class prior ($\pi_p$) and PU supervision ($\gamma$) (b) Generalization performance of contrastive objectives with varying $\kappa$.}
\label{fig:mcl}
\end{figure}
A natural approach to interpolating between the robustness of self-supervised contrastive learning (\textsc{ssCL}) and the generalization benefits of positive-unlabeled supervised contrastive learning (\textsc{sCL-PU}) is to consider a convex combination of their respective objectives. This hybrid formulation, which we refer to as Mixed Contrastive Learning (\textsc{mCL}), is motivated by similar strategies shown to be effective in learning under label noise~\citep{cui2023rethinking}, and is well-suited to the PU learning regime. (\textsc{mCL}) is defined as follows: 
\begin{equation}
    \gL_{\textsc{mCL}}(\lambda) = \lambda \gL_{\textsc{sCL-PU}} 
    + (1 - \lambda)\gL_{\textsc{ssCL}} \; , \; 0 \leq \lambda \leq 1
    \label{eq:mcl}
\end{equation} 
Here, $\gL_{\textsc{ssCL}}$ captures unsupervised learning by encouraging consistency across augmented views of the same input (i.e., enforcing similarity between $\vz_i$ and $\vz_{a(i)}$ for all $i \in \sI$). In contrast, $\gL_{\textsc{sCL-PU}}$ injects supervision using available positive labels, guiding the model to capture semantically meaningful structures within the representation space.
However, in the PU learning setting, the supervision signal provided to $\gL_{\textsc{sCL-PU}}$ is noisy, as unlabeled examples may belong to either class. Consequently, the performance of \textsc{mCL} is sensitive to the mixing coefficient $\lambda$, which governs the trade-off between the two objectives. 
\\
\\
This sensitivity reflects a classic bias-variance trade-off: stronger reliance on supervision (larger $\lambda$) can improve generalization when labels are reliable, but may introduce bias under noisy or sparse supervision.
We empirically validate this trade-off via extensive ablation studies across a range of $\lambda$ values and supervision levels (quantified by the labeled-to-unlabeled ratio $\gamma = n_\textsc{P}/n_\textsc{U}$). As shown in Figure~\ref{fig:mcl}, we observe that:
{\em when available supervision is limited i.e. for small values of $\gamma$ a smaller value of $\lambda$ (i.e. less reliance on supervised part of the loss) is preferred. Conversely, for larger values of $\gamma$ larger contribution from the supervised counterpart is necessary suggesting the need for more careful mixing in the PU setting.}
\\
\\
These observations underscore a key limitation of the current formulation: \textsc{mCL} applies a fixed mixing coefficient $\lambda$ uniformly across all training instances. However, in practice, different samples may vary in the reliability of their supervision signal. A more adaptive approach—e.g., assigning instance-specific weights—could better capture this heterogeneity and improve learning in the PU setting.

\subsection{Positive Unlabeled Contrastive Learning (\textsc{puCL}).}
\begin{figure*}[t]
\small 
\centering
\subfloat
[\bf \textsc{ssCL}]
{\includegraphics[width=0.25\textwidth]{plots/tsne/imagenet-ssCL.pdf}}
\subfloat
[\bf \textsc{puCL}($\gamma=0.2$)]
{\includegraphics[width=0.25\textwidth]{plots/tsne/imagenet-gamma=0.2.pdf}}
\subfloat
[\bf \textsc{puCL}($\gamma=0.5$)]
{\includegraphics[width=0.25\textwidth]{plots/tsne/imagenet-gamma=0.5.pdf}}
\subfloat
[\bf \textsc{sCL}(supervised)]
{\includegraphics[width=0.25\textwidth]{plots/tsne/imagenet-sup-sCL.pdf}}
\caption{\footnotesize  
\textbf{Embedding Quality vs. Supervision Ratio ($\gamma$).}  
We visualize the learned feature embeddings (t-SNE) from a ResNet-18 trained on the ImageNet-II dataset using different contrastive learning methods. The supervision ratio \( \gamma = n_{\textsc{P}} / n_{\textsc{U}} \) controls the proportion of labeled positives, while the total number of training samples \( N = n_{\textsc{P}} + n_{\textsc{U}} \) is held fixed. Compared to the unsupervised baseline \textsc{ssCL}, our proposed \textsc{puCL} yields substantially improved class separability, which improves consistently with increasing \(\gamma\). This highlights the benefit of incorporating even limited supervision. The {\bf fully supervised \textsc{sCL} serves as an upper bound} in terms of embedding structure with similar training hyper-parameters.
}
\label{fig:tsne_puCL}
\end{figure*}
Motivated by the {\bf bias-variance trade-off}, we ask the natural question:
\begin{quote} {\em Can we design a contrastive loss that enjoys low variance by leveraging information about labeled positives, but avoids the bias introduced by incorrect assumptions about unlabeled examples?} \end{quote}
To this end, we propose Positive-Unlabeled Contrastive Learning (\textsc{puCL})—a simple and effective modification to the contrastive framework that strikes a principled balance between bias and variance in the PU setting. 
Unlike $\gL_{\textsc{sCL-PU}}$~\eqref{eq:scl_pu}, which introduces bias by treating the unlabeled set as a surrogate negative class, \textsc{puCL} avoids making any explicit assumptions about the labels of the unlabeled data. Instead, it retains the self-supervised assumption for unlabeled examples—namely, that different augmentations of the same image should remain close in the learned representation space. At the same time, \textsc{puCL} uses the available labeled positives to form additional attractive pairs, thereby improving the stability of the learned representations. This is in the same spirit as the semi-supervised objective~\citep{assran2020supervision}, but is tailored for the PU learning setup. 
\\
\\
In particular, the modified objective dubbed \textsc{puCL} leverages the available supervision as follows -- each labeled positive anchor is attracted closer to all other labeled positive samples in the batch, whereas an unlabeled anchor is only attracted to its own augmentation. 
\begin{equation}
    \gL_{\textsc{puCL}}
    = - 
    \frac{1}{|\sI|}\sum_{i \in \sI}
    \Biggl[\1(i\in \sU) \bigg(\vz_i \boldsymbol \cdot \vz_{a(i)}\bigg)
    + \frac{\1(i\in \sP)}{|\sP\setminus i|}\sum\limits_{j \in \sP\setminus i}\vz_i \boldsymbol \cdot \vz_j
    - \log Z(\vz_i) \Biggr]
\label{eq:puCL}
\end{equation}
where, $\sP$ and $\sU$ denote the subset of sample indices in the {\bf multi-viewed batch} $\tilde{\gD}$ that are {\bf labeled} positives and unlabeled respectively~\eqref{eq:pu_dataset_batch}. The indicator function $\1(\cdot)$ selects the appropriate term depending on whether the anchor is a labeled positive or unlabeled.
\\
\\
{\bf \textsc{puCL}\eqref{eq:puCL} can be viewed as a sample-adaptive extension of \textsc{mCL}}~\eqref{eq:mcl}, where the mixing coefficient $\lambda$ is chosen per instance. Specifically, unlabeled samples use $\gL_{\textsc{mCL}}(\lambda = 0)$ and labeled positives use $\gL_{\textsc{mCL}}(\lambda = 1)$, i.e.
\begin{flalign}
    \gL_\textsc{puCL} = \frac{1}{|\sI|}\sum_{i\in \sI} \bigg[\1(\vx_i \in \sP)\ell^i_\textsc{mCL}(1) +  \1(\vx_i \in \sU)\ell^i_\textsc{mCL}(0)\bigg]
\end{flalign}
This adaptive strategy enables \textsc{puCL} to better interpolate between supervised and unsupervised contrastive learning in a principled and data-dependent manner.
\\
\\
Under (\cref{assumption:aug_independence}), it is easy to verify that $\gL_{\textsc{puCL}}$ is an unbiased estimator of the population contrastive loss $\gL^\ast_{\textsc{CL}}$~\eqref{eq:asymptotic_scl}, unlike its biased counterpart $\gL_{\textsc{sCL-PU}}$. Moreover, by aggregating across the labeled positives, \textsc{puCL} reduces estimation variance over its unsupervised counterpart $\gL_\textsc{ssCL}$ as formalized in~\cref{lemma:pucl_unbiased}.
\begin{figure*}[t] 
  \centering
  \subfloat[{\bf CIFAR-0}]
 {\includegraphics[width=0.4\textwidth]{plots/convergence/fig.CIFAR_convergence.pdf}}
  \subfloat[{\bf ImageNet-II}]
 {\includegraphics[width=0.4\textwidth]{plots/convergence/fig.imageNet_convergence.pdf}}
  \setlength{\abovecaptionskip}{0pt} 
  \caption{\footnotesize {\bf Convergence:} Training ResNet-18 on 
  (a) {\bf CIFAR-0} (b) {\bf ImageNet-II}. Clearly, by incorporating more labeled positives \textsc{puCL} enjoys convergence speedup over \textsc{ssCL}.}
  \label{fig:puCL_convergence}
\end{figure*}
\begin{lemma}
    \label{lemma:pucl_unbiased}
    If~\cref{assumption:aug_independence} holds, then
    $\gL_{\textsc{ssCL}}$~\eqref{eq:sscl} and $\gL_{\textsc{puCL}}$~\eqref{eq:puCL} are unbiased estimators of  $\gL_{\textsc{CL}}^\ast$~\eqref{eq:asymptotic_scl}. Additionally, it holds that:
    \begin{flalign}
        \Delta_\sigma(\gamma) \geq 0 \;\;\forall \gamma \geq 0\\
        \Delta_\sigma(\gamma_1) \geq \Delta_\sigma(\gamma_2) \quad \forall \gamma_1 \geq \gamma_2 \geq 0
    \end{flalign}
    where, $\Delta_\sigma(\gamma)= 
    \mathrm{Var}(\gL_{\textsc{ssCL}}) - \mathrm{Var}(\gL_{\textsc{puCL}})$ and $\gamma = n_\textsc{P}/n_\textsc{U}$. 
\end{lemma}
This result suggests that for PU learning $\gL_{\textsc{puCL}}$ is a {\bf statistically more efficient} estimator of $\gL^\ast_\textsc{CL}$ compared to $\gL_\textsc{ssCL}$. Furthermore, this improvement is strictly monotonic in $\gamma$, meaning that \textsc{puCL} becomes increasingly favorable as the fraction of labeled positives grows. Overall, \textsc{puCL} strikes a principled balance between bias and variance in the PU setting by integrating weak supervision in a cautious manner: labeled positives are utilized to strengthen semantic cohesion, while unlabeled samples are treated conservatively to prevent bias.
\\
\\
Consequently, $\gL_{\textsc{puCL}}$ consistently results in {\bf improved generalization} over its unsupervised counterpart; as also validated by our empirical findings (\cref{fig:nP_contrastive}-\ref{fig:mcl}).
These improvements become more pronounced with increased PU supervision, as also indicated by the better separability of the resulting embedding space (\cref{fig:tsne_puCL}). 
\\
\\
We further analyze the {\bf training dynamics} of \textsc{puCL} by studying the gradient expressions and comparing them to the unsupervised baseline \textsc{ssCL} and the fully supervised contrastive loss. This analysis highlights how PU supervision reduces the sampling bias inherent in \textsc{ssCL}, and how this, in turn, improves convergence behavior.
\\
\\
The gradient of the unsupervised contrastive loss \textsc{ssCL} is given by:
\begin{equation}
    \nabla(\gL_{\textsc{ssCL}})
    = - 
    \frac{1}{|\sI|} \sum_{i \in \sI} \frac{1}{\tau} \left[
        \vz_{a(i)} \left( 1 - P_{i,a(i)} \right)
        - \sum_{j \in \sI \setminus \{i, a(i)\}} \vz_j P_{i,j}
    \right],
\end{equation}
where, \( P_{i,j} \) denotes the softmax-normalized similarity between \(\vz_i\) and \(\vz_j\).
\begin{equation}
    P_{i,j} = \frac{\exp(\vz_i\boldsymbol \cdot \vz_j)}{Z(\vz_i)}
\end{equation}
In \textsc{ssCL}, all samples other than the augmentation \( a(i) \) are implicitly treated as negatives, including true positives. This incorrect assumption introduces \emph{gradient bias}, since the model is pushed away from examples that are semantically similar to the anchor. This effect is especially pronounced in PU settings where many positive samples remain unlabeled.
\\
\\
In contrast, the gradient of \textsc{puCL} is:
\begin{equation}
    \nabla(\gL_{\textsc{puCL}})
    = - 
    \frac{1}{|\sI|} \sum_{i \in \sI} \frac{1}{\tau} \left[
        \sum_{q \in \sP} \vz_q \left( \frac{1}{|\sP|} - P_{i,q} \right)
        - \sum_{j \in \sU \setminus \{i\}} \vz_j P_{i,j}
    \right],
\end{equation}
The key distinction here is that \textsc{puCL} uses explicit supervision to attract anchors toward all available labeled positives and refrains from making strong assumptions about the unlabeled points. Instead of treating them as negatives, it softly pushes away only those that appear dissimilar, leading to better representation quality and lower bias and thereby, the gradients of \textsc{puCL} form a closer approximation to those of the ideal supervised objective(~\citep{khosla2020supervised}).
This improved gradient alignment often leads to more stable and faster convergence in practice as validated by our experiments (\cref{fig:puCL_convergence}). The gradient derivations can be found in~\cref{sec:proof.gradient}. 

\begin{figure*}[t] 
\small 
\centering
\subfloat[\bf Unsupervised: Semantic similarity obeying]{\includegraphics[width=0.7\textwidth]{plots/point_config/point_config_unsup.pdf}}
\\
\subfloat[\bf Supervised: Semantic annotation obeying]{\includegraphics[width=0.7\textwidth]{plots/point_config/point_config_sup.pdf}}
\\
\subfloat[\bf PU Supervised: Semantic similarity and annotation obeying]{\includegraphics[width=0.8\textwidth]{plots/point_config/point_config_both.pdf}}
\caption{\footnotesize \textbf{Geometric Intuition of Incorporating Supervision:} Consider 1D feature space $\vx \in \mathbb{R}$, e.g., $\vx_i = 1$ if shape: triangle (\textcolor{blue}{$\blacktriangle$}, \textcolor{red}{$\blacktriangle$}), $\vx_i = 0$ if shape: circle (\textcolor{blue}{\textbullet}, \textcolor{red}{\textbullet}). However, the labels are $y_i=1$ if color: blue (\textcolor{blue}{$\blacktriangle$},\textcolor{blue}{\textbullet}) and $y_i=1$ if color: red (\textcolor{red}{$\blacktriangle$}, \textcolor{red}{\textbullet}). We show possible configurations (other configurations are similar) of arranging these points on the vertices of unit hypercube $\gH \in \sR^2$ when \textcolor{blue}{$\blacktriangle$} is fixed at $(0, 1)$. (a) Unsupervised objectives e.g. \textsc{ssCL}~\eqref{eq:sscl} only rely on semantic similarity (feature) to learn embeddings, implying they attain minimum loss configuration (shaded) when semantically similar objects are places close to each other (neighboring vertices on $\gH^2$). (b) Supervised objectives on the other 
All the four shaded point configurations are favored by \textsc{ssCL}, since $\vx_i = \vx_j$ are placed neighboring vertices. However, the minimum loss configurations of  
  \textsc{puCL} additionally also preserves annotation consistency. 
  }
\label{fig:point_config}
\end{figure*}
\subsection{Geometric Intuition: Minimum Energy Configurations.}
In essence, the unsupervised part in \textsc{puCL} enforces consistency between representations learned via label-preserving augmentations i.e. between $\vz_i$ and $\vz_{a(i)} \; \forall i \in \sI$, whereas the supervised component injects structural knowledge derived from labeled positives. 
To further dissect understand various contrastive losses navigate the trade-off between semantic annotation (labels) and semantic similarity (features), we analyze their corresponding minimum energy configurations~\citep{graf2021dissecting, lecun2006tutorial, ranzato2007unified}.
\\
\\
Specifically, consider the following setup: $\vx_i = 1$ if the object is a triangle (\textcolor{blue}{$\blacktriangle$}, \textcolor{red}{$\blacktriangle$}), and $\vx_i = 0$ if it is a circle (\textcolor{blue}{\textbullet}, \textcolor{red}{\textbullet}). Labels, however, depend solely on color: $y_i = 1$ for blue (\textcolor{blue}{$\blacktriangle$}, \textcolor{blue}{\textbullet}) and $y_i = 0$ for red (\textcolor{red}{$\blacktriangle$}, \textcolor{red}{\textbullet}). Thus, $p(\vx)$ provides no information about $p(y \mid \vx)$.
To analyze the behavior of different variants of contrastive objectives, we embed each of the four types of samples — (\textcolor{blue}{$\blacktriangle$}, \textcolor{red}{$\blacktriangle$}, \textcolor{blue}{\textbullet}, \textcolor{red}{\textbullet}) — on the vertices of a 2D hypercube $\gH^2 \subset \mathbb{R}^2$. Each sample is assigned to a unique vertex, yielding different configurations as depicted in~\cref{fig:point_config}.  Without loss of generality, we anchor the blue triangle (\textcolor{blue}{$\blacktriangle$}) at $(0, 1)$ to eliminate rotational symmetry.
\\
\\
Unsupervised objectives, e.g.,~\textsc{ssCL}~\eqref{eq:sscl} only rely on semantic similarity (feature) to learn embeddings, implying they attain minimum loss configuration when semantically similar objects $\vx_i = \vx_j$ are placed close to each other (neighboring vertices on $\gH^2$) since this minimizes the inner product between representations of similar examples(\cref{fig:point_config}(a)).
\\
\\
Supervised objectives on the other hand, update the parameters such that the logits match the label. Thus purely supervised objectives attain minimum loss when objects sharing same annotation are placed next to each other (\cref{fig:point_config}(b)).
\\
\\
\textsc{puCL} interpolates between the supervised and unsupervised objective. Simply put, by incorporating additional positives aims at learning representations that preserve annotation consistency. Thus, the minimum loss configurations are attained at the intersection of the minimum point configurations of \textsc{ssCL} and fully supervised \textsc{sCL} (\cref{fig:point_config}(c)).

\section{Knowledge of Class Prior Estimate}
\label{sec:puNCE}
In~\cref{sec:puCL}, we showed how incorporating weak supervision—specifically, access to labeled positive examples—significantly improves contrastive learning in the Positive-Unlabeled (PU) setting. By leveraging labeled positives, while making no assumptions about the unlabeled data, \textsc{puCL} strikes a principled balance between the robustness of self-supervised learning and the semantic structure provided by weak-supervision.
\\
\\
However, in many real-world scenarios, we may also have access to additional global information about the data distribution -- most notably, the \emph{class prior} \( \pi := \Pr(y = 1 | \rx) \), which quantifies the proportion of positives in the overall population. While the true value of \( \pi \) is typically unknown, it can often be estimated with high accuracy using Mixture Proportion Estimation (MPE) techniques~\citep{ramaswamy2016mixture, yao2021rethinking, garg2021mixture}, provided sufficient unlabeled data and computational resources.
\\
\\
This raises a natural question:
\begin{quote}
\emph{Can we go beyond \textsc{puCL} by using class prior information to refine how unlabeled examples are handled during contrastive training?}
\end{quote}
In the remainder of this section, we explore this direction and show that class prior knowledge can indeed refine the treatment of unlabeled examples.
\subsection{\textsc{puNCE}: Prior Aware PU Contrastive Learning}
\begin{figure*}[t]
\centering
\subfloat[\bf \textsc{ssCL}]{\includegraphics[width=0.24\textwidth]{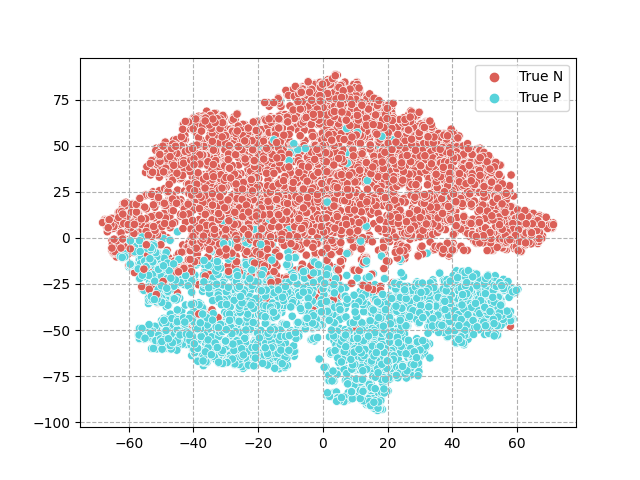}}
\subfloat[\bf \textsc{puCL} ($n_p= 3k$)]{\includegraphics[width=0.24\textwidth]{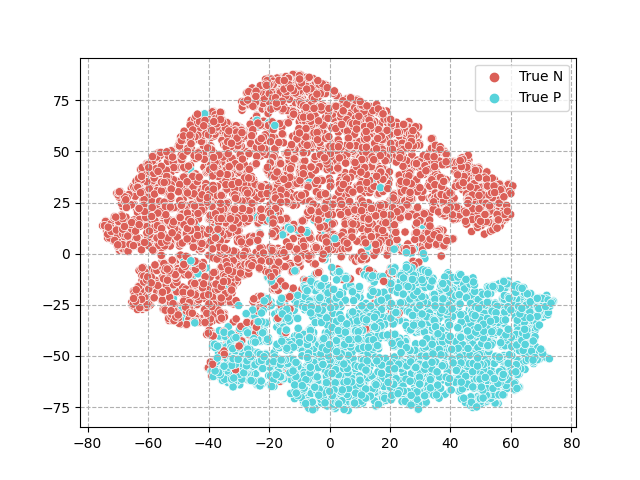}}
\subfloat[\bf \textsc{puNCE} ($n_p= 3k$)]{\includegraphics[width=0.24\textwidth]{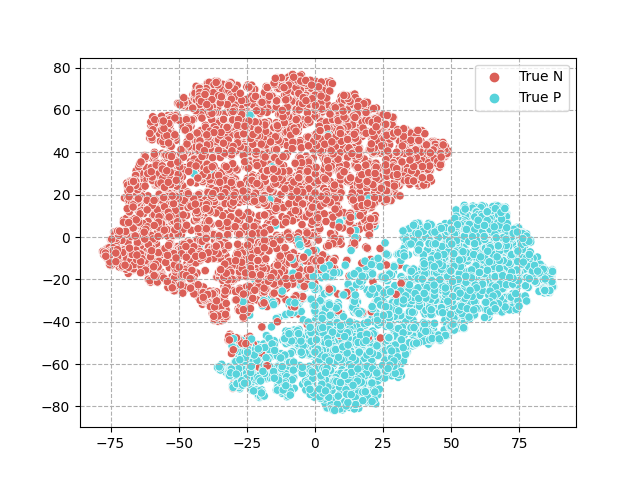}}
\subfloat[\bf \textsc{sCL} (supervised)]{\includegraphics[width=0.24\textwidth]{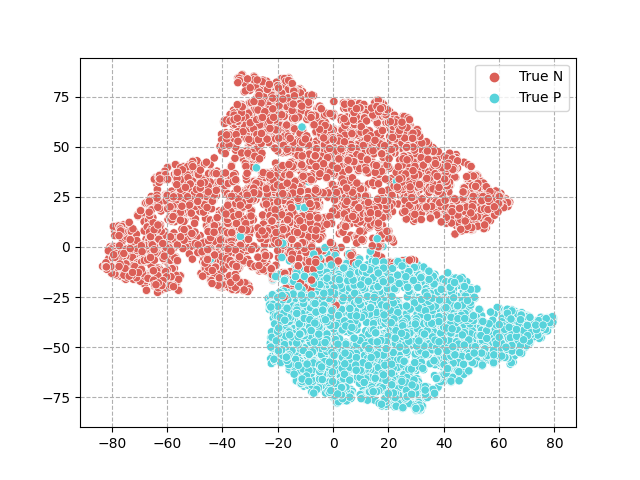}}
\caption{\footnotesize {\bf Embedding Quality.} t-SNE visualization of representations on CIFAR-III(vehicle vs animal) learned via ResNet-18. Classes are indicated by colors. Clearly the modified \textsc{puNCE} objective leads to better class separation than \textsc{puCL}, showcasing the value of incorporating class prior knowledge.}
\label{figure:puNCE_tsne}
\end{figure*}
\begin{figure*}[t]
\centering
\subfloat[\bf CIFAR-0]{\includegraphics[width=0.4\textwidth]{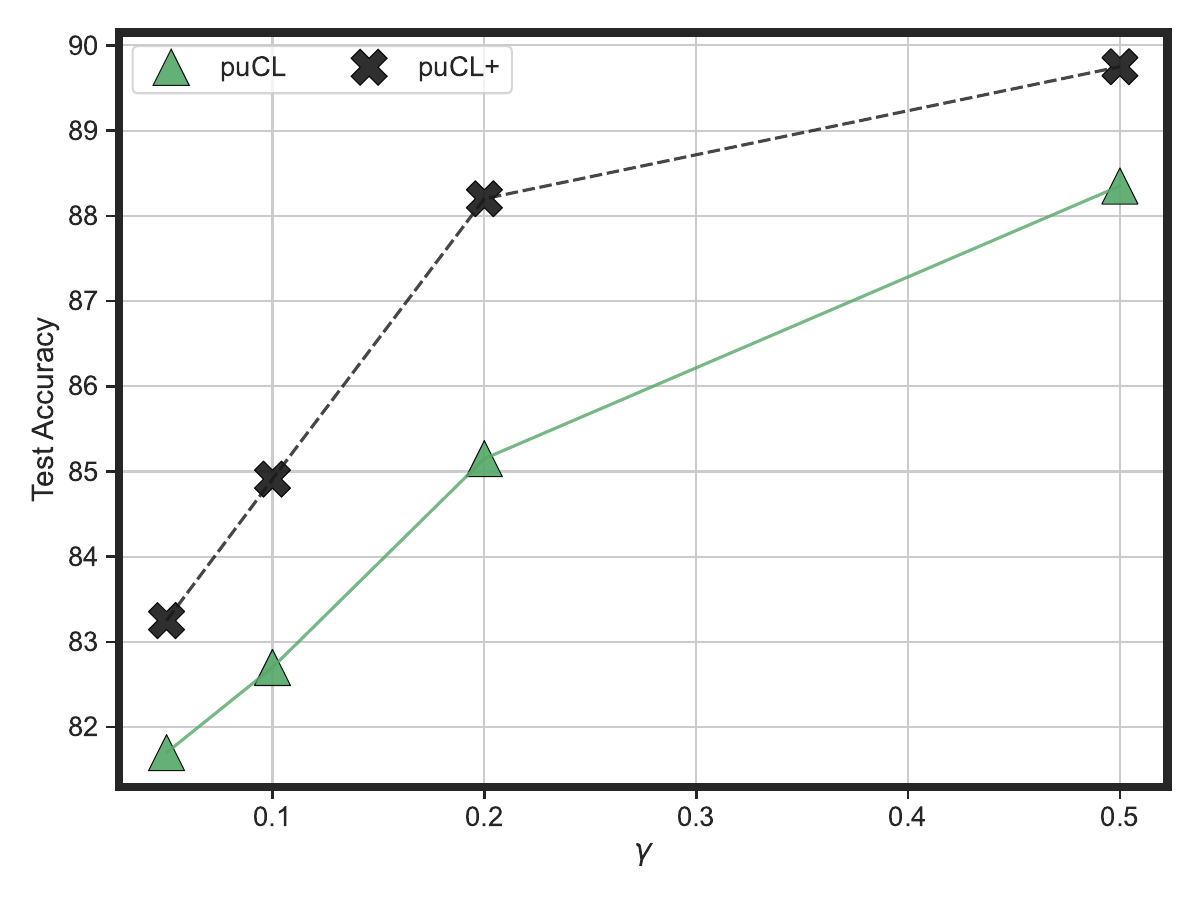}}
\subfloat[\bf Imagenet-II]{\includegraphics[width=0.4\textwidth]{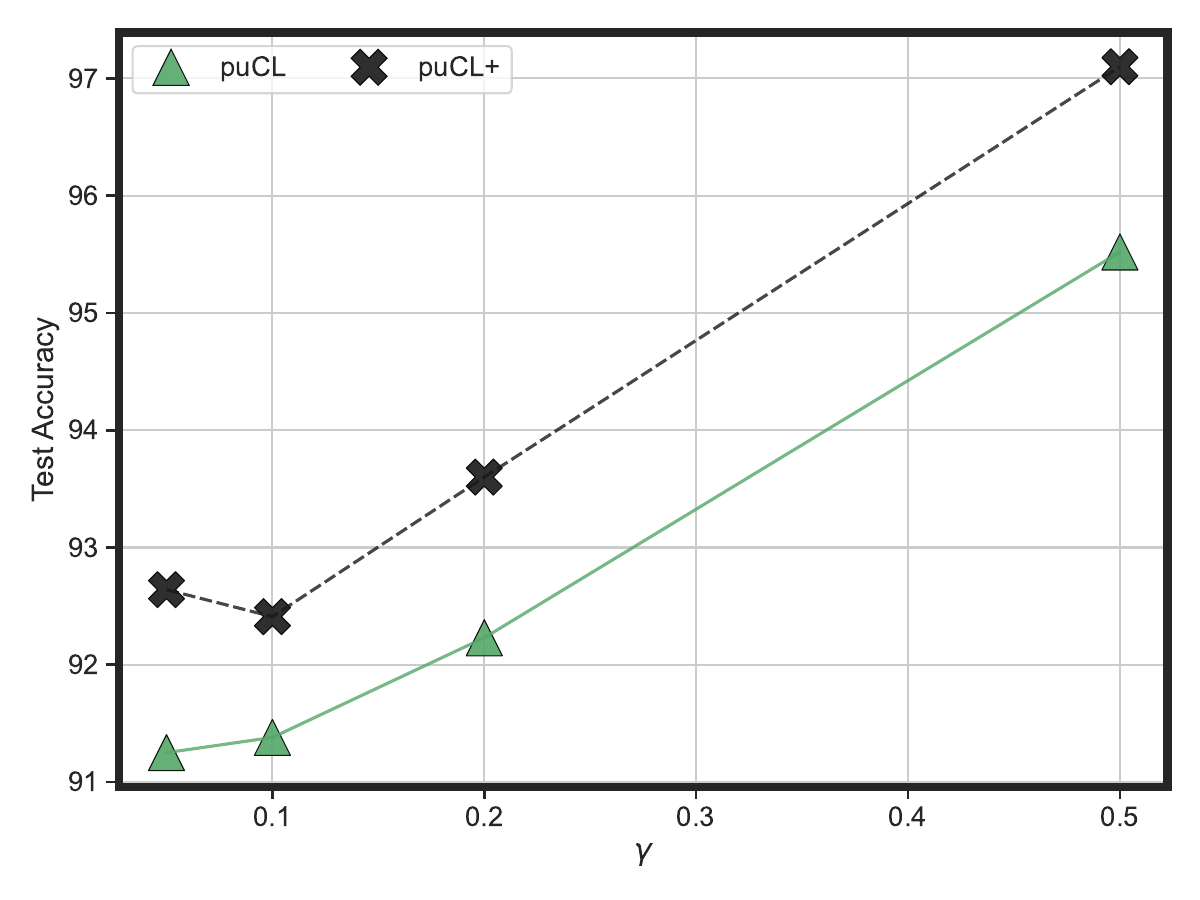}}
\caption{\footnotesize {\bf \textsc{puCL} vs \textsc{puNCE}}. ResNet-18 trained on CIFAR-0 and ImageNet-II (\cref{sec:exp}). The enhanced variant, \textsc{puNCE}, consistently outperforms the \textsc{puCL}, demonstrating the benefit of incorporating class prior information to judiciously weigh unlabeled data in addition to labeled positives.}
\label{figure:puNCE_vs_puCL}
\end{figure*}
Rather than treating all unlabeled points uniformly, as in \textsc{puCL}, we adjust the contrastive objective to reflect the underlying positive-negative mixture in the unlabeled subset in the batch. Specifically, by incorporating knowledge of \( \pi \), we construct a \emph{prior-aware} contrastive loss that accounts for the expected composition of positives and negatives among the unlabeled samples. This allows the model to weigh attraction and repulsion terms more appropriately. In doing so, it enables the model to leverage unlabeled data more effectively in the absence of dense supervision, while avoiding the pitfalls of over- or under-representation of the positive class -- leading to further improvements in both generalization and convergence.
We refer to this refined objective as Positive Unlabeled InfoNCE dubbed \textsc{puNCE}~\eqref{eq:puNCE}.
\\
\\
The main idea is to use the fact that unlabeled data is distributed as a mixture of positive and negative marginals where the mixture proportion is given by class prior $\pi$
~\citep{elkan2008learning, niu2016theoretical, du2014analysis, elkan2001foundations}. 
\begin{flalign}
\label{eq:bayes}
    p(\vx) = \pi p(\vx | y=1) + (1- \pi) p(\vx | y = -1)
\end{flalign}
\textsc{puNCE} additionally treats each unlabeled sample as a positive example and a negative example with appropriate probabilities.
\\
\\
In particular, consider a {\bf labeled anchor} $\vx_i \in \gX_\textsc{P}$; same as \textsc{puCL}, the \textsc{puNCE} risk on this sample is computed by pulling together normalized embeddings of all the available labeled samples in multi-viewed batch $\tilde{\gD}$: 
\begin{equation}
    \ell_\textsc{P}^{(i)}
    =
    \Biggl[\frac{1}{|\sP \setminus i|}\sum\limits_{j \in \sP\setminus i}\vz_i \boldsymbol \cdot \vz_j
    - \log Z(\vz_i) \Biggr]
\label{eq:puNCE.P_risk}
\end{equation}
On the other hand, {\bf unlabeled anchor} $\vx_i \in \gX_\textsc{U}$, is treated as a positive example with probability $\pi$ and as negative example with probability $(1-\pi)$. When considered positive ($\ry = 1$), all the labeled samples $\vx_i \in \gX_\textsc{P}$, along with its self-augmentation $\vx_{a(i)}$ are used as positive pairs for $\vx_i$. Whereas, when it is considered negative ($\ry = 0$), since there are no available labeled negative samples \textsc{puNCE} treats only augmentation $\vx_{a(i)}$ as positive pair. The empirical risk on unlabeled samples $\ell_\textsc{U}$ is thus computed as:   
\begin{equation}
    \ell_\textsc{U}^{(i)}  
    =\Biggl[\frac{\pi}{|\sP|+1} \sum_{j\in \{\sP,a(i)\}} \vz_i \boldsymbol \cdot \vz_j + (1-\pi)\log \bigg(\vz_i \boldsymbol \cdot \vz_{a(i)}\bigg) - \log Z(\vz_i) \Biggr]
\label{eq:puNCE.U_risk}
\end{equation}
The first term of~\eqref{eq:puNCE.U_risk} denotes the loss incurred by the positive contribution of the unlabeled samples and the second term corresponds to the negative counterpart. Combining~\eqref{eq:puNCE.P_risk}, ~\eqref{eq:puNCE.U_risk} the \textsc{puNCE} empirical risk is computed as:
\begin{equation}
\gL_{\textsc{puNCE}} = - \frac{1}{|\sI|} \sum_{i \in \sI} \Biggl[\1(\vx_i\in \gX_\textsc{P}) \ell_\textsc{P}^{(i)} + \1(\vx_i\in \gX_\textsc{U}) \ell_\textsc{U}^{(i)}\Biggr]
\label{eq:puNCE}
\end{equation}
Intuitively, all the labeled samples are given unit weight and the unlabeled samples are duplicated; one copy is labeled positive with weight $\pi$ and the other copy is labeled negative with weight $(1 - \pi)$. In this sense, the prior-aware reweighting acts as a form of importance correction, akin to the techniques used in cost-sensitive methods~\citep{elkan2008learning, du2014analysis}, where unlabeled examples are treated as soft labels with probabilistic contributions. Instead of making hard decisions about which unlabeled instances are positive or negative, \textsc{puNCE} distributes credit across both classes in expectation, using the known class prior $\pi$ as a soft surrogate for label uncertainty. This can be seen as a contrastive analogue to the EM-style reweighting or risk correction methods used in probabilistic weak supervision frameworks~\citep{elkan2008learning}.
\\
\\
Intuitively, this probabilistic treatment introduces a degree of {\bf inductive bias} that can be beneficial in practice, particularly in low-supervision regimes where the signal from labeled positives alone may be insufficient. By assigning soft labels to the unlabeled data, \textsc{puNCE} effectively leverages the entire batch for representation learning, leading to richer and more semantically meaningful embeddings. However, this comes at the cost of potentially introducing bias, especially when the estimated class prior $\pi$ deviates from the true underlying distribution. Moreover, the benefits of prior-aware weighting become especially apparent when the representation space exhibits high intra-class semantic similarity and inter-class ambiguity. In such settings, the additional soft supervision provided by $\pi$ enables \textsc{puNCE} to better disentangle clusters (Figure~\ref{figure:puNCE_tsne}), leading to improved separability and faster convergence. These results suggest that even modest estimates of the class prior can be valuable during contrastive pretraining, provided that the reweighting is performed in a stable, expectation-based manner, as done in \textsc{puNCE}.
\begin{table*}[t]
\centering
\footnotesize 
\resizebox{\textwidth}{!}{
\begin{tabular}{@{}ccccccc@{}}
\toprule
{\bf Dataset}
& $g_\mB(\cdot)$
& $n_\textsc{P}$
& {\bf \textsc{ssCL}}
& {\bf \textsc{dCL}}$^\dagger$
& {\bf \textsc{puCL}}
& {\bf \textsc{puNCE}}$^\dagger$ 
\\ \midrule \midrule 
\multirow{3}{1.75 cm}{\centering MNIST-I\\(odd/even)} 
& \multirow{3}{1.25 cm}{\centering MLP} 
& 1k    
& ${94.15\textcolor{gray}{\pm 0.15}}$         
& ${94.32\textcolor{gray}{\pm 0.42}}$        
& ${94.24\textcolor{gray}{\pm 0.09}}$        
& ${\bf 96.70\textcolor{gray}{\pm 0.19}}$     
\\ 
&
& 3k    
& ${94.84\textcolor{gray}{\pm 0.25}}$         
& ${95.09\textcolor{gray}{\pm 0.36}}$        
& ${95.82\textcolor{gray}{\pm 0.18}}$       
& ${\bf 97.81\textcolor{gray}{\pm 0.21}}$      
\\ 
&
& 10k    
& ${95.15\textcolor{gray}{\pm 0.05}}$         
& ${95.45\textcolor{gray}{\pm 0.08}}$       
& ${\bf 98.29\textcolor{gray}{\pm 0.08}}$        
& ${\bf 98.27\textcolor{gray}{\pm 0.11}}$      
\\\midrule 
\multirow{3}{2.5 cm}{\centering CIFAR-III \\ (animal / vehicle)} 
& \multirow{3}{2.5 cm}{\centering ResNet18} 
& 1k    
& ${96.33\textcolor{gray}{\pm 0.14}}$       
& ${96.21\textcolor{gray}{\pm 0.11}}$       
& ${97.42\textcolor{gray}{\pm 0.07}}$       
& ${\bf 97.59 \textcolor{gray}{\pm 0.17}}$     
\\ 
&
& 3k    
& ${96.51\textcolor{gray}{\pm 0.06}}$         
& ${96.49\textcolor{gray}{\pm 0.03}}$       
& ${97.66\textcolor{gray}{\pm 0.04}}$       
& ${\bf 97.97\textcolor{gray}{\pm 0.04}}$    
\\ 
&
& 10k    
& ${96.58\textcolor{gray}{\pm 0.04}}$       
& ${96.50\textcolor{gray}{\pm 0.02}}$        
& ${97.70\textcolor{gray}{\pm 0.07}}$       
& ${\bf 98.15\textcolor{gray}{\pm 0.02}}$ 
\\\bottomrule                                          
\end{tabular}}
\caption{
\footnotesize
\textbf{Incorporating Weak Supervision:}  
\textsc{ssCL} denotes standard self-supervised contrastive learning; \textsc{dCL} applies a class-prior-based debiasing correction to the partition function; \textsc{puCL} leverages labeled positives without class prior information; and \textsc{puNCE} incorporates both labeled positives and oracle class prior knowledge. Evaluation is performed via a supervised $k$-nearest neighbor ($k$NN) classifier over the learned representations. Across all settings, \textsc{puNCE} consistently achieves the highest accuracy, highlighting the benefit of combining additional weak global supervision (class prior) with supervision from labeled positives.  
$\dagger$ indicates methods that utilize oracle class prior information.  
}

\label{tab:puNCE}
\end{table*}
Our experiments further demonstrate that the embeddings learned by \textsc{puNCE} are qualitatively and quantitatively superior to those produced by \textsc{puCL}. As shown in~\cref{figure:puNCE_tsne}, the learned representation space under \textsc{puNCE} exhibits better cluster separation and semantic alignment. This improved structure translates to consistently stronger generalization performance on downstream tasks, as evidenced by~\cref{figure:puNCE_vs_puCL} and~\cref{tab:puNCE}. Remarkably, even with a small amount of weak supervision, \textsc{puNCE} can dramatically improve the quality of learned embeddings — highlighting the effectiveness of leveraging class prior information to guide representation learning under label scarcity.

\subsection{\textsc{dCL}: Debiased Contrastive Learning}
A closely related approach to incorporate weak latent supervision into the unlabeled is by appropriately compensating for the sampling bias referred to as Debiased Contrastive Learning, dubbed \textsc{dCL}~\citep{chuang2020debiased, robinson2020contrastive}. Specifically, they note that the standard objective \textsc{ssCL} implicitly assumes a uniform prior over positives and negatives, leading to biased gradient estimates. To mitigate this, \textsc{dCL} decomposes the partition function and introduces a weighted correction that explicitly accounts for the expected number of positive pairs among the negative set. 
Specifically, in the unsupervised~\eqref{eq:sscl} setting, the partition function can be decomposed into:
\begin{flalign}
    Z(\vz_i) := \underbrace{\exp(\vz_i \boldsymbol \cdot \vz_{a(i)})}_{\text{positive pair}}+ \underbrace{\sum_{i \in \sI\setminus\{i, a(i)\}}\exp(\vz_i\boldsymbol\cdot \vz_j)}_{\text{negative pairs sum}}
\end{flalign}
Implying that the standard contrastive loss implicitly treats all non-anchor pairs as true negatives, which introduces bias when some of these pairs may, in fact, be positives. To correct for this, \textsc{dCL} introduces a class-prior-weighted debiasing term, effectively down-weighting the negative pair contributions in the partition function according to the probability that a randomly drawn pair is actually positive.
\\
\\
Formally, the debiased partition function takes the form: 
\begin{flalign} 
Z_{\textsc{dCL}}(\vz_i) := \exp(\vz_i \cdot \vz_{a(i)}) + \frac{|\sI|-2}{1- \pi}\Bigg[\sum_{i \in \sI\setminus\{i, a(i)\}}\exp(\vz_i\boldsymbol\cdot \vz_j) - \pi\exp(\vz_i \cdot \vz_{a(i)})\Bigg] 
\end{flalign} 
The resulting loss becomes:
\begin{equation}
    \gL_{\textsc{dCL}}
    = - 
    \frac{1}{|\sI|}\sum_{i \in \sI}
    \Biggl[\vz_i \boldsymbol \cdot \vz_{a(i)} - \log Z_{\textsc{dCL}}(\vz_i) \Biggr]
\label{eq:dcl}
\end{equation}
While both \textsc{puNCE} and \textsc{dCL} incorporate class prior information to correct for biases in contrastive learning, they do so from fundamentally different perspectives. {\bf \textsc{puNCE} modifies the numerator of the contrastive loss} by reweighting the similarity terms based on the assumed label distribution, effectively interpolating between positive and negative contributions in expectation. In contrast, {\bf \textsc{dCL} leaves the numerator unchanged and instead adjusts the denominator} (i.e., the partition function), debiasing the normalization term to reflect the expected prevalence of false negatives among the nominal negatives. In essence, \textsc{puNCE} directly alters how unlabeled anchors relate to their positive and negative pairs, while \textsc{dCL} implicitly adjusts the strength of repulsion from negative samples based on prior-informed estimates. 
\\
\\
This highlights a \textbf{key design choice} in weakly supervised contrastive learning: whether to inject prior knowledge by modifying pairwise similarities or by correcting the global context in which those similarities are evaluated. Both approaches are valid and potentially complementary, and future work could explore hybrid strategies that combine the strengths of numerator and denominator correction in the PU setting.
\\
\\
Empirically, we find that \textsc{puNCE} consistently outperforms \textsc{dCL} across multiple datasets and label regimes (\cref{tab:puNCE}), suggesting that leveraging both instance-level supervision (via labeled positives) and global class prior information offers a more effective inductive bias than relying on prior information alone.

\section{On Downstream PU Classification}
\label{sec:puPL}
While so far we have discussed about training the encoder using contrastive learning, resulting in an embedding space where similar examples are sharply concentrated and dissimilar objects are far apart, performing inference on this manifold is not entirely obvious.  
\\
\\
In the standard semi-supervised setting, the linear classifier can be trained using CE loss over the representations of the labeled data~\citep{assran2020supervision} to perform downstream inference. However, in PU learning, lacking any negative examples, $v_\vv(\cdot)$ should be trained with a specialized cost-sensitive PU learning objective such as \textsc{nnPU}~\citep{kiryo2017positive}(\cref{sec:background}). 
\\
\\
However, even when operating in a high-quality representation space - where contrastive pretraining has already clustered similar examples and separated dissimilar ones - standard PU risk minimization methods like \textsc{nnPU} treat the downstream classifier as if it were learning from scratch. These objectives operate purely on the labeled positives and unlabeled samples, without leveraging the rich geometric structure already encoded by the contrastive encoder.
In doing so, the downstream classifier, trained independently via a high-variance $\sim\gO(1/n_\textsc{P})$ risk estimator, may fail to capitalize on the separability and semantic organization induced during pretraining -- especially in the low-supervision regime i.e. when $\gamma = n_\textsc{P}/n_\textsc{U}$ is small. 
As a result, the final decision boundary may deviate significantly from the ideal, despite the encoder having successfully organized the data into clean, well-separated clusters which aligns with our empirical observations in~\cref{fig:puPL_gmm}.
\\
\\
This highlights the need for downstream strategies that are aligned with the geometry of the embedding space, and that can amplify the benefits of contrastive pretraining, rather than discard them.

\subsection{Positive Unlabeled Pseudo Labeling (\textsc{puPL})}
\begin{figure*}[t] 
\centering
\subfloat[\bf \footnotesize{Supervised}]
{\includegraphics[width=0.28\textwidth]{plots/puPL_GMM/clean_data.pdf}}
\subfloat[\bf \footnotesize{{PU ($\gamma=\frac{1}{4}$)}}]
{\includegraphics[width=0.28\textwidth]{plots/puPL_GMM/PU_750_All.pdf}}
\subfloat[\bf \footnotesize{{PU ($\gamma=\frac{1}{5}$)}}]
{\includegraphics[width=0.28\textwidth]{plots/puPL_GMM/PU_500_All.pdf}}

\subfloat[\bf \footnotesize{{PU ($\gamma=\frac{1}{10}$)}}]
{\includegraphics[width=0.28\textwidth]{plots/puPL_GMM/PU_250_All.pdf}}
\subfloat[\bf \footnotesize{{PU ($\gamma=\frac{1}{25}$)}}]
{\includegraphics[width=0.28\textwidth]{plots/puPL_GMM/PU_100_All.pdf}}
\subfloat[\bf \footnotesize{{PU ($\gamma=\frac{1}{50}$)}}]
{\includegraphics[width=0.28\textwidth]{plots/puPL_GMM/PU_50_All.pdf}}

\caption{\footnotesize
\textbf{Decision Boundary Deviation:} Visualization of linear classifiers trained on a separable Gaussian Mixture distribution using different PU learning strategies.  
In the fully supervised setting (a), CE$^\ast$ recovers the ideal decision boundary.  
In PU settings (b–f), naive CE (which treats unlabeled samples as negatives) suffers from increasing decision boundary deviation as the number of labeled positives decreases (i.e., smaller $\gamma$), due to biased supervision. 
\textsc{nnPU} is more robust and aligns well with CE$^\ast$ when sufficient positives are available, but degrades under extreme label sparsity due to high estimator variance.  
In contrast, CE over \textsc{puPL} leverages clustering in the representation space to recover pseudo-labels with high accuracy, yielding decision boundaries that remain well-aligned with the supervised optimum even at very low label rates — without requiring class prior knowledge.
}
\label{fig:puPL_gmm}
\end{figure*}
A natural idea for leveraging the geometry of the embedding space is to assign {\bf pseudo labels} to unlabeled data based on clustering structure. This approach has shown promise in various weakly supervised and semi-supervised settings e.g., PiCo~\citep{wang2021pico}, SCAN~\citep{van2020scan}, SeLa~\citep{asano2019self}, SwaV~\citep{caron2020unsupervised}. This indicates  that clustering can effectively recover semantic groupings in embeddings obtained via contrastive representation learning.
\\
\\
Building on this foundation, we propose a simple yet effective pseudo-labeling mechanism tailored for the Positive-Unlabeled (PU) setting. Our approach leverages the inductive bias that encoders obtained via contrastive pretraining naturally promote: a geometric separation of semantic concepts in the embedding space~\citep{parulekar2023infonce, huang2023towards}. This assumption is empirically supported by t-SNE visualizations of embeddings in~\cref{fig:tsne_puCL,figure:puNCE_tsne}.
\\
\\
To operationalize this, we combine ideas from semi-supervised clustering and $k$-means++ seeding~\citep{arthur2007k, liu2010novel, yoder2017semi}, and adopt a PU-specific pseudo-labeling procedure over the representations - 
\begin{flalign}
    \gZ_{\textsc{PU}} = \bigg\{g_{\mB}(\vx_i) \in \mathbb{R}^k : \vx_i \in \gX_{\textsc{PU}}\bigg\}.
\end{flalign}
where $g_{\mB}(\cdot)$ is the pretrained encoder and $\gX_{\textsc{PU}}$ is the combined set of labeled positives and unlabeled instances.
\begin{definition}[\textbf{Clustering}]
A clustering is defined by a set of centroids \( C = \{ \Mu_{\textsc{P}}, \Mu_{\textsc{N}} \} \subset \mathbb{R}^k \), which induces a partition of the representation space \( \gZ_{\textsc{PU}} \) into two disjoint subsets:
\begin{equation}
    \gZ_{\textsc{P}} := \left\{ \vz_i \in \gZ_{\textsc{PU}} \;\middle|\; \Mu_{\textsc{P}} = \arg\min_{\Mu \in C} \bigg\| \vz_i - \Mu \bigg\|^2 \right\}, \qquad
\gZ_{\textsc{N}} := \gZ_{\textsc{PU}} \setminus \gZ_{\textsc{P}}.
\end{equation}
\label{def:clustering}
\end{definition}
We define the quality of a clustering via the standard $k$-means potential:
\begin{definition}[\bf Potential Function]
    Given a clustering $C$ (\cref{def:clustering}), the potential function over the dataset $\gZ_{\textsc{PU}}$ is defined as:
    \begin{flalign}
    \phi(\gZ_{\textsc{PU}}, C) = \sum_{\vz_i \in \gZ_{\textsc{PU}}}\min_{\Mu \in C}\bigg\|\vz_i - \Mu\bigg\|^2
    \end{flalign}
\label{def:potential}
\end{definition}
In particular, we seek to find cluster centers $\{\Mu_\textsc{P}, \Mu_\textsc{N}\}$ on the embedding space, that approximately solves the {\em k}-means problem: 
\begin{flalign}
\label{eq:k_means}
    \Mu^\ast 
    = \{\Mu_\textsc{P} , \Mu_\textsc{N}\}
    := \argmin_{\Mu_\textsc{P} , \Mu_\textsc{N} \in \sR^d}
    \phi\bigg(\gZ_{\textsc{PU}}, \{\Mu_\textsc{P} , \Mu_\textsc{N}\}\bigg) 
\end{flalign}
Solving~\eqref{eq:k_means}, is known to be NP-hard via reduction from the Partition problem -- both in high dimensions~\citep{aloise2009np} and even under very restrictive settings: when the dimension is fixed ($\sR^2$), the number of clusters is small ($k=2$), and the distance metric is the standard squared Euclidean distance~\citep{mahajan2012planar}.
In practice, the most widely adopted heuristic for locally minimizing the $k$-means objective is Lloyd’s algorithm~\citep{lloyd1982least}, often coupled with $k$-means++ initialization~\citep{arthur2007k}.
\\
\\
In the PU setting, however, we can improve upon the standard unsupervised initialization by leveraging the available positive labels. Instead of initializing both centroids randomly, we initialize the positive centroid to be the centroid of the representations, labeled positive.
\begin{flalign}
    \Mu_\textsc{P}^{(0)} = \frac{1}{n_\textsc{P}}\sum_{\vx_i \in \gX_\textsc{P}}g_\mB(\vx_i).
\end{flalign}
The negative centroid $\Mu_{\textsc{N}}^{(0)}$ is initialized using the standard $k$-means++ seeding procedure applied over the unlabeled portion of the dataset. Specifically, we compute the squared Euclidean distance of the unlabeled samples $\vz_i \in \gZ_{\textsc{U}}$ from the positive centroid:
\begin{flalign}
    D^2(\vz_i) := \bigg\| \vz_i - \Mu_{\textsc{P}}^{(0)} \bigg\|^2 \;\; \forall \vz_i \in \gZ_{\textsc{U}}
\end{flalign}
and then sample $\Mu_{\textsc{N}}^{(0)}$ from the set $\gZ_{\textsc{U}}$ with probability proportional to $D^2(\vz_i)$:
\begin{flalign}
\Pr\bigg[\Mu_{\textsc{N}}^{(0)} = \vz_i\bigg] = \frac{D^2(\vz_i)}{\sum_{\vz_j \in \gZ_{\textsc{U}}} D^2(\vz_j)}.
\end{flalign}
This strategy improves upon random initialization by increasing the likelihood that initial centers are well-separated, thereby reducing the chance of poor local minima.
\\
\\
After clustering, pseudo-labels $\tilde{y}_i$ are assigned according to the proximity of each embedding $\vz_i$ to the estimated centroids $\{\hat{\Mu}_\textsc{P}, \hat{\Mu_\textsc{N}}\}$ as:
\begin{equation}
    \forall \vz_i \in \gZ_{\textsc{U}} : \tilde{y}_i =
\begin{cases}
    1, & \text{if } \hat{\Mu}_{\textsc{P}} = \arg\min_{\hat{\Mu} \in \{\hat{\Mu}_\textsc{P}, \hat{\Mu_\textsc{N}}\}} \|\vz_i - \hat{\Mu}\|^2 \\
    0, & \text{otherwise}
\end{cases}
\end{equation}
These pseudo-labels can then be used to train a downstream classifier using the standard supervised cross-entropy (CE) loss. We refer to this overall procedure as {\bf Positive Unlabeled Pseudo Labeling (\textsc{puPL})}, formally described in~\cref{algo:puPL}. 

\subsection{Convergence Guarantee}
\begin{table*}[t]
\footnotesize
\centering
\resizebox{\textwidth}{!}
{
\begin{tabular}{llccccccc}
\toprule
\multicolumn{2}{c}{\bf Methods} 
& \multicolumn{6}{c}{\bf Datasets} 
& \multirow{3}{*}{\bf Mean }
\\ 
\cmidrule(lr){1-2}
\cmidrule(lr){3-8}
\multirow{2}{*}{\bf PIRL}
& \multirow{2}{*}{\bf LP}
& {\bf F-MNIST-I}
& {\bf F-MNIST-II}
& {\bf CIFAR-I}
& {\bf CIFAR-II} 
& {\bf STL-I}
& {\bf STL-II} 
&                         
\\
&
& ($\pi_p^*=0.3$)
& ($\pi_p^*=0.7$)
& ($\pi_p^*=0.4$)
& ($\pi_p^*=0.6$)
& ($\hat{\pi}_p=0.51$)
& ($\hat{\pi}_p=0.49$)
&                         
\\
\midrule \midrule
&
&\multicolumn{6}{c}{$n_\textsc{P}=1000$} 
&                   
\\
\cmidrule{3-8}   
\textsc{ssCL}
&\textsc{nnPU$^\dagger$}
& $89.5\textcolor{gray}{\pm 0.9}$
& $85.9\textcolor{gray}{\pm 0.5}$
& $91.7\textcolor{gray}{\pm 0.3}$
& $90.0\textcolor{gray}{\pm 0.4}$ 
& $81.1\textcolor{gray}{\pm 1.2}$
& $81.4\textcolor{gray}{\pm 0.8}$
& 86.6
\\
\textsc{pu-sCL}
&\textsc{nnPU$^\dagger$}
& $73.0\textcolor{gray}{\pm 4.9}$
& $81.8\textcolor{gray}{\pm 0.5}$
& $88.4\textcolor{gray}{\pm 2.1}$
& $63.7\textcolor{gray}{\pm 5.3}$
& $59.2\textcolor{gray}{\pm 8.1}$
& $68.8\textcolor{gray}{\pm 3.1}$
& 72.5
\\
\textsc{puCL}
&\textsc{nnPU$^\dagger$}
& $90.0\textcolor{gray}{\pm 0.1}$
& $86.8\textcolor{gray}{\pm 0.4}$
& $91.8\textcolor{gray}{\pm 0.2}$
& $90.3\textcolor{gray}{\pm 0.5}$  
& $81.5\textcolor{gray}{\pm 0.7}$
& $82.6\textcolor{gray}{\pm 0.4}$
& ${87.2}$
\\
\rowcolor{light-gray}
\textsc{ssCL}
&\textsc{puPL}
& $91.4\textcolor{gray}{\pm 1.2}$
& $86.2\textcolor{gray}{\pm 0.6}$
& $91.6\textcolor{gray}{\pm 0.9}$
& $90.7\textcolor{gray}{\pm 0.4}$ 
& $81.2\textcolor{gray}{\pm 1.6}$
& $81.3\textcolor{gray}{\pm 0.7}$
& ${\bf 87.1}$
\\
\rowcolor{light-gray}
\textsc{pu-sCL}
&\textsc{puPL}
& $77.8\textcolor{gray}{\pm 0.3}$
& $82.5\textcolor{gray}{\pm 4.1}$
& $90.1\textcolor{gray}{\pm 1.2}$
& $68.9\textcolor{gray}{\pm 7.5}$ 
& $58.5\textcolor{gray}{\pm 8.2}$
& $73.9\textcolor{gray}{\pm 1.2}$
& $\bf 75.3$
\\
\rowcolor{light-gray}
\textsc{puCL}
&\textsc{puPL}
& $\bf 91.8\textcolor{gray}{\pm 0.8}$
& $\bf 89.2\textcolor{gray}{\pm 0.3}$
& $\bf 92.3\textcolor{gray}{\pm 1.9}$
& $\bf 91.2\textcolor{gray}{\pm 0.5}$
& $\bf 83.8\textcolor{gray}{\pm 1.4}$
& $\bf 84.5\textcolor{gray}{\pm 0.7}$
& $\bf 88.8$
\\
\midrule
&
&\multicolumn{4}{c}{\bf $n_\textsc{P}=3000$} 
&\multicolumn{2}{c}{\bf $n_\textsc{P}=2500$} 
&
\\
\cmidrule(lr){3-6}
\cmidrule(lr){7-8}
\textsc{ssCL}
&\textsc{nnPU$^\dagger$}
& $89.6\textcolor{gray}{\pm 0.1}$
& $85.0\textcolor{gray}{\pm 0.4}$
& $92.3\textcolor{gray}{\pm 0.3}$
& $92.7\textcolor{gray}{\pm 0.3}$ 
& $81.6\textcolor{gray}{\pm 0.9}$
& $84.2\textcolor{gray}{\pm 1.0}$
& $87.6$
\\
\textsc{pu-sCL}
&\textsc{nnPU$^\dagger$}
& $85.7\textcolor{gray}{\pm 0.3}$
& $82.1\textcolor{gray}{\pm 0.2}$
& $90.5\textcolor{gray}{\pm 3.1}$
& $88.6\textcolor{gray}{\pm 0.5}$ 
& $83.2\textcolor{gray}{\pm 0.8}$
& $84.8\textcolor{gray}{\pm 1.4}$
& $85.8$
\\
\textsc{puCL}
&\textsc{nnPU$^\dagger$}
& $90.3\textcolor{gray}{\pm 0.1}$
& $87.0\textcolor{gray}{\pm 0.7}$
& $93.2\textcolor{gray}{\pm 0.1}$
& $92.9\textcolor{gray}{\pm 0.1}$
& $84.9\textcolor{gray}{\pm 0.7}$
& $85.1\textcolor{gray}{\pm 0.7}$
& $88.9$
\\
\rowcolor{light-gray}
\textsc{ssCL}
&\textsc{puPL}
& $90.1\textcolor{gray}{\pm 0.2}$
& $88.8\textcolor{gray}{\pm 0.6}$
& $92.7\textcolor{gray}{\pm 1.3}$
& $92.9\textcolor{gray}{\pm 0.8}$
& $82.0\textcolor{gray}{\pm 1.6}$
& $84.3\textcolor{gray}{\pm 0.2}$
& $\bf 88.5$
\\
\rowcolor{light-gray}
\textsc{pu-sCL}
&\textsc{puPL}
& $85.9\textcolor{gray}{\pm 1.6}$
& $84.8\textcolor{gray}{\pm 2.4}$
& $92.4\textcolor{gray}{\pm 0.9}$
& $93.4\textcolor{gray}{\pm 1.2}$ 
& $83.1\textcolor{gray}{\pm 2.9}$
& $\bf 85.5\textcolor{gray}{\pm 0.6}$
& $\bf 87.5$
\\
\rowcolor{light-gray}
\textsc{puCL}
&\textsc{puPL}
& $\bf 92.0\textcolor{gray}{\pm 0.7}$
& $\bf 89.6\textcolor{gray}{\pm 1.2}$
& $\bf 93.5\textcolor{gray}{\pm 0.8}$
& $\bf 93.8\textcolor{gray}{\pm 0.4}$
& $\bf 85.0\textcolor{gray}{\pm 0.9}$
& $85.2\textcolor{gray}{\pm 2.1}$
& $\bf 89.9$
\\
\bottomrule
\end{tabular}}
\caption{\footnotesize
{\bf Effectiveness of \textsc{puPL}.}
To demonstrate the efficacy of \textsc{puPL} 
, we train a downstream linear classifier using \textsc{puPL(CE)} and \textsc{nnPU} ( run with $\pi^\ast$ ).
over embeddings obtained via different contrastive objectives - \textsc{ssCL, sCL-PU} and \textsc{puPL}. 
}
\label{tab:puPL_LP}
\end{table*}
We now provide a theoretical guarantee for the quality of the clustering obtained after a single iteration of \cref{algo:puPL}.
\begin{theorem}[\textbf{Clustering Quality of \textsc{puPL}}]
\label{th:puPL}
Let $\hat{\Mu} = \{ \hat{\Mu}_{\textsc{P}}, \hat{\Mu}_{\textsc{N}} \}$ denote the centroids obtained after one iteration of \textsc{puPL}(\cref{algo:puPL}). Then, we have:
\begin{flalign}
    \E\bigg[\phi(\gZ_{\textsc{PU}}, \hat{\Mu})\bigg] \leq 16 \cdot \phi(\gZ_{\textsc{PU}}, \Mu^\ast)
\end{flalign}
where, $\phi(\gZ, \Mu)$ denotes the $k$-means potential(\cref{def:potential}), and $\Mu^\ast$ is the optimal set of centroids minimizing this potential~\eqref{eq:k_means}. This is an improvement over $k$-means++ with:
\begin{flalign}
    \E\bigg[\phi(\gZ_{\textsc{PU}}, \tilde{\Mu})\bigg] \leq 21.55 \cdot \phi(\gZ_{\textsc{PU}}, \Mu^\ast)
\end{flalign}
$\tilde{\Mu} = \{ \tilde{\Mu}_{\textsc{P}}, \tilde{\Mu}_{\textsc{N}} \}$ denote the centroids obtained via standard $k$-means++.
\end{theorem}
This result implies that the {\bf PU supervision aware initialization strategy} employed by \textsc{puPL} yields a provably better clustering quality relative to standard $k$-means++. Notably, this guarantee holds for the first step alone, and the $k$-means potential can only decrease in subsequent iterations of Lloyd’s algorithm.
\\
\\
Intuitively,~\cref{th:puPL} suggests that \textsc{puPL} can recover the true clustering structure of the embedding space within a constant-factor multiplicative error. This approximation holds under natural assumptions: the feature space exhibits clustering structure (i.e., positive and negative instances form distinct regions), and the labeled positives are drawn i.i.d.\ from the true positive distribution.
\\
\\
Crucially, this recovery is achieved without requiring any class prior information—unlike many classical PU learning methods. The improved guarantee over $k$-means++ is a direct consequence of the supervision-informed initialization, which significantly reduces the variance of the resulting clustering and eliminates the randomness of selecting the first center.
\\
\\
Experiments on 2D Gaussian mixtures (\cref{fig:puPL_gmm}) and multiple PU learning benchmarks (\cref{tab:puPL_LP}) validate this result. We find that even when only a small fraction of positive labels are available (i.e., small $\gamma$), training a classifier on pseudo-labels obtained from \textsc{puPL} yields decision boundaries that closely align with those learned under fully supervised training.
\\
\\
Together, these findings establish \textsc{puPL} as a simple, effective, and theoretically grounded approach for PU learning over well-structured embedding spaces. It leverages geometric inductive bias, avoids class prior estimation, and converges rapidly—making it a computationally and statistically attractive solution in low-supervision regimes.

\section{Generalization Guarantee}
\label{sec:generalization}
Next, we theoretically explore the generalization ability of the overall contrastive approach to PU Learning -- training $g_{\mB}(\cdot)$ using \textsc{puCL} (\cref{algo:puCL}) -- followed by pseudo-labeling (\cref{algo:puPL}); the pseudo labels are then used to train the linear classification head $v_\vv$ -- on a binary (P vs N) classification task. 
\\
\\
We build on the recent theoretical framework~\citep{huang2023towards} to 
study generalization performance of our approach in terms of the concentration of augmented data. Let, $C_\textsc{P} \cap C_\textsc{N}$ denote the clustering (\cref{def:clustering})
induced by the true class labels (unobserved). In absence of supervision, contrastive learning relies on a set of augmentations $\gT(\cdot)$ to learn the underlying clustering. 
\begin{definition}[{\bf($\sigma, \delta$) Augmentation}] $\gT(\cdot)$ is called $(\sigma, \delta)$ augmentation if \;$\forall \ell\in \{0, 1\}: \; \exists S_\ell\subseteq C_\ell$, such that $P(\vx \in S_\ell) \geq \sigma P(\vx \in C_\ell)$ where $0 < \sigma \leq 1$ and additionally it holds that: 
\begin{flalign}
    \sup_{\vx,\vx'\in S_\ell}d_\gT(\vx, \vx') \leq \delta.
\end{flalign} 
where, $d_\gT(\vx_i, \vx_j)$ denotes the augmentation distance(\cref{def:aug_distance}) between samples from $\gT(\cdot)$. 
\label{def:aug_delta_sigma}
\end{definition}
\begin{definition}[\bf Augmentation Distance]
For an augmentation set $\gT$, augmentation distance between two samples is defined as the minimum distance between all possible augmented views of the samples. 
    \begin{flalign}
    d_\gT(\vx_i, \vx_j) = \min_{\vx_i'\in \gT(\vx_{i}),\vx_j' \in \gT(\vx_j)}\bigg\|\vx_i' - \vx_j'\bigg\|
\end{flalign}
\label{def:aug_distance}
\end{definition}
Intuitively, {\bf $(\delta, \sigma)$ measures the concentration of augmented data}. A large $\sigma$ and small $\delta$ implies sharper concentration.
\\
\\
 We further assume that augmentations are label preserving in the following sense:
\begin{assumption}[\bf Disjoint Augmentations]
    The augmentation operator $\gT$ is said to be label preserving, if samples from different latent classes never transform into the same augmented sample : 
    \begin{flalign}
        \gT(\vx_\textsc{P}) \cap \gT(\vx_N) = \varnothing \quad \forall \vx_\textsc{P}\sim p(\rx|y=1), \vx_\textsc{N}\sim p(\rx|y=0).
    \end{flalign}
    We also assume that, $\vx \in \gT(\vx) \;\forall \vx \in \sR^d$. 
    \label{assumption:aug_non_overlapping}
\end{assumption}
Note that, this assumption on augmentation~\citep{huang2023towards} is much milder compared to assuming that augmentations are unbiased samples from the same underlying class marginal~\citep{saunshi2019theoretical, tosh2021contrastive}.
\\
\\
We can now rewrite the asymptotic form of \textsc{puCL}~\eqref{eq:puCL} in terms of augmentations as:
\begin{flalign}
    \gL_{\textsc{puCL}}^\infty
    = 
    \mathop{-\E}_{(\vx, \vx') \sim p(\vx)}
    \mathop{\E}_{
        \substack{
            \vx, \vx_a \in \gT(\vx)
            \\
            \vx' \in \gT(\vx')
        }
    }
    \Biggl[\vz^T \vz_a - \log Z(\vz) \Biggr]
    \label{eq:asymptotic_puCL}
\end{flalign}
Here we have assumed that $\tau=1$ and that the labeled positives are spanned by the augmentation set. 
\\
\\
Note that, since, $\forall \vz,\vz_a\in \sR^k$, it holds that:
\begin{flalign}
    -\vz^T\vz_a = \frac{1}{2} \|\vz -\vz_a\|^2 -1
\end{flalign} 
Thus, \eqref{eq:asymptotic_puCL} can be decomposed as: 
\begin{flalign}
    \gL_{\textsc{puCL}}^\infty = \frac{1}{2}\gL_\textsc{puCL}^\textsc{I} + \gL_\textsc{puCL}^\textsc{II} -1
\end{flalign}
where, we have denoted:
\begin{flalign}
    \gL_\textsc{puCL}^\textsc{I} =\E_{\vx \sim p(\vx)}\E_{\vx, \vx_a \in \gT(\vx)} \bigg\|\vz -\vz_a\bigg\|^2 \\
    \gL_\textsc{puCL}^\textsc{II} = \E_{(\vx, \vx') \sim p(\vx)} \E_{\vx, \vx_a \in \gT(\vx),\vx' \in \gT(\vx')}\log Z(\vz)
\end{flalign}
To simplify the analysis, we analyze downstream inference on a non-parametric {\bf Nearest Neighbor (NN) classifier}, trained on pseudo-labels obtained via \textsc{puPL} (\cref{algo:puPL}). 
\\
\\
The classifier predicts:
\begin{flalign}
    \hat{F}_{g_\mB}(\vx) = \argmin_{\Mu \in \{\hat{\Mu}_\textsc{P}, \hat{\Mu}_\textsc{N}\}} \bigg\|g_\mB(\vx) - \Mu \bigg\|
\end{flalign}
where, $\hat{\Mu} = \{\hat{\Mu}_\textsc{P}, \hat{\Mu}_\textsc{N}\}$ is the estimated class centroids obtained via \textsc{puPL}.
\\
\\
It is worth noting that,
\begin{remark}
The NN classifier is a linear classifier with class centroids as weight vectors: 
    \begin{flalign*}
    F_{g_\mB}(\vx) = \argmin_{\Mu \in \hat{\Mu}} \|g_\mB(\vx) - \bm{\mu}\| = \argmax_{\Mu \in \hat{\Mu}} \bigg( \bm{\mu}^Tg_\mB(\vx) - \frac{1}{2}\|\bm{\mu}\|^2\bigg)
    \end{flalign*}
    \label{remark:kNN}
\end{remark}
Thus, we can bound~\citep{huang2023towards} the worst case performance of $v_\vv(\cdot)$ with: 
\begin{flalign}
    err(\hat{F}_{g_\mB}) = \sum_{\ell\in \{\textsc{P}, \textsc{N}\}} P\bigg(\hat{F}_{g_\mB}(\vx) \neq \ell, \; \forall \vx \in C_\ell\bigg)
    \label{eq:kNN_error}
\end{flalign}
Suppose, $S_\epsilon$ denote the set of samples with $\epsilon$-close representations among augmented data and $R_\epsilon(\gX_\textsc{PU})$ denote the probability of embeddings from the same latent class to have non-aligned augmented views, i.e.
\begin{flalign}
    S_\epsilon:= \bigg\{\vx \in C_\textsc{P} \cup C_\textsc{N} : \forall \vx,\vx_a \in \gT(\vx), \|\vz - \vz_a\| \leq \epsilon\bigg\}\\
    R_\epsilon(\gX_\textsc{PU}) = P( \bar{S_\epsilon} )
\end{flalign}
Under this setup, we establish the following generalization guarantee: 
\begin{theorem}
    Let $\gT$ be a $(\delta, \sigma)$ augmentation (\cref{def:aug_delta_sigma}), and $g_\mB(\cdot)$ be $L$ Lipschitz. Suppose, the estimated class centroids by~\cref{algo:puPL} satisfy: 
    \begin{flalign}
        \hat{\bm{\mu}}_\textsc{P}^T\hat{\bm{\mu}}_\textsc{N} < 1 - \eta(\sigma, \delta, \epsilon) - \sqrt{2 \eta(\sigma, \delta, \epsilon)} - \Delta(\mu) - \zeta_\mu
        \label{eq:class_separation_bound}
    \end{flalign}
    where, 
    \begin{flalign}
    \eta(\sigma, \delta, \epsilon) = 2(1 - \sigma) +\frac{R_\epsilon}{\min\{\pi, 1 - \pi\}} + \sigma ( L\delta + 2 \epsilon )\\
    \Delta(\mu) = \frac{1}{2} - \frac{1}{2}\min_{\ell\in\{\textsc{P},\textsc{N}\}}\|\bm{\mu}_\ell\|^2\\
    \zeta_\mu =(\zeta_\textsc{P} + \zeta_\textsc{N} + \zeta_\textsc{P}^T\zeta_\textsc{N})\\
    \zeta_\textsc{P} = \|\hat{\bm{\mu}}_\textsc{P} - \bm{\mu}_\textsc{P}\|\;,\quad \zeta_\textsc{N} = \|\hat{\bm{\mu}}_\textsc{N} - \bm{\mu}_\textsc{N}\|
    \end{flalign}
    Then, the classification error of the NN classifier is bounded by:
    \begin{flalign}
        err(\hat{F}_{g_\mB}) \leq (1 - \sigma) + R_\epsilon(\gX_\textsc{PU})
    \end{flalign}
    \label{th:generalization}
\end{theorem}
Intuitively, \cref{th:generalization} suggests that contrastive PU learning approach generalizes well when representations from the same class are tightly aligned (i.e., $R_\epsilon$ is small), and the class centroids are well separated (i.e., $\hat{\bm{\mu}}\textsc{P}^T \hat{\bm{\mu}}\textsc{N}$ is small). The term $\zeta_\mu$ captures the error in estimating class means via pseudo-labeling (\textsc{puPL}); smaller $\zeta_\mu$ implies more accurate pseudo-label assignments, leading to better downstream performance. Overall, the bound reflects that strong generalization arises from concentrated augmentations, consistent representations, and reliable pseudo-labeling.
\\
\\
We now relate the alignment error to the training objective.
\begin{lemma}\citep{huang2023towards}
The alignment error in~\cref{th:generalization} can be bounded as:
\begin{flalign}
    R_\epsilon(\gX_\textsc{PU}) \leq \eta'(\epsilon, \gT) \sqrt{\gL_\textsc{puCL}^\textsc{I}(\gX_\textsc{PU})}
\end{flalign}
where,
\begin{flalign}
    \eta'(\epsilon, \gT)=\inf_{h\in (0, \frac{\epsilon}{2\sqrt{d}LM})}\frac{4\max(1, m^2h^2d)}{h^2d(\epsilon - 2\sqrt{d}LMh)}
\end{flalign}
for $\gT$ composed of $M$-Lipschitz continuous transformations and $m$ discrete transformations.
\label{lemma:bound_align_puCL}
\end{lemma}

\begin{lemma}
    The condition in~\cref{th:generalization} on the separation of the estimated class centroids~\eqref{eq:class_separation_bound} is satisfied, whenever:
\begin{flalign}
    \log \bigg(\exp\bigg(\gL_\textsc{puCL}^\textsc{II}(\gX_\textsc{PU}) + c(\sigma, \delta, \epsilon, R_\epsilon)\bigg) + c'(\epsilon)\bigg) \nonumber\\
    <
    1 - \eta(\sigma, \delta, \epsilon) - \sqrt{2 \eta(\sigma, \delta, \epsilon)} - \frac{1}{2}\Delta(\mu) -\zeta_\mu.
\end{flalign}
where, we have denoted: 
\begin{flalign}
    c(\sigma, \delta, \epsilon, R_\epsilon) =(2\epsilon + L\delta + 4(1-\sigma) + 8R_\epsilon)^2 + 4\epsilon + 2L\delta + 8(1 - \sigma) + 18R_\epsilon.
    \\
    c'(\epsilon) = \exp\frac{1}{\pi_p(1-\pi_p)} - \exp(1 - \epsilon).
\end{flalign}
\label{lemma:bound_centroid_divergence_puCL}
\end{lemma}
Together,~\cref{lemma:bound_align_puCL},\ref{lemma:bound_centroid_divergence_puCL} imply that by minimizing $\gL_\textsc{puCL} = \gL_\textsc{puCL}^\textsc{I} + \gL_\textsc{puCL}^\textsc{II}$, we can expect improved generalization through two mechanisms: smaller alignment error $R_\epsilon$ (\cref{lemma:bound_align_puCL}), which consequently results in larger deviation between class centers (\cref{th:generalization}, \cref{lemma:bound_centroid_divergence_puCL}). Furthermore, the labeling error $\zeta_\mu$ arising from \textsc{puPL} is also small when the representation space is well-clustered. 
\\
\\
Therefore, under mild regularity assumptions on $\gT$, the contrastive approach yields:
\begin{flalign}
    err(\hat{F}_{g_\mB}) \leq (1 - \sigma) + \eta'(\epsilon, \gT) \sqrt{\gL_\textsc{puCL}^\textsc{I}(\gX_\textsc{PU})}
\end{flalign}
whenever the centroid separation condition in \cref{lemma:bound_centroid_divergence_puCL} is satisfied. 
\\
\\
This provides a theoretically grounded characterization of when and why contrastive learning is effective in PU settings.
Detailed proofs can be found in~\cref{sec:proof.generalization} and \ref{sec:proof.bound_centroid_divergence_puCL}.

\section{Experiments}
\label{sec:exp}
In this section, we describe our experimental setup, present empirical findings, and provide insights from our results. 
We organize our experiments into three parts: 
\begin{itemize}
    \item {\bf PU Benchmarks}: Comparing our end-to-end contrastive PU learning approach with popular PU learning baselines across multiple benchmark datasets, 
    \item {\bf Ablations on Contrastive Representation Learning}: Investigating the behavior of different contrastive objectives discussed in \cref{sec:puCL}, and 
    \item {\bf Ablations on Downstream Classification}: Exploring how various post-contrastive classification strategies affect performance.
\end{itemize}
For all the experiments, contrastive pre-training is done using LARS optimizer~\citep{you2019large}, cosine annealing schedule with linear warm-up, batch size 1024, initial learning rate 1.2. 
We use a 128 dimensional projection layer $h_{\mGamma}(\cdot)$ composed of two linear layers with relu activation and batch normalization. 
We leverage Faiss~\citep{johnson2019billion} for efficient implementation of \textsc{puPL}. To ensure reproducibility, all experiments are run with deterministic cuDNN back-end and repeated 5 times with different random seeds and the confidence intervals are noted. 

\subsection{PU Learning Benchmark.}  
\begin{table*}
\footnotesize
\centering
\resizebox{\textwidth}{!}
{
\begin{tabular}{lccccccc}
\toprule
\multirow{3}{*}{\bf Methods} 
& \multicolumn{6}{c}{\bf Datasets} 
& \multirow{3}{*}{\bf Mean }
\\ 
\cmidrule{2-7}
& {\bf F-MNIST-I}
& {\bf F-MNIST-II}
& {\bf CIFAR-I}
& {\bf CIFAR-II} 
& {\bf STL-I}
& {\bf STL-II} 
&                         
\\
& ($\pi_p^*=0.3$)
& ($\pi_p^*=0.7$)
& ($\pi_p^*=0.4$)
& ($\pi_p^*=0.6$)
& ($\hat{\pi}_p=0.51$)
& ($\hat{\pi}_p=0.49$)
&                         
\\
\midrule \midrule
&\multicolumn{6}{c}{$n_P=1k$} 
&                   
\\
\cmidrule{2-7}   
\textsc{uPU$^\dagger$}
& $71.3\textcolor{gray}{\pm 1.4}$
& $84.0\textcolor{gray}{\pm 4.0}$
& $76.5\textcolor{gray}{\pm 2.5}$
& $71.6\textcolor{gray}{\pm 1.4}$ 
& $76.7\textcolor{gray}{\pm 3.8}$
& $78.2\textcolor{gray}{\pm 4.1}$
& $76.4$
\\
\textsc{nnPU$^\dagger$}
& $89.7\textcolor{gray}{\pm 0.8}$
& $88.8\textcolor{gray}{\pm 0.9}$
& $84.7\textcolor{gray}{\pm 2.4}$
& $83.7\textcolor{gray}{\pm 0.6}$ 
& $77.1\textcolor{gray}{\pm 4.5}$
& $80.4\textcolor{gray}{\pm 2.7}$
& $84.1$
\\
\textsc{nnPU$^\dagger$ (MixUp)}
& $91.4\textcolor{gray}{\pm 0.3}$
& $88.2\textcolor{gray}{\pm 0.7}$
& $87.2\textcolor{gray}{\pm 0.6}$
& $85.8\textcolor{gray}{\pm 1.2}$ 
& $79.8\textcolor{gray}{\pm 0.8}$
& $82.2\textcolor{gray}{\pm 0.9}$
& $85.8$
\\
\textsc{Self-PU$^\dagger$}
& $90.8\textcolor{gray}{\pm 0.4}$
& $89.1\textcolor{gray}{\pm 0.7}$
& $85.1\textcolor{gray}{\pm 0.8}$
& $83.9\textcolor{gray}{\pm 2.6}$ 
& $78.5\textcolor{gray}{\pm 1.1}$
& $80.8\textcolor{gray}{\pm 2.1}$
& $84.7$
\\
\textsc{PAN}
& $88.7\textcolor{gray}{\pm 1.2}$
& $83.6\textcolor{gray}{\pm 2.5}$
& $87.0\textcolor{gray}{\pm 0.3}$
& $82.8\textcolor{gray}{\pm 1.0}$ 
& $77.7\textcolor{gray}{\pm 2.5}$
& $79.8\textcolor{gray}{\pm 1.4}$
& $83.3$
\\
\textsc{vPU$^\dagger$}
& $90.6\textcolor{gray}{\pm 1.2}$
& $86.8\textcolor{gray}{\pm 0.8}$
& $86.8\textcolor{gray}{\pm 1.2}$
& $82.5\textcolor{gray}{\pm 1.1}$ 
& $78.4\textcolor{gray}{\pm 1.1}$
& $82.9\textcolor{gray}{\pm 0.7}$
& $84.7$
\\
\textsc{MixPUL}
& $87.5\textcolor{gray}{\pm 1.5}$
& $89.0\textcolor{gray}{\pm 0.5}$
& $87.0\textcolor{gray}{\pm 1.9}$
& $87.0\textcolor{gray}{\pm 1.1}$ 
& $77.8\textcolor{gray}{\pm 0.7}$
& $78.9\textcolor{gray}{\pm 1.9}$
& $84.5$
\\
\textsc{puLNS}
& $90.7\textcolor{gray}{\pm 0.5}$
& $87.9\textcolor{gray}{\pm 0.5}$
& $87.2\textcolor{gray}{\pm 0.6}$
& $83.7\textcolor{gray}{\pm 2.9}$ 
& $80.2\textcolor{gray}{\pm 0.8}$
& $83.6\textcolor{gray}{\pm 0.7}$
& $85.6$
\\
P$^3$\textsc{mix-e}
& $91.9\textcolor{gray}{\pm 0.3}$
& $\bf 89.5\textcolor{gray}{\pm 0.5}$
& $88.2\textcolor{gray}{\pm 0.4}$
& $84.7\textcolor{gray}{\pm 0.5}$ 
& $80.2\textcolor{gray}{\pm 0.9}$
& $83.7\textcolor{gray}{\pm 0.7}$
& $86.4$
\\
P$^3$\textsc{mix-c}
& $\bf 92.0\textcolor{gray}{\pm 0.4}$
& $89.4\textcolor{gray}{\pm 0.3}$
& $88.7\textcolor{gray}{\pm 0.4}$
& $87.9\textcolor{gray}{\pm 0.5}$ 
& $80.7\textcolor{gray}{\pm 0.7}$
& $84.1\textcolor{gray}{\pm 0.3}$
& $87.1$
\\
\rowcolor{light-gray}
\textsc{puCL(puPL)}
& $91.8\textcolor{gray}{\pm 0.8}$
& $89.2\textcolor{gray}{\pm 0.3}$
& $\bf 92.3\textcolor{gray}{\pm 1.9}$
& $\bf 91.2\textcolor{gray}{\pm 0.5}$
& $\bf 83.8\textcolor{gray}{\pm 1.4}$
& $\bf 84.5\textcolor{gray}{\pm 0.7}$
& $\bf 88.8$
\\
\rowcolor{light-gray}
\textsc{puNCE(puPL)}$^\dagger$
& $\bf 92.1\textcolor{gray}{\pm 1.6}$
& $\bf 90.7\textcolor{gray}{\pm 2.1}$
& $\bf 93.0\textcolor{gray}{\pm 1.8}$
& $\bf 93.8\textcolor{gray}{\pm 0.9}$
& $\bf 85.1\textcolor{gray}{\pm 1.2}$
& $\bf 85.0\textcolor{gray}{\pm 1.5}$
& $\bf 90.0$
\\
\bottomrule
&\multicolumn{4}{c}{$n_P=3k$} 
&\multicolumn{2}{c}{$n_P=2.5k$} 
&
\\
\cmidrule(lr){2-5}
\cmidrule(lr){6-7}
\textsc{uPU$^\dagger$}
& $89.9\textcolor{gray}{\pm 1.0}$
& $78.6\textcolor{gray}{\pm 1.3}$
& $80.6\textcolor{gray}{\pm 2.1}$
& $72.9\textcolor{gray}{\pm 3.2}$
& $70.3\textcolor{gray}{\pm 2.0}$
& $74.0\textcolor{gray}{\pm 3.0}$
& $77.7$
\\
\textsc{nnPU$^\dagger$}
& $90.8\textcolor{gray}{\pm 0.6}$
& $90.5\textcolor{gray}{\pm 0.4}$
& $85.6\textcolor{gray}{\pm 2.3}$
& $85.5\textcolor{gray}{\pm 2.0}$ 
& $78.3\textcolor{gray}{\pm 1.2}$
& $82.2\textcolor{gray}{\pm 0.5}$
& $85.5$
\\
\textsc{RP}
& $92.2\textcolor{gray}{\pm 0.4}$
& $75.9\textcolor{gray}{\pm 0.6}$
& $86.7\textcolor{gray}{\pm 2.9}$
& $77.8\textcolor{gray}{\pm 2.5}$
& $67.8\textcolor{gray}{\pm 4.6}$
& $68.5\textcolor{gray}{\pm 5.7}$
& $78.2$
\\
\textsc{vPU$^\dagger$}
& $\bf 92.7\textcolor{gray}{\pm 0.3}$
& $\bf 90.8\textcolor{gray}{\pm 0.6}$
& $89.5\textcolor{gray}{\pm 0.1}$
& $88.8\textcolor{gray}{\pm 0.8}$ 
& $79.7\textcolor{gray}{\pm 1.5}$
& $83.7\textcolor{gray}{\pm 0.1}$
& $87.5$
\\
\rowcolor{light-gray}
\textsc{puCL(puPL)}
& $92.0\textcolor{gray}{\pm 0.7}$
& $89.6\textcolor{gray}{\pm 1.2}$
& $\bf 93.5\textcolor{gray}{\pm 0.8}$
& $\bf 93.8\textcolor{gray}{\pm 0.4}$
& $\bf 85.0\textcolor{gray}{\pm 0.9}$
& $\bf 85.2\textcolor{gray}{\pm 2.1}$
& $\bf 89.9$
\\
\rowcolor{light-gray}
\textsc{puNCE(puPL)}$^\dagger$
& $\bf 92.6\textcolor{gray}{\pm 0.4}$
& $\bf 91.9\textcolor{gray}{\pm 0.7}$
& $\bf 94.4\textcolor{gray}{\pm 1.1}$
& $\bf 94.6\textcolor{gray}{\pm 0.8}$
& $\bf 88.2\textcolor{gray}{\pm 0.3}$
& $\bf 86.9\textcolor{gray}{\pm 1.2}$
& $\bf 91.4$
\\
\bottomrule
\end{tabular}}
\caption{
\footnotesize 
{\bf PU Learning Benchmarks.} We compare our approach against several PU Learning baselines algorithms over different datasets and different amount of labeled data. Our setup is identical as~\citep{li2022your, chen2020variational}. $\dagger$: These methods were run with oracle class prior knowledge.}

\label{tab:vision_benchmark}
\end{table*}
We benchmark the overall simple and effective contrastive PU learning framework proposed in this paper. The framework comprises contrastive representation learning using \textsc{puCL} when the class prior is unknown or unavailable, or its prior-aware variant \textsc{puNCE} when a reliable class prior estimate is available. For downstream classification, we apply the clustering-based pseudo-labeling module \textsc{puPL}, which assigns semantic labels to unlabeled samples in the learned representation space.
\\
\\
The {\bf first set of benchmark experiments (\cref{tab:vision_benchmark})}, closely follow the experimental setup of~\citep{li2022your, chen2020variational}. 
We compare our method against several widely-used \textbf{PU learning baselines}, including: 
\textsc{uPU}~\citep{du2014analysis}, \textsc{nnPU}~\citep{kiryo2017positive}, \textsc{nnPU} with \textsc{MixUp}~\cite{zhang2017mixup}, 
\textsc{Self-PU}~\cite{chen2020self}, 
\textsc{PAN}~\citep{hu2021predictive}, \textsc{vPU}~\citep{chen2020variational}, \textsc{MixPUL}~\citep{wei2020mixpul}, \textsc{PULNS}~\citep{luo2021pulns} and \textsc{RP}~\citep{northcutt2017learning}. 
We evaluate these approaches across six \textbf{benchmark datasets}: STL-I, STL-II, CIFAR-I, CIFAR-II, FMNIST-I, and FMNIST-II; derived from STL-10~\citep{coates2011analysis}, CIFAR-10~\citep{krizhevsky2009learning}, and Fashion-MNIST~\citep{xiao2017fashion}. Baselines that rely on class prior are provided oracle knowledge. For CIFAR-I, II and FMNIST-I, II, the class priors $\pi_p^\ast$ are set to 0.4, 0.6, 0.3, and 0.7 respectively. For STL-I and STL-II, where priors are not directly available, we estimate them using KM2~\citep{ramaswamy2016mixture}, following~\citep{li2022your}, resulting in priors of 0.51 and 0.49.
We use {\bf LeNet-5}~\citep{lecun1998gradient} for FMNIST experiments, and a {\bf 7-layer CNN}~\citep{chen2020variational, li2022your} for STL and CIFAR experiments. Baseline results for $n_\textsc{P} = 1k$ are taken from~\citep{li2022your}; the rest are from~\citep{chen2020variational}. Empirical results are presented in \cref{tab:vision_benchmark}.
\\
\\
In our {\bf second set of benchmark experiments (\cref{tab:benchmark_2})}, we follow the evaluation setup of~\citep{yuan2025weighted}. We compare our method with an extended set of PU baselines: 
\textsc{RP}~\citep{northcutt2017learning}, \textsc{PUSB}~\citep{kato2018learning}, \textsc{PUbN}~\citep{hsieh2019classification}, \textsc{aPU}~\citep{hammoudeh2020learning}, \textsc{ImbPU}~\citep{su2021positive}, \textsc{DistPU}~\citep{zhao2022dist}, \textsc{PiCO}~\citep{wang2021pico}, and \textsc{WConPU}~\citep{yuan2025weighted};
in addition to previously mentioned baselines
\textsc{uPU}, 
\textsc{nnPU}, 
\textsc{Self-PU} and 
\textsc{vPU}. 
\begin{table}[t]
\centering
\footnotesize
\resizebox{0.75\textwidth}{!}{
\begin{tabular}{lccccc}
\toprule
\multirow{2}{*}{\bf Method} 
& \multicolumn{4}{c}{\bf Dataset} 
& \multirow{2}{*}{\bf Mean} 
\\
\cmidrule{2-5}
& 
\textbf{CIFAR-III} & 
\textbf{SVHN-I} & 
\textbf{STL-III} & 
\textbf{Alzheimer} & \\
\midrule
\textsc{uPU}$^\dagger$       
& 88.41{$\pm$0.41} 
& 83.35{$\pm$0.45} 
& 93.13{$\pm$0.46} 
& 68.42{$\pm$2.22} 
& 83.83 \\
\textsc{nnPU}$^\dagger$      
& 88.91{$\pm$0.43} 
& 83.88{$\pm$0.45} 
& 93.38{$\pm$0.42} 
& 68.21{$\pm$2.15} 
& 83.60 
\\
\textsc{RP}       
& 88.74{$\pm$0.46} 
& 81.73{$\pm$0.15} 
& 92.88{$\pm$0.65} 
& 63.02{$\pm$2.33} 
& 81.59 
\\
\textsc{PUSB}      
& 88.97{$\pm$0.39} 
& 83.99{$\pm$0.41} 
& 93.65{$\pm$0.16} 
& 69.19{$\pm$2.41} 
& 83.95 
\\
\textsc{PUbN}      
& 89.83{$\pm$0.30} 
& 84.89{$\pm$0.30} 
& 94.01{$\pm$0.31} 
& 70.00{$\pm$1.17} 
& 84.68 
\\
\textsc{Self-PU}$^\dagger$   
& 89.31{$\pm$0.56} 
& 84.12{$\pm$0.72} 
& 93.73{$\pm$0.28} 
& 70.79{$\pm$0.73} 
& 84.49 
\\
\textsc{aPU}     
& 89.09{$\pm$0.44} 
& 84.01{$\pm$0.52} 
& 93.41{$\pm$0.45} 
& 68.41{$\pm$1.71} 
& 83.73 
\\
\textsc{vPU}$^\dagger$     
& 87.89{$\pm$0.69} 
& 76.89{$\pm$0.74} 
& 91.51{$\pm$0.65} 
& 63.16{$\pm$2.20} 
& 79.86 
\\
\textsc{ImbPU}     
& 89.43{$\pm$0.42} 
& 84.20{$\pm$0.46} 
& 93.88{$\pm$0.81} 
& 68.19{$\pm$0.69} 
& 83.93 
\\
\textsc{Dist-PU} 
& 91.88{$\pm$0.36} 
& 85.10{$\pm$0.33}
& 94.73{$\pm$0.31} 
& 71.57{$\pm$0.62} 
& 85.82 
\\
\textsc{PiCO}     
& 95.64{$\pm$0.12} 
& 89.01{$\pm$0.34} 
& 95.55{$\pm$0.23} 
& 71.94{$\pm$0.71} 
& 88.53 
\\
\textsc{WConPU}$^\dagger$ 
& 97.22{$\pm$0.15} 
& \textbf{91.49{$\pm$0.29}} 
& \textbf{97.02{$\pm$0.21}} 
& 73.02{$\pm$0.66}
& \textbf{89.69} 
\\
\rowcolor{light-gray}
\textsc{puCL(puPL)}
& ${\bf 97.37\textcolor{gray}{\pm 0.51}}$
& 90.80{$\pm$1.01}
& 96.22{$\pm$0.95}
& {\bf 73.91{$\pm$0.86}}
& 89.58
\\
\rowcolor{light-gray}
\textsc{puNCE(puPL)}$^\dagger$
& ${\bf 98.35\textcolor{gray}{\pm 0.68}}$
& ${\bf 92.25\textcolor{gray}{\pm 1.24}}$
& ${\bf 97.50\textcolor{gray}{\pm 0.98}}$
& ${\bf 75.39\textcolor{gray}{\pm 1.27}}$
& ${\bf 90.87}$
\\
\bottomrule
\end{tabular}
}
\caption{
\footnotesize
\textbf{Benchmark Accuracy ($\%$ ± std) across PU Datasets.}
Experimental setup is identical to~\citep{yuan2025weighted}. $\dagger$: These methods were run with oracle class prior knowledge.
}
\label{tab:benchmark_2}
\end{table}

We report performance on {\bf four additional benchmark datasets:} CIFAR-III (vehicles vs. animals), SVHN-I (even vs. odd), STL-III (vehicles vs. animals), and the publicly available subset~\citep{alzheimer_mri_dataset} of Alzheimers Disease Neuroimaging Initiative (ADNI)\footnote{https://adni.loni.usc.edu/.}. The respective class priors are 0.4, 0.46, 0.4, and 0.5.
Closely following~\citep{yuan2025weighted}, we use a {\bf 13-Layer CNN} for both CIFAR-III and SVHN-I. For the experiments on STL-I, we train with a {\bf ResNet-18}, whereas, the experiments on Alzheimers dataset are performed with {\bf ResNet-50},
\\
\\
Across both benchmark suites, the proposed simple contrastive PU learning framework, comprising of -- {\bf puCL or its prior-aware variant puNCE (when reliable estimate of $\pi$ is available) -- followed by downstream pseudo-labeling via puPL consistently outperforms existing PU learning methods}. 
Even without access to class prior information, puCL combined with puPL remains highly competitive, surpassing a range of baselines that rely on heuristic reweighting, complex sample selection, or strong prior assumptions. 
When class prior is available, the puNCE variant leads to further gains, improving over previous state-of-the-art by $\approx 1\%$ absolute test accuracy, averaged across the four datasets. 
\\
\\
These improvements validate the key finding: {\bf judiciously injecting weak supervision into the contrastive objective and leveraging the geometry of the learned representation space via clustering can yield superior generalization compared to traditional PU risk estimators}. Notably, the benefits become more pronounced with increasing supervision, but the framework remains robust even in low-data regimes, offering a simple, scalable, and principled alternative to prior PU learning pipelines. 
\\
\\
Additional details on baselines and datasets can be found in~\cref{sec:add_exp}. 
\begin{figure*}[t] 
\footnotesize 
\centering
\subfloat
[\bf varying $\gamma$ with fixed $n_\textsc{U}$, $\pi$]
{\includegraphics[width=0.38\textwidth]{plots/kappa_variation/gamma.pdf}}
\subfloat
[\bf varying $\pi$ with fixed $n_\textsc{U}$, $\gamma$]
{\includegraphics[width=0.38\textwidth]{plots/kappa_variation/pi.pdf}}
\caption{
\footnotesize 
{\bf Ablations on $\kappa$:}
We train a ResNet-34 encoder on ImageNet-I, while verifying two key parameters contributing to the bias-variance tradeoff (a) $\gamma$, and (b) $\pi$ in isolation, while keeping the other factors fixed. 
}
\label{fig:kappa_var_isolation}
\end{figure*}

\subsection{Ablations on Contrastive Representation Learning from PU data.}  
\begin{figure*}[t]
\small 
\centering
\subfloat
[\footnotesize {\bf ImageNet-II} ]
{\includegraphics[width=0.38\textwidth]{plots/ablations/imageNet_nP_contrastive.pdf}}
\subfloat
[\footnotesize {\bf CIFAR-0}]
{\includegraphics[width=0.38\textwidth]{plots/ablations/CIFAR-0_nP_contrastive.pdf}}
\caption{\footnotesize {\bf  Generalization ($\gamma$): }  In this experiment we train a ResNet-18 on CIFAR-0 (Subset of Dogs and Cats) and ImageNet-II (ImageWoof vs ImageNette). Number of unlabeled samples $n_\textsc{U}$ is kept fixed, while we vary the number of labeled positives $n_\textsc{P}$. 
In both settings, we find \textsc{puCL} to remain robust across different levels of supervision while consistently outperforming its unsupervised counterpart \textsc{ssCL} and being competitive with \textsc{sCL-PU} even in high supervision regimes. While, \textsc{sCL-PU} suffers from large degradation especially in the low-supervision regime. 
}
\label{fig:nP_contrastive}
\end{figure*}
A central goal of our ablation experiments is to systematically understand how different contrastive objectives behave in the Positive-Unlabeled (PU) setting, particularly with respect to incorporating weak supervision to {\bf improve generalization while ensuring robustness}. We primarily compare \textsc{puCL}~\eqref{eq:puCL} and its prior-aware variant \textsc{puNCE}~\eqref{eq:puNCE} with two natural baselines - unsupervised \textsc{ssCL}~\eqref{eq:sscl}, and supervised \textsc{sCL-PU}~\eqref{eq:scl_pu}. 
We also discuss other weakly supervised objectives including \textsc{dCL}~\citep{chuang2020debiased}~\eqref{eq:dcl}, \textsc{mCL}~\citep{cui2023rethinking}~\eqref{eq:mcl} 
and compare them with \textsc{puCL}. 
\\
\\
Our theoretical analysis (\cref{th:scl_pu_bias}) identifies a dataset specific quantity:
\begin{equation}
    \kappa = {\pi (1 - \pi)}/({ 1 + \gamma})
\end{equation}
playing a crucial role in the bias variance trade-off of incorporating weak supervision. 
\\
\\
To gain deeper insights into the role of weak supervision, we conduct systematic ablations across different settings of the PU-specific parameter $\kappa$, and evaluate downstream generalization performance. We perform experiments on {\bf three additional datasets} -- {\bf ImageNet-I}: a subset of dog (P) vs non-dog (N) images sampled from ImageNet-1k~\citep{hua2021feature, robustness}; {\bf ImageNet-II}: Imagewoof (P) vs ImageNette (N) -- two subsets from ImageNet-1k~\citep{fastai_imagenette} and {\bf CIFAR-0}: dog (P) vs cat (N), two semantically similar classes of CIFAR-10. 

\subsection*{Generalization (\(\gamma\)).}
To isolate the effect of labeled positives, we fix $n_\textsc{U}$ and $\pi$, while varying $n_\textsc{P}$, resulting in a range of values of $\gamma ={n_\textsc{P}}/{n_\textsc{U}}$ across experiments. In{\bf ~\cref{fig:kappa_var_isolation}(a)}, we train a ResNet-34 over ImageNet-I 
In {\bf \cref{fig:nP_contrastive}}, we repeat the experiment for training a ResNet-18 over CIFAR-0 and Imagenet-I. 
\\
\\
Both of our experiments suggest, {\bf \textsc{ssCL}} remains robust across different levels of supervision but shows minimal improvement as $\gamma$ increases, due to its inability to utilize labeled data.
On the other hand, {\bf \textsc{sCL-PU}} significantly outperforms \textsc{ssCL} when $\gamma$ is large, leveraging abundant positives effectively. However, it exhibits severe performance degradation in the low-supervision regime, where bias in the supervised objective becomes detrimental. 
{\bf \textsc{puCL}}, in contrast, smoothly interpolates between these two extremes. It matches the performance of \textsc{sCL-PU} under high supervision, while retaining the robustness of \textsc{ssCL} under scarce labeled data. Notably, the performance gap between \textsc{puCL} and \textsc{ssCL} increases as $\gamma$ increases, highlighting \textsc{puCL}'s ability to effectively leverage even moderate amounts of supervision without compromising robustness.

\subsection*{Generalization($\pi$)}
We next investigate how the class prior $\pi$ affects downstream performance. To isolate the effect, 
in {\bf \cref{fig:kappa_var_isolation}(b)},  we fix $\gamma$ and $n_\textsc{U}$ while using different $\pi$. The unlabeled set is constructed by mixing $\pi_p n_\textsc{U}$ positives and $(1 - \pi) n_\textsc{U}$ negatives~\footnote{Positives and negatives are sub-sampled from known ground-truth labels solely for experimental evaluation; {\bf the model has no access to this information}.}.
\\
\\
As predicted by our theory in{\bf ~\cref{th:scl_pu_bias}}, the robustness of supervised contrastive learning in the PU setting deteriorates as $\kappa \propto \pi(1-\pi)$ increases and should be maximum when $\pi=1/2$. Furthermore, perhaps more interestingly, we observe that as $\pi_p \to 1$, the performance of \textsc{sCL-PU} collapses. We hypothesize that this degradation stems from the scarcity of hard negatives in the unlabeled set at high $\pi$. The supervised contrastive loss overestimates intra-class similarity due to the lack of inter-class contrast, resulting in a larger discrepancy $(\rho_{intra} - \rho_{inter})$. This causes the learned representations to collapse or become poorly clustered. In contrast, both \textsc{puCL} and \textsc{ssCL}, being unbiased and not reliant on potentially misleading pseudo-negatives, 
maintain stable performance even in such imbalanced scenarios. 

\subsection*{Generalization($\pi, \gamma$)}
Finally, in{\bf ~\cref{fig:kappa_var_3D}}, reports generalization, when both $\gamma$ and $\pi$ are jointly varied. The resulting 3D visualization in panel (b), along with its 2D projections in panel (c), illustrates how the effectiveness of each contrastive objective evolves across different regions of the $(\pi, \gamma)$ space, parameterized via $\kappa$. Notably, while \textsc{sCL-PU} suffers sharp degradation in high $\pi$,low $\gamma$ regimes, \textsc{puCL} maintains consistently strong performance, validating its robustness across a wide spectrum of PU scenarios.
\\
\\
These results underscore the fragility of supervised objectives under extreme class imbalance, and demonstrate the advantage of PU-aware formulations like \textsc{puCL}, which remain robust without relying on strong assumptions about the unlabeled data distribution.

\subsection*{Convergence($\gamma$)}
Incorporating available positives in the loss, not only improves the generalization performance, it also improves the convergence of representation learning from PU data. As verified empirically in {\bf \cref{fig:puCL_convergence}}, \textsc{puCL} exhibits substantially faster convergence compared to \textsc{ssCL}, particularly as the supervision ratio $\gamma$ increases.
\\
\\
We attribute this to reduced sampling bias~\citep{chuang2020debiased}, as suggested by our gradient analysis in~\cref{sec:proof.gradient}. By leveraging multiple labeled positives during contrastive pair construction, \textsc{puCL} produces a more representative and stable set of positive anchors, thereby reducing gradient variance and improving learning stability. This results in faster and more consistent convergence, even under limited supervision.

\subsection*{Hard to Distinguish Classes}
\begin{figure*}[t] 
\footnotesize
\centering
\subfloat
[\bf \textsc{puCL($\gamma$)}]
{\includegraphics[width=0.33\textwidth]{plots/hard-to-classify/cifar_hardness_nP.pdf}}
\subfloat[\bf Gains (PU supervision)]
{\includegraphics[width=0.33\textwidth]{plots/hard-to-classify/cifar_hardness_gains.pdf}}
\subfloat
[\bf Gains (full supervision)]
{\includegraphics[width=0.33\textwidth]{plots/hard-to-classify/cifar_hardness_gains_sup.pdf}}
\caption{
\footnotesize 
{\bf  Grouping dissimilar objects together:}  We train a ResNet-18 on CIFAR-hard -- (airplane, cat) vs (bird, dog), CIFAR-easy -- (airplane, bird) vs (cat,  dog) and CIFAR-medium -- (airplane, cat, dog) vs bird. Note that,  airplane and bird are semantically similar, also dog-cat are semantically closer to each other. We repeat the experiments across different supervision levels - amount of supervision is measured with $\gamma = \frac{n_P}{n_U}$. 
We keep the total number of samples $N = n_\textsc{P} + n_\textsc{U}$ fixed, while varying $n_\textsc{P}$. 
Observe that, (a) shows generalization of \textsc{puCL} across different $\gamma$. (b), (c) denote the performance gains of \textsc{puCL} and fully supervised \textsc{sCL} over unsupervised \textsc{ssCL}. Clearly, in the hard setting, \textsc{ssCL} i.e. \textsc{puCL}($\gamma =0$), suffers from large performance degradation. However, given enough supervision signal \textsc{puCL} is still able to learn representations that preserves class label obeying linear separability.   
}
\label{fig:hard-to-classify}
\end{figure*}
Distinguishing between semantically similar objects i.e. when $p(\rx)$ contains insufficient information about $p(\ry|\rx)$ is a difficult task, especially for unsupervised learning e.g. "cats vs dogs" is harder than "dogs vs table". Indeed, this implies that the augmentations are weakly concentrated, resulting in potentially poor generalization as suggested by \cref{th:generalization}. 
\\
\\
In order to investigate this scenario more closely, we experiment with three CIFAR subsets: CIFAR-hard (airplane, cat vs. bird, dog), CIFAR-easy (airplane, bird vs. cat, dog), and CIFAR-medium (airplane, cat, dog vs. bird); carefully crafted to simulate varying degrees of classification difficulty based on semantic proximity. Notably, airplanes and birds, as well as dogs and cats, are semantically close.
\\
\\
We train a ResNet-18 on each of these settings, keeping the total number of samples fixed $N = n_\textsc{P} + n_\textsc{U}$ fixed, while varying $n_\textsc{P}$. As evidenced in {\bf \cref{fig:hard-to-classify}}, 
As shown in~\cref{fig:hard-to-classify}, the benefit of incorporating labeled positives via \textsc{puCL} is significantly more pronounced in harder settings, where semantic similarity alone is insufficient to reliably separate classes. In particular: Panel (a) reports the generalization performance of \textsc{puCL} across different values of $\gamma$, Panels (b) and (c) show the performance gain of \textsc{puCL} and fully supervised \textsc{sCL}, respectively, over the unsupervised baseline \textsc{ssCL}. Interestingly, the advantage from incorporating the additional weak supervision is more pronounced in scenarios where distinguishing between positive and negative instances based solely on semantic similarity proves insufficient. Furthermore, when an adequate number of labeled positives is available, the generalization gains are comparable to those achieved with full supervision.  

\subsection*{Gains from incorporating class prior knowledge}
In {\bf \cref{sec:puNCE}}, we proposed {\bf \textsc{puNCE}} -- a prior aware variant of the contrastive objective, where in addition to incorporating the labeled positives, we also leverage additional side information $\pi$ to improve representation. 
\\
\\
We investigate the gains from introducing such inductive bias in {\bf ~\cref{figure:puNCE_tsne},\ref{figure:puNCE_vs_puCL},~\cref{tab:puNCE}} across 4 PU datasets (\cref{sec:add_exp}): MNIST-I, CIFAR-0, III, ImageNet-I. To isolate the contribution of class prior knowledge, we compare \textsc{puNCE} to \textsc{puCL} (which does not use $\pi$), \textsc{ssCL} (fully unsupervised), and \textsc{dCL} (which also incorporates $\pi$ but via partition function debiasing). We ensure fair comparison by using identical augmentations, optimizer settings, batch sizes, and temperature hyper-parameters across all methods. All models are trained for a fixed number of epochs and evaluated under the same protocol.  We observe that incorporating such inductive bias consistently improves both the quality of the learned representations and downstream classification performance. 
\begin{figure*}[t]
\centering
\subfloat{\includegraphics[width=0.45\textwidth]{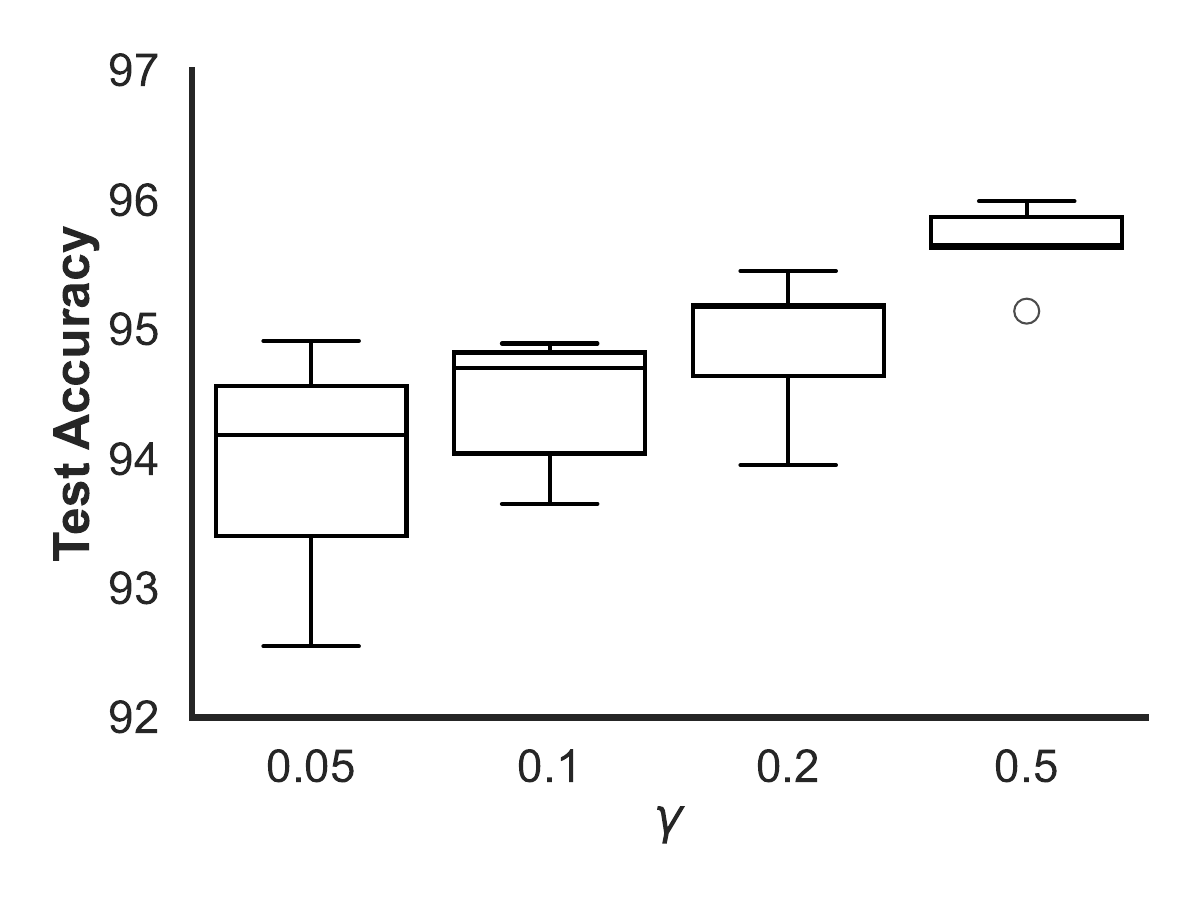}}
\caption{
\footnotesize 
{\bf \textsc{puNCE} under class prior misspecification}. 
ResNet-18 trained on ImageNet-II (\cref{sec:exp}) via \textsc{puNCE}~\eqref{eq:puNCE} when class prior estimate is noisy $\hat{\pi} \in (1 + \epsilon)\pi^\ast$ where, $\epsilon \in (0, 0.3)$.   
}
\label{figure:puNCE_pi}
\end{figure*}
\\
\\
In {\bf ~\cref{figure:puNCE_pi}}, we evaluate the robustness of \textsc{puNCE} to {\bf class prior misspecification} on the ImageNet-II benchmark. Specifically, we perturb the true class prior $\pi^\ast$ by introducing multiplicative noise $\epsilon \in (0, 0.2)$, evaluating performance with $\hat{\pi} \in {(1 \pm \epsilon)\pi^\ast}$. Despite the injected noise, \textsc{puNCE} maintains competitive performance across all values of $\gamma$, with only mild degradation in accuracy under extreme prior shifts. Notably, the variance is highest in the low-supervision regime ($\gamma = 0.05$), where the model is more reliant on the prior to compensate for scarce labeled positives. As supervision increases, the accuracy stabilizes and the impact of prior perturbation diminishes — highlighting its capacity to gracefully absorb moderate inaccuracies in prior estimates.

\subsection{Ablations on Downstream Classification}
\begin{figure}[t]
\footnotesize 
\centering
\subfloat
[\bf ImageNet-II]
{\includegraphics[width=0.4\textwidth]{plots/puPL_LP/imageNet-ii_LP.pdf}}
\subfloat
[\bf CIFAR-Hard]
{\includegraphics[width=0.4\textwidth]{plots/puPL_LP/cifar_hard_LP.pdf}}
\caption{\footnotesize {\bf Linear Probing ($\gamma$) :} In order to study downstream classification in isolation, we
perform Logistic Regression on pretrained (frozen) ResNet-18 embeddings obtained via \textsc{puCL}(\cref{algo:puCL} on (a) ImageNet-II and (b) CIFAR-Hard (\cref{sec:exp}). The proposed approach (\cref{algo:puPL}) is compared against several popular PU Learning baselines. 
}
\label{fig:pupl_lp}
\end{figure}
Finally, we evaluate the effectiveness of our simple clustering based downstream PU classification strategy (\cref{algo:puPL}) in isolation by training a linear classifier over {\bf frozen} pretrained embeddings from (\cref{algo:puCL}). 
\\
\\
In{\bf~\cref{fig:pupl_lp}}, we compare \textsc{puPL} strategy against logistic regression baselines trained on CIFAR-Hard, ImageNet-II, with popular different objectives such as -- standard Cross-Entropy (CE), nnPU, variants like MixUp, vPU. In all cases, the encoder is frozen during the downstream classification phase, isolating the effect of the decoder on fixed representations. 
Additionally, in{\bf ~\cref{tab:puPL_LP}}, we ablate across different pre-training objectives and compare the performance when downstream classification is performed using nnPU vs puPL. We perform these experiment over all the six benchmark datasets: CIFAR-I,II, FMNIST-I,II, STL-I,II. 
Both sets of experiments are repeated across varying levels of PU supervision i.e. across different values of $\gamma$. These experiments highlight the effectiveness of our clustering-based strategy: \textsc{puPL} not only consistently improves downstream performance over nnPU-style decoding but also exhibits particular strength in low-label settings, where traditional PU losses struggle with high estimator variance. The simplicity and robustness of \textsc{puPL}, combined with its independence from class prior estimation, make it a compelling default choice for downstream classification over pretrained PU embeddings.
\\
\\
To understand the empirical success of \textsc{puPL}, we examine its behavior on a synthetic binary Gaussian Mixture model in{\bf ~\cref{fig:puPL_gmm}}. This setting allows for clear visualization of decision boundary deviation under varying PU strategies. In the fully supervised case (Panel a), CE$^\ast$ yields the Bayes-optimal linear separator. As the supervision ratio $\gamma$ decreases (Panels b–f), naive CE—treating all unlabeled samples as negatives—induces increasingly biased decision boundaries, reflecting the accumulation of systematic label noise. While \textsc{nnPU} mitigates this bias via importance reweighting, its performance deteriorates under extreme label sparsity due to high estimator variance and limited effective sample size.
\\
\\
In contrast, \textsc{puPL-CE} maintains a decision boundary closely aligned with the supervised optimum across all $\gamma$, including $\gamma = \tfrac{1}{50}$. This robustness arises from the fact that clustering in representation space acts as a form of structure-aware regularization. By exploiting the geometric separability of contrastive embeddings, \textsc{puPL} produces stable pseudo-labels that reflect latent class structure, even in regimes where standard risk-based estimators fail. These observations support the hypothesis that contrastive pretraining induces clusterable manifolds that can be effectively decoded without explicit reliance on class prior estimates.

\section{Conclusion}
\label{sec:conclusion}
In summary, this work developed a unified theoretical and algorithmic framework for contrastive representation learning under the PU learning paradigm, where supervision is partial and asymmetric. Classical contrastive objectives—either fully supervised or unsupervised—fail to balance the inherent bias-variance trade-off in this regime. Our proposed method, \textsc{puCL}, addresses this by leveraging labeled positives to reduce variance, while treating unlabeled data conservatively to avoid bias. We showed that \textsc{puCL} yields an unbiased estimator of the population contrastive loss with variance decreasing monotonically in the supervision ratio. We further introduced \textsc{puNCE}, a prior-aware extension that incorporates soft supervision using class prior information. This generalizes importance weighting in contrastive settings, improving sample efficiency and generalization in low-label regimes, while remaining robust to moderate prior misspecification. For downstream classification, we proposed \textsc{puPL}, a pseudo-labeling algorithm that exploits embedding geometry via PU-aware clustering. We established provable guarantees on its clustering quality and showed that it enables effective classification without requiring labeled negatives. Our analysis draws from augmentation concentration, gradient dynamics, and robust estimation theory. Together, \textsc{puCL}, \textsc{puNCE}, and \textsc{puPL} form a modular pipeline for PU contrastive learning with both theoretical guarantees and strong empirical performance. This framework lays a foundation for future study at the intersection of contrastive learning, weak supervision, and statistical robustness.

\clearpage
\bibliography{bibs/pu, 
bibs/contrastive_learning, 
bibs/misc_ml, 
bibs/label_noise,
bibs/mpe,
bibs/clustering,
bibs/pseudo_label
}

\clearpage
\appendix
\section{Notations and Abbreviations}
\label{sec:notations}
\bgroup
\def\arraystretch{1.5}
\begin{tabular}{p{1in}p{3.25in}}
$\displaystyle a$ & A scalar (integer or real)\\
$\displaystyle \va$ & A vector\\
$\displaystyle \mA$ & A matrix\\
$\displaystyle \ra$ & A scalar random variable\\
$\displaystyle \rva$ & A vector-valued random variable\\
$\displaystyle \sA$ & A set\\
$\displaystyle \{0, 1\}$ & The set containing 0 and 1 \\
$\displaystyle \{0, 1, \dots, n \}$ & The set of all integers between $0$ and $n$\\
$\displaystyle [a, b]$ & The real interval including $a$ and $b$\\
$\displaystyle \sA \backslash \sB$ & Set subtraction, i.e., the set containing the elements of $\sA$ that are not in $\sB$\\
$\displaystyle \erva_i$ & Element $i$ of the random vector $\rva$ \\
$\displaystyle f: \sA \rightarrow \sB$ & The function $f$ with domain $\sA$ and range $\sB$\\
$\displaystyle f \circ g $ & Composition of the functions $f$ and $g$ \\
  $\displaystyle f(\vx ; \vtheta) $ & A function of $\vx$ parametrized by $\vtheta$.
  (Sometimes we write $f(\vx)$ and omit the argument $\vtheta$ to lighten notation) \\
$\displaystyle || \vx ||_p $ & $\normlp$ norm of $\vx$ \\
$\displaystyle \1(condition)$ & is 1 if the condition is true, 0 otherwise\\
$\displaystyle \textsc{PU}$ & Positive Unlabeled\\
$\displaystyle p(\ra)$ & probability measure over a continuous variable, or over
a variable whose type has not been specified.\\
$\displaystyle \gamma$ & $\frac{n_P}{n_U}$\\
$\displaystyle \textsc{ssCL}$ & Self Supervised Contrastive Learning\\
$\displaystyle \textsc{sCL-PU}$ & Naive PU adaptation of Supervised Contrastive Learning\\
$\displaystyle \textsc{puCL}$ & Positive Unlabeled Contrastive Learning\\
$\displaystyle \textsc{puPL}$ & Positive Unlabeled Pseudo Labeling\\
\end{tabular}
\egroup

\section{Additional Experimental Details}
\label{sec:add_exp}
\subsection{Datasets}
Consistent with recent literature on PU Learning~\citep{li2022your, chen2020variational} we conduct our experiments on six benchmark datasets: STL-I, STL-II, CIFAR-I, CIFAR-II, FMNIST-I, and FMNIST-II, obtained via modifying STL-10~\citep{coates2011analysis}, CIFAR-10~\citep{krizhevsky2009learning}, and Fashion MNIST~\citep{xiao2017fashion}, respectively. The specific definitions of labels (“positive” vs “negative”) are as follows: 
\begin{itemize}
  \item {\bf FMNIST-I}: The labels "positive" correspond to the classes "1, 4, 7", while the labels "negative" correspond to the classes "0, 2, 3, 5, 6, 8, 9".
  
  \item {\bf FMNIST-II}: The labels "positive" correspond to the classes "0, 2, 3, 5, 6, 8, 9", while the labels "negative" correspond to the classes "1, 4, 7".

\item {\bf CIFAR-0}: "positive" dog vs "negative" cat. 
    
  \item {\bf CIFAR-I}: The labels "positive" correspond to the classes "0, 1, 8, 9", while the labels "negative" correspond to the classes "2, 3, 4, 5, 6, 7".
  
  \item {\bf CIFAR-II}: The labels "positive" correspond to the classes "2, 3, 4, 5, 6, 7", while the labels "negative" correspond to the classes "0, 1, 8, 9".

    \item {\bf CIFAR-III} : images of
vehicles (i.e., “airplanes,” “automobiles,” “ships,” and
“trucks”) as the positive class and images of
animals (i.e., “birds,” “cats,” “deer,” “dogs,” “frogs,” and
“horses”) as the negative class.

  \item {\bf CIFAR-hard} : "positive" airplane, cat vs. "negative" bird, dog.

  \item {\bf CIFAR-medium}: "positive" airplane, cat, dog vs. "negative" bird
  
  \item {\bf CIFAR-easy} : "positive" airplane, bird vs. "negative" cat, dog. 
  
  \item {\bf STL-I}: The labels "positive" correspond to the classes "0, 2, 3, 8, 9", while the labels "negative" correspond to the classes "1, 4, 5, 6, 7".
  
  \item {\bf STL-II}: The labels "positive" correspond to the classes "1, 4, 5, 6, 7", while the labels "negative" correspond to the classes "0, 2, 3, 8, 9".

  \item {\bf STL-III}: vehicles (i.e.,
“airplanes,” “cars,” “ships,” and “trucks”) as the positive
class, and the animals (i.e., “birds,” “cats,” “deer,”
“dogs,” “horses,” and “monkeys”) as the negative class. 

  \item {\bf ImageNet-I}: a subset of dog (P) vs non-dog (N) images sampled from ImageNet-1k~\citep{hua2021feature, robustness}; 
  
  \item {\bf ImageNet-II}: Imagewoof (P) vs ImageNette (N) -- two subsets of ImageNet-1k~\citep{fastai_imagenette}; 

  \item {\bf SVHN-I} The even numbers (i.e., “0,”
“2,” “4,” “6,” and “8”) are regarded as positive class and
the odd numbers (i.e., “1,” “3,” “5,” “7,” and “9”) are
regarded as negative class from the SVHN -- a collection of colored images of
street view house numbers.

\item {\bf Alzheimer} dataset contains the MRI images for
identifying the Alzheimer’s Disease. The MRI images
of demented patients are recognized as positive class
and the MRI images of healthy people are recognized
as negative class.
  
\end{itemize}

\subsection{Baseline Algorithms}
Next, we describe the PU baselines
\begin{itemize}
  \item {\bf Unbiased PU learning (\textsc{uPU})}~\citep{du2014analysis}: A foundational method that estimates the classification risk in an unbiased manner using positive and unlabeled data. It incorporates cost-sensitivity but may lead to negative risk values in practice. 
  
  \item {\bf Non-negative PU learning (\textsc{nnPU})}~\citep{kiryo2017positive}: An extension of uPU that prevents overfitting by clipping the negative part of the empirical risk to zero, ensuring non-negativity. It is cost-sensitive and widely adopted in practice. Suggested settings: $\beta = 0$ and $\gamma = 1.0$.
  
  \item {\bf \textsc{nnPU w Mixup}}~\cite{zhang2017mixup} : This cost-sensitive method combines the nnPU approach with the mixup technique. It performs separate mixing of positive instances and unlabeled ones.
  
  \item {\bf \textsc{Self-PU}}~\cite{chen2020self}: Incorporates a self-supervision mechanism with curriculum learning. Confident samples are iteratively added to the labeled set based on a self-paced thresholding scheme. Suggested settings: $\alpha = 10.0$, $\beta = 0.3$, $\gamma = \frac{1}{16}$, Pace1 = 0.2, and Pace2 = 0.3.
  
  \item {\bf Predictive Adversarial Networks (\textsc{pan}) }~\citep{hu2021predictive}: This method is based on GANs and specifically designed for PU learning. Suggested settings: $\lambda = 1e-4$.
  
  \item {\bf Variational PU learning (\textsc{vPU})} ~\citep{chen2020variational}: This approach is based on the variational principle and is tailored for PU learning. The public code from net.9 was used for implementation. Suggested settings: $\alpha = 0.3$, $\beta \in \{1e-4, 3e-4, 1e-3, \ldots, 1, 3\}$.
  
  \item {\bf \textsc{MixPUL}}~\citep{wei2020mixpul}: This method combines consistency regularization with the mixup technique for PU learning. The implementation utilizes the public code from net.10. Suggested settings: $\alpha = 1.0$, $\beta = 1.0$, $\eta = 1.0$.
  
  \item {\bf Positive-Unlabeled Learning with effective Negative sample Selector (\textsc{PULNS})}~\citep{luo2021pulns}: This approach incorporates reinforcement learning for sample selection. We implemented a custom Python code with a 3-layer MLP selector, as suggested by the paper. Suggested settings: $\alpha = 1.0$ and $\beta \in \{0.4, 0.6, 0.8, 1.0\}$.

  \item {\bf P$^3$\textsc{mix-c/e}}~\citep{li2022your}: Denotes the heuristic mixup based approach. 

  \item {\bf Reweighted PU (\textsc{RP})}~\citep{northcutt2017learning}
ranks the training data by confidence and selects
the most confident examples as positive or negative.

\item {\bf PU learning with Sample Bias (\textsc{PUSB})}~\citep{kato2018learning}
proposes a threshold estimation algorithm to
deal with the selection bias during the labeling process.

\item {\bf PU learning with biased Negative (\textsc{PUbN})}~\citep{hsieh2019classification}
first pretrains a model with nnPU algorithm
to classify some reliable positive data, negative data, and
unlabeled data, and then minimizes a risk approximated
by the above three partitions.

\item{\bf Arbitrary PU (\textsc{aPU})}~\citep{hammoudeh2020learning}
deals with the arbitrary positive shift between
source and target distributions.
\item {\bf Imbalanced PU (\textsc{ImbPU})}~\citep{su2021positive} re-designs the nnPU estimator to enable the
learning from imbalanced data.
\item {\bf \textsc{PiCO}}~\citep{wang2021pico} introduces a prototypical label disambiguation algorithm for addressing the PLL problem. 
\end{itemize}

\section{Gradient Analysis}
\label{sec:proof.gradient}
To gain deeper understanding into the behavior of the training dynamics we 
derive the gradient expressions for \textsc{ssCL} and \textsc{puCL}
\subsection{Gradient of \textsc{ssCL}: }
Recall that, \textsc{ssCL} takes the following form for any random sample from the multi-viewed batch indexed by $i \in \sI$
\begin{equation}
    \begin{aligned}
    \ell_i 
        & = - \log \frac{\exp(\vz_i \boldsymbol \cdot \vz_{a(i)})}{Z(\vz_i)}\;;\;\forall i\in \sI \\
        &= - \frac{\vz_i^T\vz_{a(i)}}{\tau} + \log Z(\vz_i)
    \end{aligned}  
    \label{eq:nt_xent_simplified}
\end{equation}
Furher, recall that the partition function $Z(\vz_i)$ is defined as : 
\[
Z(\vz_i) = \sum_{j \in \sI} \1(j \neq i)\exp(\vz_i\boldsymbol \cdot \vz_j)
\]
Note that, $\vz_i = g_{\vw}(\vx_i)$ where we have consumed both encoder and projection layer into $\vw$, and thus by chain rule we have, 
\begin{equation}
    \frac{\partial \ell_i}{\partial \vw} = \frac{\partial\ell_i}{\partial\vz_i} \cdot \frac{\partial\vz_i}{\partial \vw}
    \label{eq:chain_rule}
\end{equation}
Since, the second term depends on the encoder and fixed across the losses, the first term is sufficient to compare the gradients resulting from different losses. Thus, taking the differential of ~\eqref{eq:nt_xent_simplified} w.r.t representation $\vz_i$ we get: 
\begin{equation}
    \begin{aligned}
        \frac{\partial\ell_i}{\partial\vz_i} 
        & =  - \frac{1}{\tau}\left[\vz_{a(i)} - \frac{\sum_{j \in \sI \setminus \{i\}} \vz_j \exp(\vz_i\boldsymbol \cdot \vz_j)}{Z(\vz_i)} \right]\\
        & = - \frac{1}{\tau}\left[\vz_{a(i)} - \frac{\vz_{a(i)} \exp(\vz_i\boldsymbol \cdot \vz_{a(i)}) + \sum_{j \in \sI \setminus \{i, a(i)\}} \vz_j \exp(\vz_i\boldsymbol \cdot \vz_j)}{Z(\vz_i)} \right]\\
        & = - \frac{1}{\tau}\left[ \vz_{a(i)} \left(1 - \frac{\exp(\vz_i\boldsymbol \cdot \vz_{a(i)})}{Z(\vz_i)}\right) - \sum_{j \in \sI \setminus \{i, a(i)\}}\vz_j\frac{ \exp(\vz_i\boldsymbol \cdot \vz_j)}{Z(\vz_i)}\right]\\
        & = - \frac{1}{\tau}\left[ \vz_{a(i)} \left(1 - \frac{\exp(\vz_i\boldsymbol \cdot \vz_{a(i)})}{Z(\vz_i)}\right) - \sum_{j \in \sI \setminus \{i, a(i)\}}\vz_j\frac{ \exp(\vz_i\boldsymbol \cdot \vz_j)}{Z(\vz_i)}\right]\\
        & = - \frac{1}{\tau}\left[ \vz_{a(i)} \left( 1 - P_{i,a(i)} \right) - \sum_{j \in \sI \setminus \{i, a(i)\}}\vz_j P_{i,j}\right]
    \end{aligned}
\end{equation}
where, the functions $P_{i,j}$ are defined as: 
\begin{equation}
    P_{i,j} = \frac{\exp(\vz_i\boldsymbol \cdot \vz_j)}{Z(\vz_i)}
\end{equation}

\subsection{Gradient of \textsc{puCL}.}
Recall that, given a randomly sampled mini-batch $\gD$, \textsc{puCL} takes the following form for any sample $i \in \sI$ where $\sI$ is the corresponding multi-viewed batch. Let, $\sP(i) = \sP \setminus i$ i.e. all the other positive labeled examples in the batch w/o the anchor.
\begin{equation}
    \begin{aligned}
    \ell_i
    & = - \frac{1}{|\sP(i)|}\sum_{q\in \sP(i)}\log \frac{\exp{(\vz_i} \boldsymbol \cdot \vz_q)}{Z(\vz_i)}\;;\;\forall i\in \sI \\
    &= - \frac{1}{|\sP(i)|}\sum_{q\in \sP(i)} \left [ \frac{\vz_i^T\vz_q}{\tau} - \log Z(\vz_i)) \right]
    \end{aligned}  
    \label{eq:supcon_simplified}
\end{equation}
where $Z(\vz_i)$ is defined as before. Then, we can compute the gradient w.r.t representation $\vz_i$ as: 

\begin{equation}
    \begin{aligned}
        \frac{\partial\ell_i}{\partial\vz_i}
        & = - \frac{1}{|\sP(i)|}\sum_{q\in \sP(i)}\left[ \frac{\vz_q}{\tau} - \frac{\partial Z(\vz_i)}{Z(\vz_i)}\right]\\
        & = - \frac{1}{\tau|\sP(i)|}\sum_{q\in \sP(i)}\left[ \vz_q  - \frac{\sum_{j \in \sI \setminus \{i\}} \vz_j \exp(\vz_i\boldsymbol \cdot \vz_j)}{Z(\vz_i)}\right]\\
        & = - \frac{1}{\tau|\sP(i)|}\sum_{q\in \sP(i)}\left[ \vz_q  - \sum_{q'\in \sP(i)}\vz_{q'} P_{i,q'} - \sum_{j \in \sU(i)} \vz_j P_{i,j} \right]\\
        & = - \frac{1}{\tau|\sP(i)|}\left[ \sum_{q\in \sP(i)} \vz_q - \sum_{q\in \sP(i)} \sum_{q'\in \sP(i)}\vz_{q'} P_{i,q'} - \sum_{q\in \sP(i)} \sum_{j \in \sU(i)} \vz_j P_{i,j} \right]\\
        & = - \frac{1}{\tau|\sP(i)|} \left[ \sum_{q\in \sP(i)} \vz_q - \sum_{q'\in \sP(i)}|\sP(i)|\vz_{q'} P_{i,q'}  - \sum_{j \in \sU(i)}|\sP(i)|\vz_j P_{i,j} \right]\\
        & = - \frac{1}{\tau} \left[ \frac{1}{|\sP(i)|}\sum_{q\in \sP(i)}\vz_q - \sum_{q\in \sP(i)}\vz_{q} P_{i,q} - \sum_{j \in \sU(i)}\vz_j P_{i,j}\right]\\
        & = - \frac{1}{\tau} \left[ \sum_{q\in \sP(i)} \vz_q \left( \frac{1}{|\sP(i)|} - P_{i,q} \right) - \sum_{j \in \sU(i)}\vz_j P_{i,j} \right]
    \end{aligned}
\end{equation}
where we have defined $\sU(i) = \sI\setminus\{i, \sP(i)\}$ i.e. $\sU(i)$ is the set of all samples in the batch that are unlabeled. 
\\
\\
In case of fully supervised setting we would similarly get: 
\begin{equation}
    \begin{aligned}
        \frac{\partial\ell_i}{\partial\vz_i}
        = - \frac{1}{\tau} \left[ \sum_{q\in \sP^\ast(i)} \vz_q \left( \frac{1}{|\sP^\ast(i)|} - P_{i,q} \right) - \sum_{j \in \sN(i)}\vz_j P_{i,j} \right]
    \end{aligned}
\end{equation}
Comparing the three gradient expressions, it is clear that \textsc{puCL} enjoys lower gradient bias compared to \textsc{ssCL} with respect to fully supervised counterpart.

\section{Complete Proofs.}
For theoretical analysis, we define some additional notation over~\cref{sec:problem_setup}. 
\\
\\
$\gX_{\textsc{PU}}$ is generated from the underlying supervised dataset $\gX_{\textsc{PN}} = \gX_\textsc{P} \cup \gX_\textsc{N}$ i.e. labeled positives $\gX_{\textsc{P}_\textsc{L}}$ is a subset of $n_{\textsc{P}_\textsc{L}}$ elements chosen uniformly at random from all subsets of $\gX_\textsc{P}$ of size $n_\textsc{L}$, i.e.  
\begin{equation}
    \gX_{\textsc{P}_\textsc{L}} \subseteq \gX_\textsc{P} = \bigg\{\vx_i \in \sR^d \sim p(\rx | \ry=1)\bigg\}_{i=1}^{n_\textsc{P}}.
\end{equation} 
Further, denote the set positive and negative examples that are unlabeled as $\gX_{\textsc{P}_\textsc{U}}$ and $\gX_{\textsc{N}_\textsc{U}}$.
\begin{flalign}
    \gX_\textsc{PU} = \gX_{\textsc{P}_\textsc{L}} \cup \gX_{\textsc{P}_\textsc{U}} \cup \gX_{\textsc{N}_\textsc{U}} \\
    \gX_\textsc{P} = \gX_{\textsc{P}_\textsc{L}} \cup \gX_{\textsc{P}_\textsc{U}} \\
    \gX_\textsc{U} = \gX_{\textsc{P}_\textsc{U}} \cup \gX_{\textsc{N}_\textsc{U}}
    \label{eq:decomposed_pu_dataset}
\end{flalign}

\subsection{Proof of \cref{th:scl_pu_bias}.}
\label{sec:proof.scl_pu_bias}
We restate~\cref{th:scl_pu_bias} for convenience - 
$\gL_{\textsc{sCL-pu}}$~\eqref{eq:scl_pu} is a biased estimator of $\gL_{\textsc{CL}}^\ast$~\eqref{eq:asymptotic_scl} characterized as follows:
\begin{flalign*}
    \mathop{\E}_{\gX_{\textsc{PU}}} 
    \bigg[ 
    \gL_{\textsc{sCL-PU}}
    \bigg]
    - \gL_{\textsc{CL}}^\ast
     = 2 \kappa
    \bigg(\rho_{intra} - \rho_{inter}\bigg).
\end{flalign*}
where, $\rho_{intra}=\mathop{\E}_{\substack{\vx_i,\vx_j \sim p(\vx | y_i=y_j)}}\big( \vz_i \boldsymbol \cdot \vz_j \big)$ 
captures the concentration of embeddings of samples from same latent class marginals and 
$\rho_{inter} = \mathop{\E}_{\substack{\vx_i,\vx_j \sim p(\vx | \ry_i \neq \ry_j)}}\big( \vz_i \boldsymbol \cdot \vz_j \big)$ 
captures the expected proximity between embeddings of dissimilar samples. $\kappa = \frac{\pi (1 - \pi)}{ 1 + \gamma}$ is PU dataset specific constant where, $\gamma = \frac{n_{\textsc{P}}}{n_{\textsc{U}}}$ and $\pi=p(y=1|x)$. 
\\
\\
\begin{proof}
Now, we can establish the result by carefully analyzing the bias of $\gL_\textsc{sCL-PU}$~\eqref{eq:scl_pu} in estimating the ideal contrastive loss~\eqref{eq:asymptotic_scl} over each of these subsets. 
\\
\\
For the {\bf labeled positive subset} $\gX_{\textsc{P}_\textsc{L}}$ the bias can be computed as: 
\begin{flalign}
    \gB_{\gL_\textsc{sCL-PU}}(\vx_i \in \gX_{\textsc{P}_\textsc{L}}) 
    &= -\E_{\vx_i \in \gX_{\textsc{P}_\textsc{L}}}\Bigg[ \frac{1}{n_{\textsc{P}_\textsc{L}}}\sum\limits_{\vx_j \in \gX_{\textsc{P}_\textsc{L}}}\vz_i \boldsymbol \cdot \vz_j \Bigg] 
    + \E_{\vx_i, \vx_j \sim p(\rx|\ry=1)} \Bigg(\vz_i \boldsymbol \cdot \vz_j\Bigg) = 0.
\end{flalign}
For the {\bf unlabeled positive subset }$\gX_{\textsc{P}_\textsc{U}}\subseteq \gX_\textsc{PU}$ the bias can be computed as: 
\begin{flalign*}
    - \gB_{\gL_\textsc{sCL-PU}}(\vx_i \in \gX_{\textsc{P}_\textsc{U}}) 
    &= \E_{\vx_i \in \gX_{\textsc{P}_\textsc{U}}}\Bigg[ \frac{1}{n_\textsc{U}}\sum\limits_{\vx_j \in \gX_\textsc{U}}\vz_i \boldsymbol \cdot \vz_j \Bigg] 
    - \E_{\vx_i, \vx_j \sim p(\rx|\ry=1)} \Bigg(\vz_i \boldsymbol \cdot \vz_j\Bigg)\\
    &=\E_{\vx_i \in \gX_{\textsc{P}_\textsc{U}}}
    \Bigg[ 
    \pi \E_{\vx_j \in \gX_{\textsc{P}_\textsc{U}}} \Bigg(\vz_i \boldsymbol \cdot \vz_j \Bigg) 
    + (1 - \pi) \E_{\vx_j \in \gX_{\textsc{N}_\textsc{U}}}\Bigg(\vz_i \boldsymbol \cdot \vz_j\Bigg)
    \Bigg] 
    - \rho_\textsc{P}\\
    &=\pi \rho_\textsc{P} + (1 - \pi) \E_{\vx_i \in \gX_{\textsc{P}_\textsc{U}}}\Bigg[ \E_{\vx_j \in \gX_{\textsc{N}_\textsc{U}}}\Bigg(\vz_i \boldsymbol \cdot \vz_j\Bigg) \Bigg] - \rho_\textsc{P}\\
    &=(1 - \pi)\E_{\vx_i \sim p(\rx|y=1)}\Bigg[ \E_{\vx_j \sim p(\rx|y=0)}\Bigg(\vz_i \boldsymbol \cdot \vz_j\Bigg) \Bigg] - (1 - \pi)\rho_\textsc{P}\\
    &=(1 - \pi)\E_{\vx_i,\vx_j \sim p(\rx|y_i \neq y_j)}\Bigg(\vz_i \boldsymbol \cdot \vz_j\Bigg) - (1 - \pi)\rho_\textsc{P}\\
    &=(1 - \pi)\rho_{inter} - (1 - \pi)\rho_\textsc{P}\\
\end{flalign*}
where, we denote $\rho_\textsc{P} = \E_{\vx_i, \vx_j \sim p(\rx|\ry=1)} \Bigg(\vz_i \boldsymbol \cdot \vz_j\Bigg)$ and  $\rho_{inter} = \E_{\vx_i,\vx_j \sim p(\rx|y_i \neq y_j)}\Bigg(\vz_i \boldsymbol \cdot \vz_j\Bigg)$.
\\
\\
Finally, for the {\bf negative unlabeled subset}:
\begin{flalign*}
    -\gB_{\gL_\textsc{sCL-PU}}(\vx_i \in \gX_{\textsc{N}_\textsc{U}}) 
    &= \E_{\vx_i \in \gX_{\textsc{N}_\textsc{U}}}\Bigg[ \frac{1}{n_\textsc{U}}\sum\limits_{\vx_j \in \gX_\textsc{U}}\vz_i \boldsymbol \cdot \vz_j \Bigg] 
    - \E_{\vx_i, \vx_j \sim p(\rx|\ry=0)} \Bigg(\vz_i \boldsymbol \cdot \vz_j\Bigg)\\
    &=\E_{\vx_i \in \gX_{\textsc{N}_\textsc{U}}}
    \Bigg[ 
    \pi \E_{\vx_j \in \gX_{\textsc{P}_\textsc{U}}} \Bigg(\vz_i \boldsymbol \cdot \vz_j \Bigg) 
    + (1 - \pi) \E_{\vx_j \in \gX_{\textsc{N}_\textsc{U}}}\Bigg(\vz_i \boldsymbol \cdot \vz_j\Bigg)
    \Bigg] 
    - \rho_\textsc{N}\\
    &=\pi\E_{\vx_i, \vx_j \sim p(\rx|\ry_i\neq \ry_j)}\Bigg(\vz_i \boldsymbol \cdot \vz_j\Bigg) + (1 - \pi)\rho_\textsc{N}   - \rho_\textsc{N}\\
    &=\pi\rho_{inter} - \pi\rho_\textsc{N}
\end{flalign*}
where, $\rho_\textsc{N} = \E_{\vx_i, \vx_j \sim p(\rx|\ry=1)} \Bigg(\vz_i \boldsymbol \cdot \vz_j\Bigg)$.
\\
\\
Now, using the fact that the unlabeled examples are sampled uniformly at random from the mixture distribution with positive mixture weight $\pi$ we can compute the total bias as follows: 
\begin{flalign*}
    \gB_{\gL_\textsc{sCL-PU}}(\vx_i \in \gX_\textsc{PU}) 
    &= \frac{\pi}{1 + \gamma}\gB_{\gL_\textsc{sCL-PU}}(\vx_i \in \gX_{\textsc{P}_\textsc{U}}) 
    + 
    \frac{1 - \pi}{1 + \gamma}\gB_{\gL_\textsc{sCL-PU}}(\vx_i \in \gX_{\textsc{N}_\textsc{U}})
\end{flalign*}
where, $\gamma = {|\gX_{\textsc{P}_\textsc{L}}|}/{|\gX_\textsc{U}|}$. 
plugging in the bias of the subsets:
\begin{flalign*}
    \gB_{\gL_\textsc{sCL-PU}}(\vx_i \in \gX_\textsc{PU}) 
    &=-\frac{\pi}{1 + \gamma}\Bigg[(1 - \pi)\rho_{inter} - (1 - \pi)\rho_\textsc{P} \Bigg] - \frac{1 - \pi}{1 + \gamma}\Bigg[\pi\rho_{inter} - \pi\rho_\textsc{N}\Bigg]\\
    &=\frac{1}{1 + \gamma}\Bigg[\pi(1 -\pi)\bigg(\rho_\textsc{P} + \rho_\textsc{N}\bigg) - 2\pi(1 -\pi)\rho_{inter}\Bigg]\\
    &=\frac{2\pi(1 -\pi)}{1 + \gamma}\Bigg[\frac{1}{2}\bigg(\rho_\textsc{P} + \rho_\textsc{N}\bigg) - \rho_{inter}
    \Bigg]\\
    &=2\kappa\bigg( \rho_{intra} - \rho_{inter} \bigg).
\end{flalign*}
This completes the proof.
\end{proof}

\subsection{Proof of~\cref{lemma:pucl_unbiased}}
\label{sec:proof.pucl_unbiased}
We restate~\cref{lemma:pucl_unbiased} for convenience - 
\\
\\
If~\cref{assumption:aug_independence} holds, then
$\gL_{\textsc{ssCL}}$~\eqref{eq:sscl} and $\gL_{\textsc{puCL}}$~\eqref{eq:puCL} are unbiased estimators of  $\gL_{\textsc{CL}}^\ast$~\eqref{eq:asymptotic_scl}. Additionally, it holds that:
\begin{flalign}
    \Delta_\sigma(\gamma) \geq 0 \;\;\forall \gamma \geq 0\\
    \Delta_\sigma(\gamma_1) \geq \Delta_\sigma(\gamma_2) \quad \forall \gamma_1 \geq \gamma_2 \geq 0
\end{flalign}
where, $\Delta_\sigma(\gamma)= 
\mathrm{Var}(\gL_{\textsc{ssCL}}) - \mathrm{Var}(\gL_{\textsc{puCL}})$ and $\gamma = n_\textsc{P}/n_\textsc{U}$. 
\\
\\
\begin{proof}
We first prove that both $\gL_{\textsc{ssCL}}$and $\gL_{\textsc{puCL}}$ are unbiased estimators of  $\gL_{\textsc{CL}}^\ast$. 
\\
\\
{\bf For the labeled positive} subset $\gX_{\textsc{P}_\textsc{L}}$ the bias can be computed as: 
\begin{flalign*}
    \gB_{\gL_\textsc{puCL}}(\vx_i \in \gX_{\textsc{P}_\textsc{L}}) 
    &= \E_{\vx_i \in \gX_{\textsc{P}_\textsc{L}}}\Bigg[ \frac{1}{n_{\textsc{P}_\textsc{L}}}\sum\limits_{\vx_j \in \gX_{\textsc{P}_\textsc{L}}}\vz_i \boldsymbol \cdot \vz_j \Bigg] 
    - \E_{\vx_i, \vx_j \sim p(\rx|\ry=1)} \Bigg[\vz_i \boldsymbol \cdot \vz_j\Bigg] = 0
\end{flalign*}
Here we have used the fact that labeled positives are drawn i.i.d from the positive marginal.
\\
\\
{\bf For the unlabeled samples }:
\begin{flalign*}
    \gB_{\gL_\textsc{puCL}}(\vx_i \in \gX_\textsc{U}) 
    &= \E_{\vx_i \in \gX_\textsc{U}}\Bigg[ \vz_i \boldsymbol \cdot \vz_{a(i)} \Bigg] 
    - \E_{\vx_i, \vx_j \sim p(\rx|\ry_i=\ry_j)} \Bigg[\vz_i \boldsymbol \cdot \vz_j\Bigg]\\
    &=\E_{\vx_i, \vx_j \sim p(\rx|\ry_i=\ry_j)} \Bigg[\vz_i \boldsymbol \cdot \vz_j\Bigg]
    - \E_{\vx_i, \vx_j \sim p(\rx|\ry_i=\ry_j)} \Bigg[\vz_i \boldsymbol \cdot \vz_j\Bigg]=0
\end{flalign*}
Thus $\gL_{\textsc{puCL}}$ is an unbiased estimator of $\gL_{\textsc{CL}}^\ast$. 
Similarly, $\gL_\textsc{ssCL}$ is also an unbiased estimator.
\\
\\
Next we can do a similar {\bf decomposition of the variances} for both the objectives. Then the difference of variance under the PU dataset - 
\begin{flalign*}
    \Delta_\sigma(\gX_\textsc{PU})
    &= \mathrm{Var}_{\gL_\textsc{ssCL}}(\gX_\textsc{PU}) - \mathrm{Var}_{\gL_\textsc{puCL}}(\gX_\textsc{PU})\\
    &=\Delta_\sigma(\gX_{\textsc{P}_\textsc{L}}) + \Delta_\sigma(\gX_\textsc{U})\\
    &=\Delta_\sigma(\gX_{\textsc{P}_\textsc{L}})\\
    &=\Bigg(1 - \frac{1}{n_{\textsc{P}_\textsc{L}}}\Bigg) \mathrm{Var}\Bigg( \vz_i \boldsymbol \cdot \vz_j : \vx_i, \vx_j \in \gX_{\textsc{P}_\textsc{L}}\Bigg)\\
    &=\Bigg(1 - \frac{1}{\gamma|\gX_\textsc{U}|}\Bigg)\mathrm{Var}\Bigg( \vz_i \boldsymbol \cdot \vz_j : \vx_i, \vx_j \in \gX_{\textsc{P}_\textsc{L}}\Bigg)\\
\end{flalign*}
Clearly, since variance is non-negative we have $\forall \gamma > 0 : \Delta_\sigma(\gX_\textsc{PU}) \geq 0$
\\
\\
Now consider two settings where we have different amounts of labeled positives defined by ratios $\gamma_1$ and $\gamma_2$ and denote the two resulting datasets $\gX_\textsc{PU}^{\gamma_1}$ and $\gX_\textsc{PU}^{\gamma_2}$ then 
\begin{flalign*}
    \Delta_\sigma(\gX_\textsc{PU}^{\gamma_1}) - \Delta_\sigma(\gX_\textsc{PU}^{\gamma_2})
    &=\Delta_\sigma(\gX_{\textsc{P}_\textsc{L}}^{\gamma_1}) - \Delta_\sigma(\gX_{\textsc{P}_\textsc{L}}^{\gamma_2})\\
    &=\frac{1}{|\gX_\textsc{U}|}\Bigg(\frac{1}{\gamma_2} - \frac{1}{\gamma_1}\Bigg)\mathrm{Var}\Bigg( \vz_i \boldsymbol \cdot \vz_j : \vx_i, \vx_j \in \gX_{\textsc{P}_\textsc{L}}\Bigg)\\
    & \geq 0 
\end{flalign*}
The last inequality holds since $\gamma_1 \geq \gamma_2$. This concludes the proof.
\end{proof}

\subsection{Proof of \cref{th:puPL}.}
\label{sec:proof.puPL}
Central to the analysis is the following two lemmas:
\begin{lemma}[\bf Positive Centroid Estimation]
    Suppose,
    $\gZ_{\textsc{P}_\textsc{L}}$ is a subset of $n_\textsc{L}$ elements chosen uniformly at random from all subsets of $\gZ_\textsc{P}$ of size $n_\textsc{L}$ : $\gZ_{\textsc{P}_\textsc{L}} \subset \gZ_\textsc{P} = \{\vz_i = g_{\mB}(\vx_i) \in \sR^k : \vx_i \in \sR^d \sim p(\rx | \ry=1)\}_{i=1}^{n_\textsc{P}}$ implying that the labeled positives are generated according to~\eqref{eq:decomposed_pu_dataset}. Let, $\Mu$ denote the centroid of $\gZ_{\textsc{P}_\textsc{L}}$ i.e. $\Mu_\textsc{P} = \frac{1}{n_{\textsc{P}_\textsc{L}}}\sum_{\vz_i \in \gZ_{\textsc{P}_\textsc{L}}}\vz_i$\; and $\Mu^\ast$ denote the optimal centroid of $\gZ_\textsc{P}$ i.e. $\phi^\ast(\gZ_\textsc{P}, \Mu^\ast) = \sum_{\vz_i \in \gZ_{\textsc{PU}}}\|\vz_i - \Mu^\ast\|^2$ then we can establish the following result: 
    \[
        \E\Bigg[ \phi(\gZ_\textsc{P}, \Mu_\textsc{P})\Bigg] = \Bigg(1 + \frac{n_\textsc{P} - n_{\textsc{P}_\textsc{L}}}{n_{\textsc{P}_\textsc{L}}(n_\textsc{P} - 1)}\Bigg)
        \phi^\ast(\gZ_\textsc{P}, \Mu^\ast)
    \]
    \label{lemma:pos_centroid}
\end{lemma}
\begin{proof}
\begin{flalign*}
    \E\Bigg[ \phi(\gZ_\textsc{P}, \Mu_\textsc{P})\Bigg] 
    &=\E\Bigg[ \sum_{\vz_i \in \gZ_\textsc{P}} \|\vz_i - \Mu_\textsc{P}\|^2 \Bigg]  \\
    &=\E\Bigg[ \sum_{\vz_i \in \gZ_\textsc{P}} \|\vz_i - \Mu^\ast\|^2 + n_\textsc{P}\|\Mu_\textsc{P} - \Mu^\ast\|^2 \Bigg]\\
    &=\phi^\ast(\gZ_\textsc{P}, \Mu^\ast) + n_\textsc{P}\E \Bigg[\|\Mu_\textsc{P} - \Mu^\ast\|^2\Bigg]\\
\end{flalign*}    
Now we can compute the expectation as:
\begin{flalign*}
    \E \Bigg[\|\Mu_\textsc{P} - \Mu^\ast\|^2\Bigg] 
    &=\E\Bigg[ \Mu_\textsc{P}^T\Mu_\textsc{P}\Bigg] + \Mu^{\ast T}\Mu^\ast - 2\Mu^{\ast T}\E\Bigg[ \frac{1}{n_{\textsc{P}_\textsc{L}}}\sum_{\vz_i \in \gZ_{\textsc{P}_\textsc{L}}}\vz_i \Bigg]\\
    &=\E\Bigg[ \Mu_\textsc{P}^T\Mu_\textsc{P}\Bigg] + \Mu^{\ast T}\Mu^\ast - 2\Mu^{\ast T}\frac{1}{n_{\textsc{P}_\textsc{L}}} \E\Bigg[ \sum_{\vz_i \in \gZ_{\textsc{P}_\textsc{L}}}\vz_i \Bigg]\\
    &=\E\Bigg[ \Mu_\textsc{P}^T\Mu_\textsc{P}\Bigg] + \Mu^{\ast T}\Mu^\ast - 2\Mu^{\ast T}\frac{1}{n_{\textsc{P}_\textsc{L}}} n_{\textsc{P}_\textsc{L}}\E_{\vz_i \in \gZ_\textsc{P}}\Bigg[\vz_i \Bigg]\\
    &=\E\Bigg[ \Mu_\textsc{P}^T\Mu_\textsc{P}\Bigg] - \Mu^{\ast T}\Mu^\ast\\
\end{flalign*}
We can compute the first expectation as:
\begin{flalign*}
    \E\Bigg[ \Mu_\textsc{P}^T\Mu_\textsc{P}\Bigg] 
    &= \frac{1}{n_{\textsc{P}_\textsc{L}}^2} \E\Bigg[ \Bigg(\sum_{\vz_i \in \gZ_{\textsc{P}_\textsc{L}}}\vz_i\Bigg)^T \Bigg(\sum_{\vz_i \in \gZ_{\textsc{P}_\textsc{L}}}\vz_i \Bigg)\Bigg] \\
    &=\frac{1}{n_{\textsc{P}_\textsc{L}}^2} \Bigg[
    p(i \neq j)\sum_{\vz_i, \vz_j \in \gZ_\textsc{P}, i\neq j}\vz_i^T\vz_j + p(i=j) \sum_{\vz_i \in \gZ_\textsc{P}}\vz_i^T\vz_i\Bigg]\\
    &=\frac{1}{n_{\textsc{P}_\textsc{L}}^2} \Bigg[\frac{\binom{n_\textsc{P}-2}{n_{\textsc{P}_\textsc{L}} -2}}{\binom{n_\textsc{P}}{n_{\textsc{P}_\textsc{L}}}}
    \sum_{\vz_i, \vz_j \in \gZ_\textsc{P}, i\neq j}\vz_i^T\vz_j 
    + \frac{\binom{n_\textsc{P}-1}{n_{\textsc{P}_\textsc{L}} -1}}{\binom{n_\textsc{P}}{n_{\textsc{P}_\textsc{L}}}} \sum_{\vz_i \in \gZ_\textsc{P}}\vz_i^T\vz_i\Bigg]\\
    &=\frac{1}{n_{\textsc{P}_\textsc{L}}^2} \Bigg[\frac{n_{\textsc{P}_\textsc{L}} (n_{\textsc{P}_\textsc{L}} -1)}{n_\textsc{P}(n_\textsc{P}-1)}\sum_{\vz_i, \vz_j \in \gZ_\textsc{P}, i\neq j}\vz_i^T\vz_j 
    + \frac{n_{\textsc{P}_\textsc{L}}}{n_\textsc{P}}\sum_{\vz_i \in \gZ_\textsc{P}}\vz_i^T\vz_i\Bigg]\\
\end{flalign*}
Plugging this back we get:
\begin{flalign*}
    \E \Bigg[\|\Mu_\textsc{P} - \Mu^\ast\|^2\Bigg] 
    &= \frac{1}{n_{\textsc{P}_\textsc{L}}^2} \Bigg[\frac{n_{\textsc{P}_\textsc{L}} (n_{\textsc{P}_\textsc{L}} -1)}{n_\textsc{P}(n_\textsc{P}-1)}\sum_{\vz_i, \vz_j \in \gZ_\textsc{P}, i\neq j}\vz_i^T\vz_j 
    + \frac{n_{\textsc{P}_\textsc{L}}}{n_\textsc{P}}\sum_{\vz_i \in \gZ_\textsc{P}}\vz_i^T\vz_i\Bigg] - \Mu^{\ast T}\Mu^\ast \\
    &= \frac{n_\textsc{P} - n_{\textsc{P}_\textsc{L}}}{n_{\textsc{P}_\textsc{L}}(n_\textsc{P} - 1)}\Bigg[ \frac{1}{n_\textsc{P}}\sum_{\vz_i \in \gZ_\textsc{P}}\vz_i^T\vz_i - \Mu^{\ast T}\Mu^\ast\Bigg]\\
    &= \Bigg(1 + \frac{n_\textsc{P} - n_{\textsc{P}_\textsc{L}}}{n_{\textsc{P}_\textsc{L}}(n_\textsc{P} - 1)}\Bigg)
        \phi^\ast(\gZ_\textsc{P}, \Mu^\ast)
\end{flalign*}
This concludes the proof.
\end{proof}

\begin{lemma}[\bf $k$-means++ Seeding]
    Given initial cluster center $\Mu_\textsc{P} = \frac{1}{n_{\textsc{P}_\textsc{L}}}\sum_{\vz_i \in \gZ_{\textsc{P}_\textsc{L}}}\vz_i$, if the second centroid $\Mu_\textsc{N}$ is chosen according to the distribution $D(\vz) = \frac{\phi(\{\vz\}, \{\Mu_\textsc{P}\})}{\sum_{\vz \in \gZ_\textsc{U}}\phi(\{\vz\}, \{\Mu_\textsc{P}\})} \; \forall \vz \in \gZ_\textsc{U}$, then:
    \[
    \E\Bigg[ \phi(\gZ_\textsc{PU}, \{\Mu_\textsc{P}, \Mu_\textsc{N}\})\Bigg] \leq 2 \phi(\gZ_{\textsc{P}_\textsc{L}}, \{\Mu_\textsc{P}\}) + 16\phi^\ast(\gZ_\textsc{U}, C^\ast)
    \]
\label{lemma:d_square}
\end{lemma}
\begin{proof}
    This result is a direct consequence of Lemma 3.3 from~\citep{arthur2007k} and specializing to our case where we only have 1 uncovered cluster i.e. $t=u=1$ and consequently the harmonic sum $H_t = 1$.
\end{proof}
\\
\\
Now, we are ready to prove~\cref{th:puPL}. 
We will closely follows the proof techniques from~\citep{arthur2007k} \emph{mutatis mutandis} to prove this theorem. 
\begin{proof}
Recall that we choose our first center from supervision i.e. $\Mu_\textsc{P} = \frac{1}{n_{\textsc{P}_\textsc{L}}}\sum_{\vz_i \in \gZ_{\textsc{P}_\textsc{L}}}\vz_i$ and then choose the next center from the unlabeled samples according to probability $D(\vz) = \frac{\phi(\{\vz\}, \{\Mu_\textsc{P}\})}{\sum_{\vz \in \gZ_\textsc{U}}\phi(\{\vz\}, \{\Mu_\textsc{P}\})} \; \forall \vz \in \gZ_\textsc{U}$. Then, from Lemma~\ref{lemma:d_square}:
\begin{flalign*}
    \E\Bigg[ \phi(\gZ_\textsc{PU}, \{\Mu_\textsc{P}, \Mu_\textsc{N}\})\Bigg] 
    &\leq 2 \phi(\gZ_{\textsc{P}_\textsc{L}}, \{\Mu_\textsc{P}\}) + 16\phi^\ast(\gZ_\textsc{U}, C^\ast)\\
    &=2 \phi(\gZ_{\textsc{P}_\textsc{L}}, \{\Mu_\textsc{P}\}) + 16 \Bigg( \phi^\ast(\gZ_\textsc{PU}, C^\ast) - \phi^\ast(\gZ_{\textsc{P}_\textsc{L}}, C^\ast)\Bigg)\\
    &=2 \Bigg( \phi(\gZ_{\textsc{P}_\textsc{L}}, \{\Mu_\textsc{P}\}) - 8 \phi^\ast(\gZ_{\textsc{P}_\textsc{L}}, \{\Mu_\textsc{P}^\ast\})\Bigg) + 16\phi^\ast(\gZ_\textsc{PU}, C^\ast)
\end{flalign*}
Now we use Lemma~\ref{lemma:pos_centroid} to bound the first term - 
\begin{flalign*}
    \E\Bigg[ \phi(\gZ_\textsc{PU}, \{\Mu_\textsc{P}, \Mu_\textsc{N}\})\Bigg] 
    &\leq 2 \Bigg[ \Bigg(1 + \frac{n_\textsc{P} - n_{\textsc{P}_\textsc{L}}}{n_{\textsc{P}_\textsc{L}}(n_\textsc{P} - 1)}\Bigg) - 8\Bigg]\phi^\ast(\gZ_{\textsc{P}_\textsc{L}}, \{\Mu_\textsc{P}^\ast\}) + 16\phi^\ast(\gZ_\textsc{PU}, C^\ast)\\
    & \leq 2 \Bigg[ \frac{n_\textsc{P} - n_{\textsc{P}_\textsc{L}}}{n_{\textsc{P}_\textsc{L}}(n_\textsc{P} - 1)} - 7\Bigg]\phi^\ast(\gZ_{\textsc{P}_\textsc{L}}, \{\Mu_\textsc{P}^\ast\}) + 16\phi^\ast(\gZ_\textsc{PU}, C^\ast)\\
    &\leq 16\phi^\ast(\gZ_\textsc{PU}, C^\ast)
\end{flalign*}
Note that this bound is much tighter in practice when a large amount of labeled examples are available i.e. for larger values of $n_{\textsc{P}_\textsc{L}}$. Additionally our guarantee holds only after the initial cluster assignments are found. Subsequent standard $k$-means iterations can only further decrease the potential. 
\\
\\
On the other hand for $k$-means++ strategy~\citep{arthur2007k} the guarantee is:
\begin{flalign}
    \E\Bigg[ \phi(\gZ_\textsc{PU}, C_{_{k-\text{means++}}})\Bigg] 
 & \leq \Bigg( 2 + \ln 2\Bigg)8\phi^\ast(\gZ_\textsc{PU}, C^\ast) \\
 & \approx 21.55 \phi^\ast(\gZ_\textsc{PU}, C^\ast) 
\end{flalign}

This concludes the proof.
\end{proof}

\subsection{Proof of~\cref{remark:kNN}}
We want to show that
The Nearest Neighbor Classifier $F_{g_\mB}(\cdot)$ can be formulated as a linear classifier:
\begin{flalign*}
F_{g_\mB}(\vx) = \argmin_{\bm{\mu} \in \{\bm{\mu}_\textsc{P}, \bm{\mu}_\textsc{N}\}} \|g_\mB(\vx) - \bm{\mu}\| = \argmax_{\bm{\mu} \in \{\bm{\mu}_\textsc{P}, \bm{\mu}_\textsc{N}\}} \bigg( \bm{\mu}^Tg_\mB(\vx) - \frac{1}{2}\|\bm{\mu}\|^2\bigg)
\end{flalign*}

\begin{proof}
    Consider the decision rule:  
    \begin{flalign*}
        &\|g_\mB(\vx) - \bm{\mu}_\textsc{P}\|^2 \leq \|g_\mB(\vx) - \bm{\mu}_\textsc{N}\|^2\\
        \implies& \bm{\mu}_\textsc{P}^T g_\mB(\vx) - \frac{1}{2}\|\bm{\mu}_\textsc{P}\|^2 \geq \bm{\mu}_\textsc{N}^T g_\mB(\vx) - \frac{1}{2}\|\bm{\mu}_\textsc{N}\|^2
     \end{flalign*}

Clearly, this is equivalent to a linear classifier :
\begin{flalign*}
    F_{g_\mB}(\vx) = \argmax_{\bm{\mu} \in \{\bm{\mu}_\textsc{P}, \bm{\mu}_\textsc{N}\}}\bigg( \bm{\mu}^Tg_\mB(\vx) - \frac{1}{2}\|\bm{\mu}\|^2\bigg)
    \end{flalign*}
\end{proof}

\subsection{Proof of \cref{th:generalization}}
\label{sec:proof.generalization}
Let $\gT$ be a $(\delta, \sigma)$ augmentation (\cref{def:aug_delta_sigma}), and $g_\mB(\cdot)$ be $L$ Lipschitz. Suppose, the estimated class centroids by~\cref{algo:puPL} satisfy: 
\begin{flalign*}
    \hat{\bm{\mu}}_\textsc{P}^T\hat{\bm{\mu}}_\textsc{N} < 1 - \eta(\sigma, \delta, \epsilon) - \sqrt{2 \eta(\sigma, \delta, \epsilon)} - \Delta(\mu) - \zeta_\mu
\end{flalign*}
where, 
\begin{flalign*}
\eta(\sigma, \delta, \epsilon) = 2(1 - \sigma) +\frac{R_\epsilon}{\min\{\pi, 1 - \pi\}} + \sigma ( L\delta + 2 \epsilon )\\
\Delta(\mu) = \frac{1}{2} - \frac{1}{2}\min_{\ell\in\{\textsc{P},\textsc{N}\}}\|\bm{\mu}_\ell\|^2\\
\zeta_\mu =(\zeta_\textsc{P} + \zeta_\textsc{N} + \zeta_\textsc{P}^T\zeta_\textsc{N})\\
\zeta_\textsc{P} = \|\hat{\bm{\mu}}_\textsc{P} - \bm{\mu}_\textsc{P}\|\;,\quad \zeta_\textsc{N} = \|\hat{\bm{\mu}}_\textsc{N} - \bm{\mu}_\textsc{N}\|
\end{flalign*}
Then, our goal is to show that the classification error of the NN classifier is bounded by:
\begin{flalign*}
    err(\hat{F}_{g_\mB}) \leq (1 - \sigma) + R_\epsilon(\gX_\textsc{PU})
\end{flalign*}
Before proving the theorem we state and prove (as necessary) the intermediate lemmas. 
\begin{lemma}
    Let, $\zeta_m = \|\hat{\vx}_m - \vx_m\|$ denote the estimation error for any normalized random variable $\vx \in \sR^d$ such that, $\|\vx\| = 1$. Then, for any two random variables $\vx_m, \vx_n$:
    \begin{flalign*}
        \|\vx_m^T\vx_n\| - \|\hat{\vx}_m^T\hat{\vx}_n\| 
        \leq \zeta_m +\zeta_n + \zeta_m^T\zeta_n.
    \end{flalign*}
    \label{lemma:bound_estimation_bias}
\end{lemma}
\begin{proof}
    \begin{flalign*}
             \|\vx_m^T\vx_n\| - \|\hat{\vx}_m^T\hat{\vx}_n\|\\
        \leq \|\vx_m^T\vx_n\| - \|\hat{\vx}_m\|\cdot\|\hat{\vx}_n\|\\
        \leq \|\vx_m\|\cdot\|\vx_n\| - \|\hat{\vx}_m\|\cdot\|\hat{\vx}_n\|\\
        \leq \bigg(\hat{\vx}_m + \zeta_m\bigg)\bigg(\hat{\vx}_n + \zeta_n\bigg) - \|\hat{\vx}_m\|\cdot\|\hat{\vx}_n\|\\
        = \hat{\vx}_m\zeta_n + \hat{\vx}_n\zeta_m + \zeta_m^T\zeta_n\\
        \leq \|\hat{\vx}_m\|\zeta_n + \|\hat{\vx}_n\|\zeta_m + \zeta_m^T\zeta_n\\
        \leq \zeta_m +\zeta_n + \zeta_m^T\zeta_n.
    \end{flalign*}
    This concludes the proof.
\end{proof}

\begin{lemma}
    Given a $(\delta, \sigma)$ augmentation $\gT$ and $L$ Lipschitz continuous encoder $g_\mB(\cdot)$, if:
    \begin{flalign*}
        \bm{\mu}_\textsc{P}^T\bm{\mu}_\textsc{N} < 1 - \eta(\sigma, \delta, \epsilon) - \sqrt{2 \eta(\sigma, \delta, \epsilon)} - \frac{1}{2}\bigg(1 - \min_{\ell\in\{\textsc{P},\textsc{N}\}}\|\bm{\mu}_\ell\|^2\bigg)
    \end{flalign*}
    where, $\eta(\sigma, \delta, \epsilon) = 2(1 - \sigma) +\frac{R_\epsilon}{\min\{\pi_p, 1 - \pi_p\}} + \sigma ( L\delta + 2 \epsilon )$.
    Then, the error rate for supervised NN classifier on a downstream PN classification task is bounded as:
    \begin{flalign*}
        err(F_{g_\mB}) \leq (1 - \sigma) + R_\epsilon
    \end{flalign*}
    \label{lemma:bound_clustdist}
\end{lemma}
\begin{proof}
    This is a direct consequence of \cite{huang2023towards}, Theorem 1.
\end{proof}
Now, we are ready to prove~\cref{th:generalization}. 
\begin{proof}
Applying~\cref{lemma:bound_estimation_bias} to derive a relationship between the optimal and estimated cluster centroids on the representation space. let, $\zeta_\textsc{P} = \|\hat{\bm{\mu}}_\textsc{P} - \bm{\mu}_\textsc{P}\|$ and $\zeta_\textsc{N} = \|\hat{\bm{\mu}}_\textsc{N} - \bm{\mu}_\textsc{N}\|$ be the  errors due to \textsc{puPL} on positive and negative centroid estimation. Then :
\begin{flalign}
    \|\bm{\mu}_\textsc{P}^T\bm{\mu}_\textsc{N}\| - \|\hat{\bm{\mu}}_\textsc{P}^T\hat{\bm{\mu}}_\textsc{N}\| \leq \zeta_\textsc{P} +\zeta_\textsc{N} + \zeta_\textsc{P}^T\zeta_\textsc{N}  
\end{flalign}
Comparing the bound with the bound in~\cref{lemma:bound_clustdist}, 
\begin{flalign*}
\|\hat{\bm{\mu}}_\textsc{P}^T\hat{\bm{\mu}}_\textsc{N}\|
+ 
\zeta_\textsc{P} +\zeta_\textsc{N} + \zeta_\textsc{P}^T\zeta_\textsc{N}
\\
\leq 1 - \eta(\sigma, \delta, \epsilon) - \sqrt{2 \eta(\sigma, \delta, \epsilon)} - \frac{1}{2}\bigg(1 - \min_{\ell\in\{\textsc{P},\textsc{N}\}}\|\bm{\mu}_\ell\|^2\bigg)
\end{flalign*}
Thus, we have:  
\begin{flalign}
    \|\hat{\bm{\mu}}_\textsc{P}^T\hat{\bm{\mu}}_\textsc{N}\| 
    \leq 1 - \eta(\sigma, \delta, \epsilon) - \sqrt{2 \eta(\sigma, \delta, \epsilon)} - \frac{1}{2}\bigg(1 - \min_{\ell\in\{\textsc{P},\textsc{N}\}}\|\bm{\mu}_\ell\|^2\bigg) - \zeta_\mu
\end{flalign}
where we have assumed $\zeta_\mu = \bigg(\zeta_\textsc{P} + \zeta_\textsc{N} + \zeta_\textsc{P}^T\zeta_\textsc{N}\bigg)$.
\\
\\
This concludes the proof.
\end{proof}

\subsection{Proof of \cref{lemma:bound_centroid_divergence_puCL}.}
\label{sec:proof.bound_centroid_divergence_puCL}
Our goal is show that,
the condition in~\cref{th:generalization} on the separation of the estimated class centroids~\eqref{eq:class_separation_bound} is satisfied, whenever:
\begin{flalign*}
    \log \bigg(\exp\bigg(\gL_\textsc{puCL}^\textsc{II}(\gX_\textsc{PU}) + c(\sigma, \delta, \epsilon, R_\epsilon)\bigg) + c'(\epsilon)\bigg) \nonumber\\
    <
    1 - \eta(\sigma, \delta, \epsilon) - \sqrt{2 \eta(\sigma, \delta, \epsilon)} - \frac{1}{2}\Delta_\mu -\zeta_\mu.
\end{flalign*}
where, 
\begin{flalign*}
    c(\sigma, \delta, \epsilon, R_\epsilon) =(2\epsilon + L\delta + 4(1-\sigma) + 8R_\epsilon)^2 + 4\epsilon + 2L\delta + 8(1 - \sigma) + 18R_\epsilon.
    \\
    c'(\epsilon) = \exp\frac{1}{\pi_p(1-\pi_p)} - \exp(1 - \epsilon).
\end{flalign*}
\begin{proof}
By adapting \cite{huang2023towards}, Theorem 3 to our setting and simplifying the constants, we get 
\begin{flalign*}
    \bm{\mu}_\textsc{P}^T\bm{\mu}_\textsc{N} 
    \leq \log 
    \bigg(\exp\bigg(
        \frac{1}{\pi_p(1- \pi_p)}\bigg(\gL_\textsc{puCL}^\textsc{II}(g_\mB) + c(\sigma, \delta, \epsilon, R_\epsilon)
        \bigg)\bigg) 
        - \exp(1 - \epsilon)
        \bigg)
\end{flalign*}
where,
\begin{flalign}
    c(\sigma, \delta, \epsilon, R_\epsilon) =\bigg(2\epsilon + L\delta + 4(1-\sigma) + 8R_\epsilon\bigg)^2 + 4\epsilon + 2L\delta + 8(1 - \sigma) + 18R_\epsilon.
\end{flalign}
\\
\\
Comparing this bound  with \cref{lemma:bound_clustdist} we get the condition: 
\begin{flalign*}
    &\log \bigg(\exp\bigg(\frac{1}{\pi_p(1- \pi_p)}\bigg(\gL_\textsc{puCL}^\textsc{II}(g_\mB) + c(\sigma, \delta, \epsilon, R_\epsilon)\bigg)\bigg) - \exp(1 - \epsilon)\bigg) \\
    &<
    1 - \eta(\sigma, \delta, \epsilon) - \sqrt{2 \eta(\sigma, \delta, \epsilon)} - \frac{1}{2}\bigg(1 - \min_{\ell\in\{\textsc{P},\textsc{N}\}}\|\bm{\mu}_\ell\|^2\bigg)  - \zeta_\mu
\end{flalign*}
This ensures:
\begin{flalign*}
    \hat{\bm{\mu}}_\textsc{P}^T\hat{\bm{\mu}}_\textsc{N} 
    < 1 - \eta(\sigma, \delta, \epsilon) - \sqrt{2 \eta(\sigma, \delta, \epsilon)} - \frac{1}{2}\bigg(1 - \min_{\ell\in\{\textsc{P},\textsc{N}\}}\|\bm{\mu}_\ell\|^2\bigg) - \zeta_\mu
\end{flalign*}
This concludes the proof.
\end{proof}

\end{document}